\documentclass[10pt,twocolumn,letterpaper]{article}

\usepackage{cvpr}              %

\usepackage[dvipsnames]{xcolor}
\definecolor{cvprblue}{rgb}{0.21,0.49,0.74}
\usepackage[pagebackref,breaklinks,colorlinks,citecolor=cvprblue]{hyperref}

\def\paperID{13614} %
\def\confName{CVPR}
\def\confYear{2024}

\usepackage[utf8]{inputenc} %
\usepackage[T1]{fontenc}    %

\usepackage{url}            %
\usepackage{booktabs}       %
\usepackage{amsfonts}       %
\usepackage{nicefrac}       %
\usepackage{microtype}      %

\usepackage{amsmath}
\usepackage{amssymb}
\usepackage{amsthm}
\usepackage{graphicx}
\usepackage{xspace}

\usepackage{enumitem}

\usepackage{amsmath,amsfonts,bm}

\def\eqref#1{equation~\ref{#1}}

\def\1{\bm{1}}

\DeclareMathAlphabet{\mathsfit}{\encodingdefault}{\sfdefault}{m}{sl}
\SetMathAlphabet{\mathsfit}{bold}{\encodingdefault}{\sfdefault}{bx}{n}

\def\gG{{\mathcal{G}}}
\def\gH{{\mathcal{H}}}

\def\gL{{\mathcal{L}}}

\def\gO{{\mathcal{O}}}
\def\gP{{\mathcal{P}}}

\def\gR{{\mathcal{R}}}
\def\gS{{\mathcal{S}}}

\def\gX{{\mathcal{X}}}

\def\gZ{{\mathcal{Z}}}

\def\sD{{\mathbb{D}}}

\def\sR{{\mathbb{R}}}

\newcommand{\E}{\mathbb{E}}

\usepackage{dsfont}
\usepackage{multirow}
\usepackage{subcaption}
\usepackage{siunitx}
\usepackage{chemformula}
\usepackage{collcell}
\usepackage{supertabular}
\usepackage{makecell}
\usepackage{algorithm}
\usepackage[noend]{algorithmic}

\colorlet{shadecolor}{gray!20}
\definecolor{babyblueeyes}{rgb}{0.63, 0.79, 0.95}
 \definecolor{babyblue}{rgb}{0.54, 0.81, 0.94}
 \definecolor{bluegray}{rgb}{0.4, 0.6, 0.8}
 \definecolor{cadmiumgreen}{rgb}{0.0, 0.42, 0.24}
 \definecolor{camouflagegreen}{rgb}{0.47, 0.53, 0.42}
\definecolor{darkseagreen}{rgb}{0.56, 0.74, 0.56}
\usepackage{tikz}
\usetikzlibrary{fit,calc}
\newcommand*{\tikzmk}[1]{\tikz[remember picture,overlay,] \node (#1) {};\ignorespaces}

\newcommand{\boxit}[1]{\tikz[remember picture,overlay]{\node[xshift=-1pt,yshift=-8pt,fill=#1,opacity=.15,fit={(A)($(B)+(.52\linewidth,.8\baselineskip)$)}] {};}\ignorespaces}
\newcommand{\boxitclient}[1]{\tikz[remember picture,overlay]{\node[xshift=-178pt,yshift=-9pt,fill=#1,opacity=.15,fit={(A)($(B)+(1.77\linewidth,1.5\baselineskip)$)}] {};}\ignorespaces}          
\colorlet{mypink}{red!30}
\colorlet{myblue}{cyan!50}

\usepackage[capitalize,noabbrev]{cleveref}

\theoremstyle{plain}

\newtheorem{lemma}{Lemma}

\theoremstyle{definition}
\newtheorem{definition}{Definition}
\newtheorem{assumption}{Assumption}
\theoremstyle{remark}
\newtheorem{remark}{Remark}
\usepackage[normalem]{ulem}
\usepackage{thmtools} 
\usepackage{thm-restate}

\newcommand{\cifar}{CIFAR-10\xspace}
\newcommand{\cifarpointone}{CIFAR-10.1\xspace}
\newcommand{\cifarc}{CIFAR-10-C\xspace}
\newcommand{\localtest}{$\operatorname{Local-test}$\xspace}
\newcommand{\globaltest}{$\operatorname{Global-test}$\xspace}
\newcommand{\office}{Office-Home\xspace}

\newcommand{\chexpert}{CheXpert\xspace}
\newcommand{\name}{\textsc{PerAda}\xspace}
\newcommand{\namewokd}{\textsc{PerAda w/o KD}\xspace}
\newcommand{\fedavg}{\textsc{FedAvg}\xspace}
\newcommand{\fedbn}{\textsc{FedBN}\xspace}
\newcommand{\ditto}{\textsc{Ditto}\xspace}
\newcommand{\apfl}{\textsc{APFL}\xspace}
\newcommand{\pfedme}{\textsc{pFedMe}\xspace}
\newcommand{\mtl}{\textsc{MTL}\xspace}
\newcommand{\lgfedavg}{\textsc{LG-FedAvg}\xspace}
\newcommand{\fedrep}{\textsc{FedRep}\xspace}
\newcommand{\fedsim}{\textsc{FedSim}\xspace}
\newcommand{\fedalt}{\textsc{FedAlt}\xspace}
\newcommand{\fedprox}{\textsc{FedProx}\xspace}
\newcommand{\feddyn}{\textsc{FedDyn}\xspace}
\newcommand{\feddf}{\textsc{FedDF}\xspace}
\newcommand{\standalone}{\textsc{Standalone}\xspace}

\newcommand{\chulin}[1]{\textcolor{black}{#1}}

\definecolor{revcolor}{HTML}{0a46f4}  %
\newcommand{\add}[1]{\textcolor{black}{#1}}

\newcommand\blfootnote[1]{%
  \begingroup
  \renewcommand\thefootnote{}\footnote{#1}%
  \addtocounter{footnote}{-1}%
  \endgroup
}

\title{\name: Parameter-Efficient Federated Learning Personalization with Generalization Guarantees}

\author{Chulin Xie$^{\dagger,\ddagger}$, De-An Huang$^\spadesuit$, Wenda Chu$^\heartsuit$, Daguang Xu$^\spadesuit$,\\  Chaowei Xiao$^{\spadesuit,\P,*}$, Bo Li$^{\dagger,\S,*}$,  Anima Anandkumar$^{\heartsuit,*}$ \\
{$^\dagger$UIUC \quad $^\spadesuit$NVIDIA  \quad $^\heartsuit$Caltech  } 
{\quad $^\P$UW-Madison \quad $^\S$UChicago}
}

\begin{document}

\maketitle

\looseness=-1
\begin{abstract} 
Personalized Federated Learning (pFL) has emerged as a promising solution to tackle data heterogeneity across clients in FL. 
However, existing pFL methods either (1) introduce high computation and communication costs or (2) overfit to local data, which can be limited in scope and vulnerable to evolved test samples with natural distribution shifts.
In this paper, we propose  \textit{\name}, a parameter-efficient pFL framework that reduces communication and computational costs and exhibits superior generalization performance, especially under test-time distribution shifts. 
\name reduces the costs by leveraging the power of pretrained models and only updates and communicates a small number of additional parameters from adapters. \name achieves high generalization by regularizing each client's personalized adapter with a global adapter, while the global adapter uses knowledge distillation to aggregate generalized information from all clients.
Theoretically, we provide generalization bounds of \name, and we prove its convergence to stationary points under non-convex settings.
Empirically, \name demonstrates higher personalized performance (+4.85\% on \chexpert) and enables better out-of-distribution generalization (+5.23\% on CIFAR-10-C) on different datasets across natural and medical domains compared with baselines, while only updating 12.6\% of parameters per model.
Our code is available at \href{https://github.com/NVlabs/PerAda}{https://github.com/NVlabs/PerAda}.

\end{abstract}

\blfootnote{$\ddagger$ work done during an internship at NVIDIA; ${*}$~equal advising.}

\vspace{-2mm}
\section{Introduction}
\label{sec:intro}
 \vspace{-2mm}
 
Federated Learning (FL) allows clients to collaboratively train machine learning models without direct access to their data, especially for privacy-sensitive tasks~\cite{mcmahan2016communication}. 
FL was initially designed to train a single global model for all clients. However, such a one-model-fits-all paradigm is not effective when there is \emph{client heterogeneity}, i.e., the local data are non-IID across clients with heterogeneous features or label distributions~\cite{Li2020On}.
Personalized Federated Learning (pFL)~\cite{mansour2020three} has emerged as an effective solution to  tackle client heterogeneity. 
In pFL, each client trains a personalized model on its local data to ensure personalized performance, while leveraging the aggregated knowledge from other clients to improve its
generalization.

\begin{figure}
\centering
{
\begin{tabular}{ll}
\includegraphics[width=0.38\linewidth]{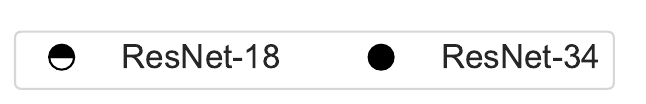}
\vspace{-16pt}\\
\end{tabular}
\vspace{-2mm}
\begin{subtable}{\linewidth}
\centering
\begin{tabular}{c@{}c@{}c@{}c@{}}
        \vspace{-3pt}\\
\includegraphics[height=0.35\linewidth]{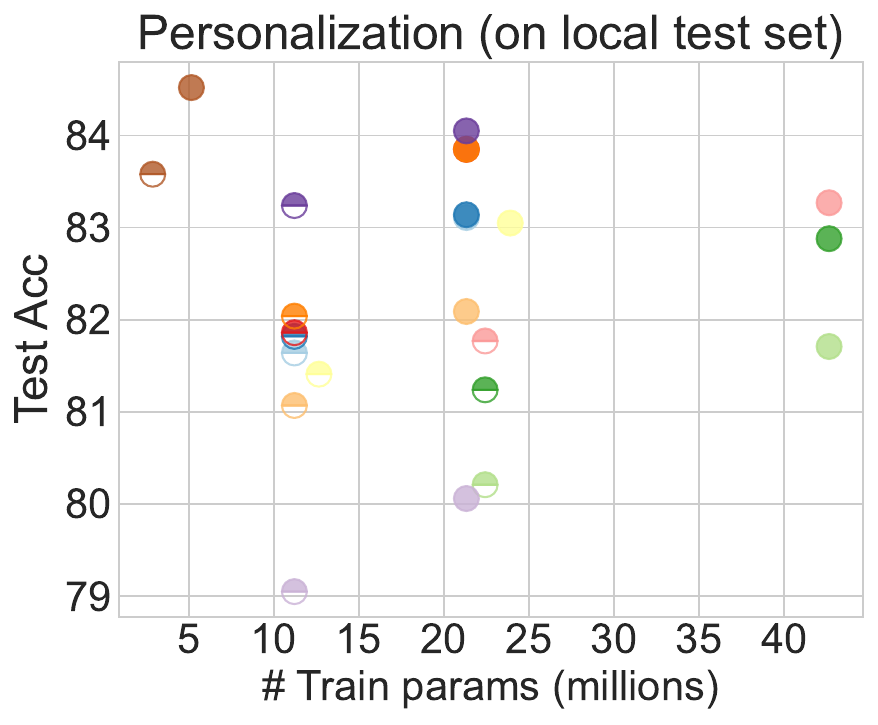}&
\includegraphics[height=0.35\linewidth]{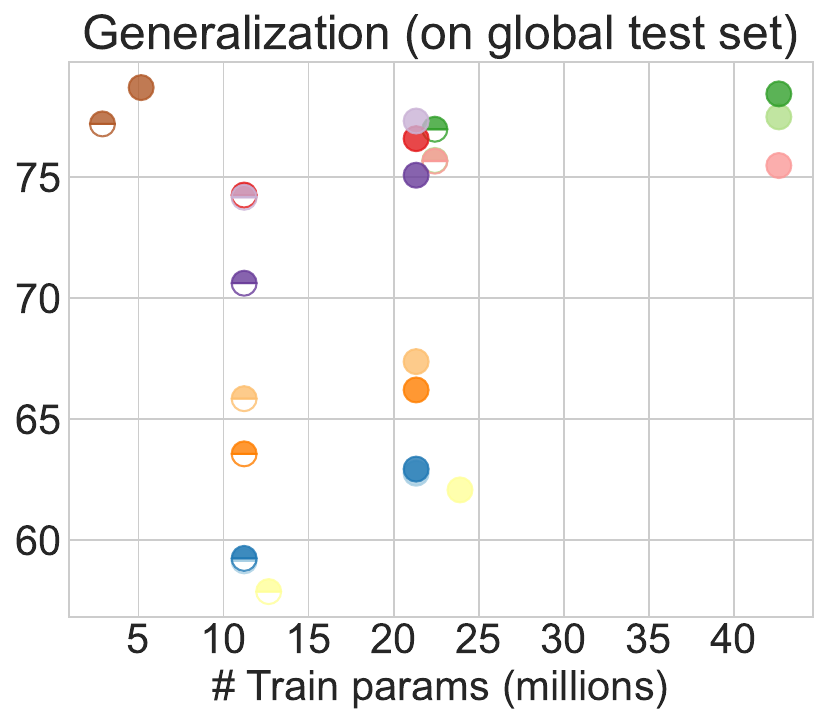}&
\includegraphics[height=0.35\linewidth]{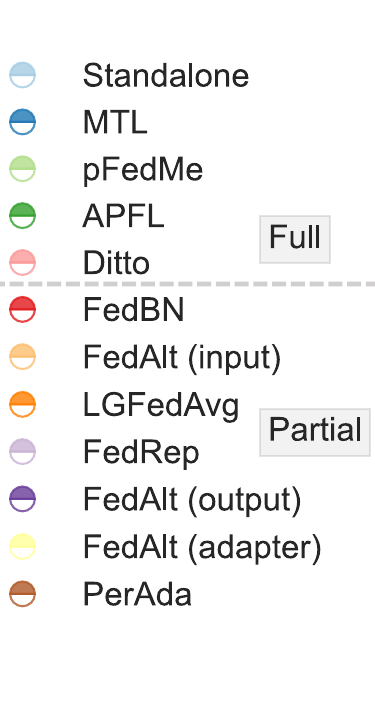}
\\[-1.2ex]
\end{tabular}
\end{subtable}
}
\vspace{-1mm}
\caption{\chulin{\small  Accuracy of personalized models on \office. ``Full''/``Partial'' denotes full/partial model personalization. 
 \name achieves the highest personalized performance and generalization by updating the smallest number of model parameters.  
}}%
\label{fig:modelarc}
\vspace{-8mm}
\end{figure}

Existing works in pFL commonly use \textit{full model personalization}, where each client trains a personalized model as well as a copy of the global model from the server for regularization~\cite{li2021ditto,t2020personalized}.
However, these methods are parameter-expensive, leading to high computational and communicational costs, which is impractical for clients with limited computation resources and network bandwidth~\cite{kairouz2021advances}. Later on, \textit{partial model personalization} alleviates this issue by splitting each client's \textit{one} model into personalized parameters and  shared parameters, where only the  set of shared parameters would be communicated with the server~\cite{pillutla2022federated}.
Nonetheless, these methods tend to overfit more to the local training samples since the set of shared parameters does not encode generalized knowledge well compared to a full global model. This hurts the performance of partially personalized models in real-world  FL deployment, where the incoming local test samples are evolving with natural shifts from the local training distribution~\cite{jiang2022test}, e.g., images taken under varying weather or lighting conditions.

\begin{figure}
    \centering
    \includegraphics[width=0.88\linewidth]{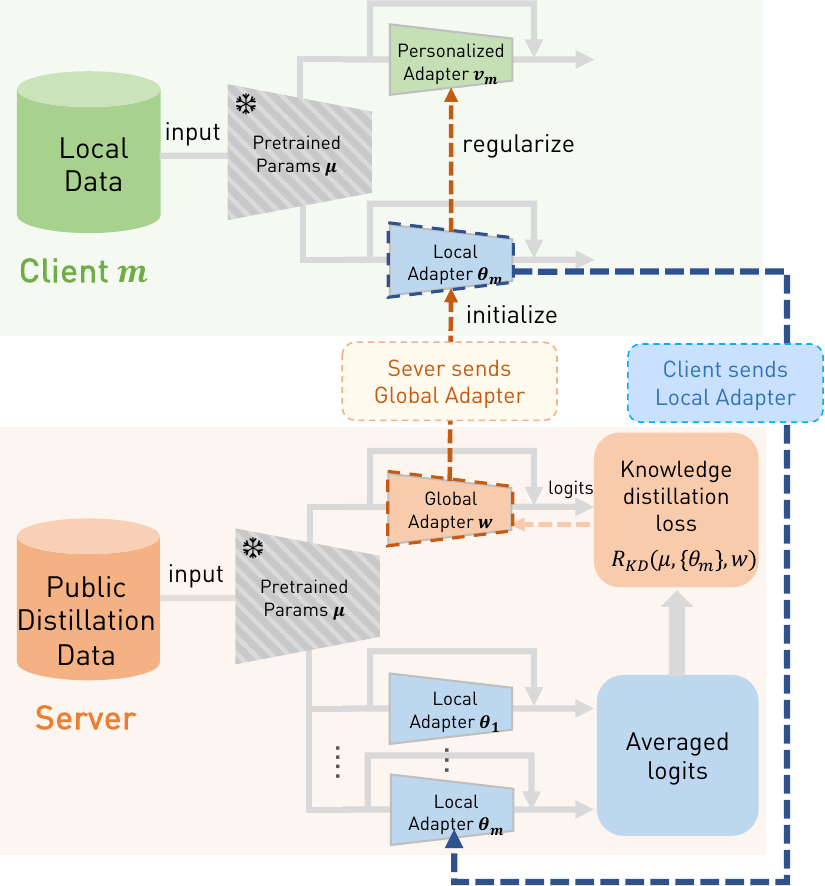} %
         \vspace{-3mm}
    \caption{\small \chulin{Illustration of \name}.   
    }
    \label{fig:overview}
     \vspace{-7mm}
\end{figure}

\noindent\textbf{Our Approach.} In this work, we \chulin{propose \name, a pFL framework}
that \textit{reduces communication and computation costs \chulin{for clients} while personalizing the model and 
 maintaining its generalization to test-time distribution shifts}\chulin{, as shown in \cref{fig:modelarc}}.
\name is a parameter-efficient \textbf{per}sonalized FL framework based on  \chulin{\textbf{Ada}pter~\cite{rebuffi2017learning} and Knowledge Distillation (KD)~\cite{hinton2015distilling}}.
The overview is shown in  \cref{fig:overview}. 

Each client has a pretrained model, a personalized adapter, and a local adapter, where each adapter consists of a small number of additional parameters planted in the pretrained model with skip connections.
At each training round, \textit{to reduce the computation and communication costs}, \name leverages the power of the pretrained model, and \textit{only} updates the personalized adapter and the local adapter using  local data, and sends the local adapter to the server. 
In this way, it limits the number of trainable parameters and only communicates %
the local adapter, instead of the full model.

Then, to \textit{improve the generalization}, 
the server aggregates clients' local adapters (i.e., teachers) via knowledge distillation and trains the global adapter (i.e., student). 
Specifically, it uses the averaged logits from teachers on an unlabeled public distillation dataset as the pseudo-labels to train the student. 
\add{This avoids directly averaging clients' models trained on heterogeneous local data, while enriching the global adapter with the ensemble knowledge from clients' models and mitigating the potential model aggregation drifts caused by heterogeneity.}
After that, the server sends the distilled global adapter back to the clients, which is used to initialize the local adapter and regularize the training of the personalized adapter to prevent overfitting and \textit{improve the generalization}. 
During the testing phase, each client uses the personalized adapter for inference. 

To explain why \name is effective in improving generalization, we theoretically derive its generalization bounds under FL  covariate (or feature) shift non-IID setting~\cite{marfoq2022personalized}. We are the \textit{first} to show that the generalization on a target distribution \chulin{(e.g., potentially with test-time distribution shift)} can be enhanced for both global model and personalized models by KD when the \textit{distillation optimization error is small}, and the distribution of the unlabeled distillation dataset is \textit{close} to the \add{target} distribution.
We also characterize the role of different components in \name on {generalization}, such as client heterogeneity, pretrained model, and the prediction distance between the global  and personalized models.  

In addition, we establish convergence guarantees for \name in general non-convex settings.
The analysis of \name is challenging due to the bi-level optimization between server distillation training and local client training. 
We establish the convergence rates for the global model and personalized models to stationary points and demonstrate the effects of KD and client heterogeneity on the convergence. \add{As far as we know,} these are the \textit{first}-known results for FL convergence under \textit{server distillation}.

Empirically, we conduct extensive evaluations on different datasets, including natural and medical images (\cifar{}, \office, and \chexpert) under both FL covariate-shift and label-shift non-IID settings.  We show that \name achieves competitive personalized accuracy over state-of-the-art pFL methods with only 12.6\% of trainable parameters  while obtaining higher generalization, especially when evaluated on out-of-distribution data. We further show that the benefits of \name extend to differentially private (DP) FL settings and improve the DP-utility trade-offs compared to full model personalization.
In summary,
\begin{itemize}[noitemsep,leftmargin=*]
\item We propose \name, a lightweight pFL framework with personalized adapters that provides personalization while reducing computation/communication costs. We improve the generalization of \name with server-side KD.
\item 
We theoretically analyze the effectiveness of \name, and prove the generalization bounds and the convergence rates for both the global model and personalized models under non-convex settings. 
\item Through extensive experiments, we show that \name achieves higher personalized performance and better generalization than state-of-the-art pFL methods with smaller computation and communication costs. Moreover, \name retains its benefits under differential privacy. 
\end{itemize}

\vspace{-2mm}
\section{Related Work}
 \vspace{-2mm}
\textbf{Full Model Personalization.}
Many pFL approaches require each client to train a personalized model and a global model,
where the global model is used to prevent the personalized model from overfitting. It includes methods based on meta learning~\cite{fallah2020personalized}, model mixture ~\cite{hanzely2020federated,deng2020adaptive,mansour2020three}, global reguarlization~\cite{li2021ditto}, mean regularization~\cite{t2020personalized,hanzely2020federated,hanzely2020lower} and clustering~\cite{sattler2020clustered,ghosh2020efficient}. 
However, these methods induce high costs by training two full models in each client and communicating the full model.  
Another approach is to locally finetune an FL global model (e.g., from \fedavg~\cite{mcmahan2016communication}). While local fine-tuning  yields promising personalized accuracy~\cite{yu2020salvaging,wang2019federated,chen2022on}, it could be prone to catastrophic forgetting and overfitting to its (limited) local data,  sacrificing the generalizability~\cite{jiang2022test,ramasesh2022effect}.

\textbf{Partial Model Personalization} trains one model for each client to reduce the costs, which is partitioned into shared parameters and personalized parameters,
such as personalized feature extractors~\cite{collins2021exploiting}, prediction head~\cite{arivazhagan2019federated,liang2020think,chen2021bridging}, batch normalization~\cite{li2021fedbn}, adapters~\cite{pillutla2022federated}, and adaptively selected parameters~\cite{sun2021partialfed}. 
Nevertheless, the shared parameters do not learn generalized information well compared to a full global model, so the partially personalized models can have inferior generalization ability. To further reduce the costs, Shysheya et al.~\cite{shysheya2023fit} apply parameter-efficient transfer learning techniques to train \fedavg and perform local finetuning. However, it does not specifically address the generalization issues of personalization, which is the focus of our work.

\textbf{Knowledge Distillation (KD) in FL.} KD is a technique that transfers the knowledge from one or multiple teacher models to a student model~\cite{hinton2015distilling}. 
\textit{Ensemble distillation} has been used to tackle data heterogeneity in generic FL,  by refining the \textit{server} model with ensemble knowledge from clients, rather than directly aggregating their model parameters. 
Specifically, the ensemble predictions from clients' models on an unlabeled dataset are used to guide the training of the server model, where the unlabeled dataset can be public data~\cite{lin2020ensemble,chen2020fedbe,li2019fedmd} or generated data~\cite{zhang2022fine}. 
Another line of work leverages \textit{client}-side \textit{local distillation} to transfer global knowledge to local models in generic FL ~\cite{zhu2021data,lee2021preservation} or personalized models in pFL~\cite{zhang2021parameterized,ozkara2021quped}. 
To reduce the load for clients, we focus on parameter-efficient ensemble distillation in the server with public data to train a better global model, and study its effects on  personalized models with novel convergence guarantees and generalization bounds. 

\chulin{\textbf{Parameter-efficient fine-tuning} techniques applied to pretrained large models~\cite{bommasani2021opportunities} have become the prominent practice in transfer learning to save computation costs~\cite{gao2021clip,liu2021pre,lester2021power}. Motivated by the success of Adapter, a low-cost plug-in mounted on pre-trained vision models~\cite{rebuffi2017learning} or large language models~\cite{houlsby2019parameter,li2021large,liu2023communication}, we investigate Adapter in the context of parameter-efficient personalization. Instead of training both the backbone and adapter for pFL as in~\cite{pillutla2022federated}, we treat the adapter parameters as personal and the rest of the model parameters as {frozen}, and further leverage sever-side ensemble distillation to improve pFL performance.}
\vspace{-3mm}
\vspace{-1mm}
\section{Preliminaries and Challenges}\label{sec:prelim}
\vspace{-2mm}
We consider a typical setting of FL with  $M$ clients where each client $m$ has a training dataset $\sD_m=\{\left(x_{m,j}, y_{m,j} \right), j \in [n_m] \}$ with  $n_m$ data samples dawn from its local distribution $\mu_m$. 
Let 
 $f (W, x)$  represents a model that outputs the logit vector given input $x$,  where $W \in \sR^{d}$, denotes its model parameters.
Let the loss function be 
$\ell (f(W, x),y)$,
and  the empirical loss on local data $\sD_m$ associated with client $m$ be 
$ \gL_m(W) :=  \frac{1}{n_m} \sum_{j=1}^{n_m} \ell\left(f\left(W,x_{m,j}\right),y_{m,j}\right)$.

\textbf{Generic FL}  aims to optimize a single global model with all clients' local data with the FL objective:
$ \min_{W} \gL(W) $ where $  \gL(W):=\frac{1}{M}\sum_{m=1}^M \gL_m(W)$.
A standard way to solve it is \fedavg, which iterates between local model training and global model aggregation for multiple communication rounds. 
However, due to the heterogeneous local data distributions among clients, local model would drift away from each other, making the aggregated global model deviate from the optimal solution.

\textbf{Personalized FL} learns a personalized model for each client to perform well on its local data while preventing overfitting by leveraging the knowledge from other clients.
However, achieving the goal is non-trivial due to the following challenges:
(1)  \textbf{High costs}: existing full model personalization studies~\cite{li2021ditto,t2020personalized,fallah2020personalized, hanzely2020federated}, which optimize $ {\min_{W,\{V_m\}}   \frac{1}{M} \sum_{m=1}^M  ( \gL_m (V_m ) + \frac{\lambda}{2}\|V_m - W\|^2  )}$,  require \textit{twice} the memory footprint of the full model at each client by locally updating personalized model $V_m \in \sR^{d}$ and global model $W \in \sR^{d}$ where $\lambda$ is the $\ell_2$ regularization weight controlling the extent of personalization.
(2) \textbf{Limited generalization}: partial model personalization~\cite{collins2021exploiting,liang2020think,chen2021bridging,pillutla2022federated} is more efficient by training a full model  $V_m=(u,v_m)$ at each client and communicating a subset of parameters, where $u\in \sR^{d_u}$ are shared parameters and  $v_m\in \sR^{d_v}$ are personal parameters:
$ {\min_{u,\{v_m\}}   \frac{1}{M} \sum_{m=1}^M   \gL_m ( u, v_m)}.$
However, such a partially personalized model can be \textit{dominated by personal knowledge} with $v_m$ and \textit{poor at encoding generalized knowledge} with the remaining $u$ from global distribution, 
leading to inferior performance under test-time distribution shifts. \chulin{\cref{fig:compare_pfl} depicts such challenges
in existing studies.
}

\vspace{-3mm}
\section{Method} \label{sec:method}
\vspace{-2mm}
Here we introduce the objectives and algorithm for \name. 

\textbf{Personalized and Global Objectives of \name.}
We address the challenges discussed in \cref{sec:prelim} by proposing \name, 
which improves the efficiency of learning personalized adapters and enhances their generalization with regularization and KD.  Specifically, we (1) train the personalized adapter $\{v_m\}$  regularized towards a global adapter $w$ to optimize a personalized objective (\ref{eq:per-obj}), and (2) train a well-generalized $w$ via KD to optimize a global objective (\ref{eq:global-obj}) under non-IID data, where we use the {\textit{alternative} optimization} between client local training of local adapter $\{\theta_m\}$ and server KD training of $w$.

Concretely, we improve the efficiency of partial model personalization with a pretrained model and personalized adapters. Here the personalized adapter consists of a small number of additional parameters with skip connections (in \cref{fig:overview}), which can reduce to the identity function when its parameters are zero~\cite{rebuffi2017learning,zhang2021parameterized}.
Our personalized adapter is trained with regularization to prevent overfitting, yielding the personal objective of each client $m$:
\vspace{-3mm}
\begin{align}
 {\min_{v_m}  P_m(v_m, w):= \gL_m ( u , v_m ) + \frac{\lambda}{2}\|v_m - w\|^2}  \tag{Personal Obj},  \label{eq:per-obj}  
\end{align}
where $u \in \sR^{d_u}$ denotes the fixed pretrained parameters, and $v_m, w \in \sR^{d_a}$ are \textbf{personalized adapter} and \textbf{global adapter}, respectively, with $d_a \ll d_u$.

Since the global adapter $w$ is trained with all client data, regularizing  $v_m$ with $w$ could potentially boost $v_m$'s generalization power. Thus, enhancing $w$'s generalization capacity is crucial for training a personalized model that demonstrates robust generalization as well.
Instead of using \fedavg~\cite{mcmahan2016communication} to learn $w$
as in
regularization-based pFL method~\cite{li2021ditto}, we leverage server-side ensemble distillation~\cite{lin2020ensemble} to enrich the global adapter with ensemble knowledge from clients' models and alleviate model aggregation drifts induced by client heterogeneity, yielding the global objective: 
\vspace{-2mm}
\begin{align} 
 &\min_{w}   \mathcal{R}_{\mathtt{KD}}(u, \{\theta_m\}_{m=1}^M, w ) \tag{Global Obj} \label{eq:global-obj}  \\ 
 & \text{\quad where\quad}  \theta_m= \arg \min_{\theta}  \gL_m (u,\theta) , \text{initialized with $w$}. \nonumber
 \end{align}
Here $\theta_m \in \sR^{d_a}$ is client $m$'s \textbf{locally updated global adapter}, and we call it as \textbf{local adapter} for distinguishment. 
The KD loss is defined as:
$
  \mathcal{R}_{\mathtt{KD}}( u,  \{\theta_m\}_{m=1}^M, w ):=\sum_{j=1}^{n_{\mathtt{aux}}} \ell_{\mathtt{KD}} (\sum_{m=1}^M \frac{f ((u, {\theta}_m), x_j) } {M} ,   f ((u, w) , x_j)),
$
which is the average distillation loss (between the
 averaged logits  of local models and logits of the global model) on an auxiliary (unlabeled) dataset  $\sD_{\mathtt{aux}}=\{x_j\}_{j=1}^{n_{\mathtt{aux}}}$ drawn from the distribution $\mu_{\mathtt{aux}}$. 
Here $\ell_{\mathtt{KD}}(a,b)= \mathtt{KL}(\sigma(a), \sigma(b))$ is  Kullback-Leibler divergence loss where $\sigma$ is softmax function~\cite{hinton2015distilling}.  Compared to server-side KD in generic FL~\cite{lin2020ensemble,chen2020fedbe,zhang2022fine},  we only update adapters instead of full models, which is more efficient for training  and communication. 
  
\vspace{-3mm}
\begin{algorithm}[H]\small
 \caption{\name with \colorbox{red!10}{client} and \colorbox{cyan!10}{server} training}
 \label{algo:fl}
\begin{algorithmic}[1]
\footnotesize
 \STATE {\bfseries Input:}{$M$ clients,  pretrained model parameters $u$, initialized adapters $w^0$, $\{v_m^0\}$, local datasets $\{\sD_m\}$, an unlabeled dataset $\sD_{\mathtt{aux}}$}
 \STATE {\bfseries Output:}{ Personalized adapters 
$v_1^T,\ldots,v_M^{T}$ 
\; }
\FOR {communication round $ t\in [T]$} 
    \STATE $\gS_t \gets $ Server samples $C$ clients from $M$ clients
    \STATE Server {sends} \textbf{global adapter} ${w}^{t}$ to the selected clients
     \tikzmk{A}\FOR {client $ m \in \gS_t$}
        \STATE Client initializes \textbf{personalized adapter}  ${v}_m^{t,0}$ as ${v}_m^{t}$ 
        \FOR{step $ s \in [S]$}
             \STATE  \textcolor{blue}{{\footnotesize\ttfamily // update personalized adapter}}
             \STATE  {\tiny ${v}_m^{t,s+1}\gets$ ${v}_m^{t,s} - \eta_p \left( \widetilde{\nabla} \gL_m\left( u, {v}_m^{t,s}\right) + \lambda \left({v}_m^{t,s} - {w}^{t} \right) \right)$\label{algoline:permodel-update}}
        \ENDFOR
        \STATE Client sets ${v}_m^{t+1}\gets {v}_m^{t,S} $
        \STATE Client initializes \textbf{local adapter}  ${\theta}_m^{t,0}$ as ${w}^{t}$ 
        \FOR{step $ e \in [E]$}
            \STATE  \textcolor{blue}{\footnotesize\ttfamily // update  local adapter}
            \STATE  {\scriptsize${\theta}_m^{t,e+1}\gets{\theta}_m^{t,e} - \eta_l \widetilde{\nabla} \gL_m(u, {\theta}_m^{t,e})$ \label{algoline:localmodel-update}}
        \ENDFOR 
        \STATE Client {sends} \textbf{local adapter} ${\theta}_m^{t+1}\gets {\theta}_m^{t,E}$ to  server 
    \ENDFOR \tikzmk{B}\boxitclient{mypink}
    \tikzmk{A}\STATE Server {initializes}  the \textbf{global adapter} $w^{t,0}$ by averaging
    \STATE {\scriptsize $w^{t,0} \gets \sum_{m \in \gS_t} \frac{1} {|\gS_t|}  {\theta}_m^{t+1}$ }
    \FOR{step $ r \in [R]$}
     \STATE   \textcolor{blue}{{\footnotesize\ttfamily // update global adapter}}
    \STATE {\scriptsize $w^{t,r+1}\gets$$ w^{t,r} - \eta_g \widetilde{\nabla}_w \gR_{\mathtt{KD}}(u,\{\theta^{t+1}_{m}\}_{m \in \gS^t} , w^{t,r})$ \label{algoline:server-update}} 

    \ENDFOR
    \STATE Server sets ${w}^{t+1}\gets {w}^{t,R} $\tikzmk{B}\boxit{myblue}
\ENDFOR
 \end{algorithmic}

\end{algorithm}
\vspace{-5mm}

\begin{figure}
    \centering
    \includegraphics[width=1.02\linewidth]{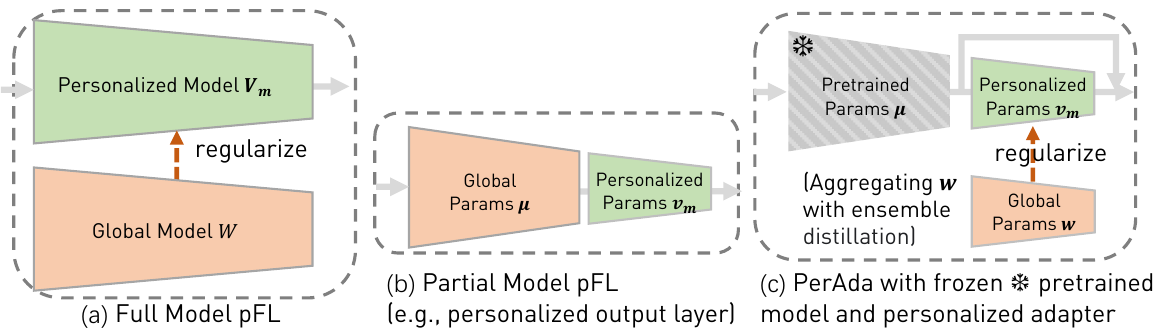} %
         \vspace{-7mm}
    \caption{\small \chulin{Current full model personalization incurs high computation costs by  training two models, whereas existing partial model personalization often falls short in terms of generalizability. By updating adapter {only}, \name achieves a favorable balance between training/communication costs of clients and their pFL performance.}
    }
    \label{fig:compare_pfl}
     \vspace{-7mm}
\end{figure}

\textbf{\name Algorithm.}
Now we introduce the details of iteratively optimizing the personalized objective and the global objective.
\cref{algo:fl} presents our workflow. 
At each communication round $t \in [T]$, the server selects $C$ clients  $\gS_t$ and broadcasts the current global adapter $w^t$. 
\textbf{To optimize personalized objective}, 
each selected client   $m\in \gS_t$  initializes personalized adapter as ${v}_m^{t,0} \gets v_m^t$, and updates it for $S$ steps with learning rate $\eta_p$ and mini-batches $\{\xi_m^{t,s}\}_{s=0}^{S-1}$ sampled from $\sD_m$  (Line \ref{algoline:permodel-update}).
The client sets personalized adapter ${v}_m^{t+1} \gets {v}_m^{t,S}$ after training. 
\textbf{To optimize global objective}, each selected client $m$ initializes local adapter as the received global adapter ${\theta}_m^{t,0} \gets w^t$, and makes local updates for $E$ steps with learning rate $\eta_l$ and mini-batches $\{\xi_m^{t,e}\}_{e=0}^{E-1}$ sampled from $\sD_m$  (Line \ref{algoline:localmodel-update}).
Then client $m$ sends the updated local adapter $ {\theta}_m^{t+1} \gets {\theta}_m^{t,E} $ to server.
After receiving local adapters, the server first initializes the global adapter by parameter-averaging $w^{t,0} \gets {\bar\theta}_m^{t+1}$ where ${\bar\theta}_m^{t+1}:=\sum_{m \in \gS_t} \frac{1} {|\gS_t|}  {\theta}_m^{t+1}$. Then, the server updates global adapter for $R$ steps via knowledge distillation from local adapters (Line ~\ref{algoline:server-update}) with learning rate $\eta_g$ 
and batches $\{\xi^{t,r}\}_{r=1}^{R} $ sampled from $\sD_{\mathtt{aux}}$. The server will send the updated global adapter as $w^{t+1}\gets w^{t,R}$  to clients at the next communication round.

 \vspace{-2mm}
\section{Generalization Bounds of \name}\label{sec:generalizationbounds}
 \vspace{-2mm}
In this section, we analyze the generalization bounds for \name \add{by answering the questions: 
\textit{how do the distillation data distribution and KD optimization impact the generalization of  the global model? How does the global model impact the generalization of personalized models?  }
}

For notation simplicity, we define 
$p_1, \cdots, p_M$ as the personalized hypothesis, where each hypothesis $p_m \in \gP_m: \gX \rightarrow [0, 1]^k$  maps the input  $x \in \gX$ to a \textit{probability vector} over the $k$ classes (i.e.,  softmax outputs).
Similarly, we define global hypothesis $g \in \gG$  and local hypothesis $h_m(x) \in \mathcal{H}_m, \forall m \in [M]$. We call ``hypothesis'' as ``model'' in this section. 
The local dataset  $\sD_m$  of each client $m$ is drawn from the local distribution $\mu_m$, and the distillation dataset $\sD_{\mathtt{aux}}$ of the server is drawn $\mu_{\mathtt{aux}}$. 
We study generalization of the global model and personalized models on  a  \chulin{\textbf{target distribution $\mu$ of interest} (e.g., with distribution shifts)}, by analyzing the effect of  local distributions $\{\mu_m\}$ and distillation distribution $\mu_{\mathtt{aux}}$ \chulin{used in FL training}.
We focus on the generalization bounds under FL covariate shifts following ~\cite{marfoq2022personalized} and defer all proofs to \cref{app:generalization-proof}.

\noindent\textbf{Global Model.} 
Previous KD-based FL generalization bounds  ~\cite{lin2020ensemble,zhu2021data} simply assume a perfect distillation 
(i.e., the global model is the ensemble of local models) 
\textit{which neglects the actual distillation errors and the choice of distillation distribution.}
To take them into account, we define the \textit{ensemble distillation distance} on $n_{\mathtt{aux}}$ points  $\{x_i\}_{i=1}^{n_{\mathtt{aux}}}$ drawn
from ${\mu_{\mathtt{aux}}}$ as:
$\Phi_{\mu_{\mathtt{aux}},n_{\mathtt{aux}}} (h_1,\ldots, h_M; g ):= \frac{1}{n_{\mathtt{aux}}} \sum_{i=1}^{n_{\mathtt{aux}}} \|  g(x_i) -  \frac{1}{M}\sum_{m=1}^M h_m(x_i) \|_1 $ 
which measures the output difference between the global model and the ensemble of local models.
To show  $g$ can have good generalization bounds on $\mu$ with KD,  our main idea is to bound error probabilities of $g$ with the expected distillation distances and errors of local models,
and then bound the errors on $\mu$ by $\mu_m$ based on 
\chulin{prior arts from domain adaptation~\cite{ben2010theory}.
We defer the preliminaries about learning theory to \cref{app:proof-g-model-generalization}.
}
{
\abovedisplayskip=1pt
\abovedisplayshortskip=1pt
\belowdisplayskip=1pt
\belowdisplayshortskip=1pt
\vspace{-2mm}
  \begin{restatable}[Generalization bound of \name global model]{theorem}{globalgen} 
 \label{thm:main_distill}
Consider empirical datasets $\sD \sim \mu, \sD_{\mathtt{aux}} \sim \mu_{\mathtt{aux}}, \sD_{m} \sim \mu_{m}$ 
with $|\sD|=|\sD_m|=n, |\sD_{\mathtt{aux}}|=n_{\mathtt{aux}}$. 
Let $d_m$ be the VC dimension of $\gH_m$,  $\operatorname{Rad}_{n_{\mathtt{aux}}}$ be the empirical Rademacher complexity measured on $n_{\mathtt{aux}}$ samples. 
With probability at least $ 1- \delta$,
for every $h_m \in \gH_m , \forall m \in [M] $ and  $g \in\gG $, we have 
$\underset{(x,y)\sim \mu}{\operatorname{Pr}}\left[\underset{y^{\prime}}{\arg \max } g(x)_{y^{\prime}} \neq y\right]  
\leq  \underset{(x,y)\sim \mu}{2\E} [1- g(x)_y] \leq  
  {\mathcal{O}}( {k^{3/2}} [ \max_{j} (\frac{1}{M} \sum_{m=1}^M  \operatorname{Rad}_{n_{\mathtt{aux}}}(\mathcal{H}_m|_j)	) + \max_{j} \operatorname{Rad}_{n_{\mathtt{aux}}}(\mathcal{G}|_j) ]) 
+\frac{6}{M} \sum\limits_{m=1}^M  (\frac{4}{3}\sqrt{\frac{2d_m \log (2n)+ \log(6M/\delta)}{n}} +  \sqrt{\frac{\log(6M/\delta)}{2n}} + \sqrt{\frac{\log(6/\delta)}{2n_{\mathtt{aux}}}} + {\mathcal{O}}(\operatorname{Rad}_{n}(\mathcal{H}_m))  ) 
+ \frac{ 1}{M} \sum\limits_{m=1}^M   (  \underbrace{2 \mathtt{ERR}(\sD_m, h_m)}_{\text{local empirical risk}}  +  \underbrace{ {\hat d}_{\gH \Delta \gH} (\sD_m, \sD)}_{\text{client heterogeneity}} +  \lambda_m 
  )  +  2 \underbrace{ \Phi_{\mu_{\mathtt{aux}},n_{\mathtt{aux}}} (h_1,\ldots, h_M; g )}_{\text{ensemble distillation distance}} + 4\underbrace{\mathbb{TV} (\mu,\mu_{\mathtt{aux}})}_{\text{TV divergence}}   
  $,
where $\mathtt{ERR}(\sD_m, h_m) =  \frac{1}{n} \sum_{j=1}^{n} \left[1- h_m(x_{m,j})_{y_{m,j}}\right] , \lambda_m = \varepsilon_{\mu_m}(h^*) +  \varepsilon_{\mu}(h^*), h^* := \arg\min_{h \in\gH } \varepsilon_{\mu_m} (h) +  \varepsilon_{\mu} (h)$.
\end{restatable}
}
\vspace{-2mm}
\begin{remark}\label{remark:generalization-global}
\looseness=-1
We discuss key implications of \cref{thm:main_distill}:
(1) \textbf{Ensemble distillation.} $ \Phi_{\mu_{\mathtt{aux}},n_{\mathtt{aux}}}$ captures the distillation error measured on the distillation dataset $\sD_{\mathtt{aux}}$ as minimized in Line \ref{algoline:server-update}.
When $\mu_{\mathtt{aux}}=\mu$, e.g., \add{using data from the target distribution as the distillation dataset,} KD improves the generalization of  $g$ during training by directly minimizing $ \Phi_{\mu_{\mathtt{aux}},n_{\mathtt{aux}}}$. The smaller the distillation distance, the better the generalization.  When $\mu_{\mathtt{aux}}\neq \mu$, KD on $\mu_{\mathtt{aux}}$ decreases $ \Phi_{\mu_{\mathtt{aux}},n_{\mathtt{aux}}}$ while causing additional generalization gap measured by TV divergence $\mathbb{TV}(\mu_{\mathtt{aux}},\mu)$. Compared to without KD, using a distillation dataset from a domain close to $\mu$ with small $\mathbb{TV}(\mu_{\mathtt{aux}},\mu)$ and reducing $ \Phi_{\mu_{\mathtt{aux}},n_{\mathtt{aux}}}$ during KD  can also improve the generalization 
(e.g., \add{when $ \Phi_{\mu_{\mathtt{aux}},n_{\mathtt{aux}}} + 2 \mathbb{TV}(\mu_{\mathtt{aux}},\mu) \leq \Phi_{\mu,n_{\mathtt{aux}}} $}).
We empirically verify \add{the effect of different distillation datasets} in \cref{sec:evaluation-results}.
(2) \textbf{Quality of local models.} The $\mathtt{ERR}(\sD_m, h_m)$ term shows that reducing the empirical risk of local models w.r.t local distributions $\mu_m$ improves the generalization of the global model. We verify in \cref{sec:evaluation-results} that a more powerful pretrained model, which results in higher quality local models, leads to better generalization. 
(3) \textbf{Sample complexity.} More empirical samples during training improve the generalization. 
We further discuss the effect of \textit{client heterogeneity} ${\hat d}_{\gH \Delta \gH} (\sD_m, \sD)$ (i.e., the empirical $\gH$ -divergence between two datasets) and \textit{number of classes} $k$ in \cref{app:generalization-add-results}.
\end{remark}

\vspace{-1mm}
\noindent\textbf{Personalized Models.} We show that personalized model $p_m$  can generalize well on $\mu$ if global model $g$ generalizes well on $\mu$ and $p_m$ has small prediction distance with $g$.
\vspace{-1mm}
{
\abovedisplayskip=1pt
\abovedisplayshortskip=1pt
\belowdisplayskip=1pt
\belowdisplayshortskip=1pt
 \begin{restatable}[Generalization bound of \name personalized model]{theorem}{pergen} 
\label{thm:main_distill_personal}
With probability at least $ 1- \delta$,
for every $p_m \in \gP_m , \forall m \in [M] $, and for every  $g \in\gG $,  we have $\operatorname{Pr}_{(x,y)\sim \mu}\left[\underset{y^{\prime}}{\arg \max }  p_m(x)_{y^{\prime}} \neq y\right]  \leq {2\E_{(x,y)\sim \mu }( 1- g(x)_y) }  +     2\frac{1}{n} \sum_{i=1}^n\min \left\{1,\|p_m(x)-g(x)\|_1\right\}  + 6\sqrt{\frac{\log(2/\delta)}{2n}} + {\mathcal{O}} \left({k^{3/2}}\left [ \max_{j} \operatorname{Rad}_n(\mathcal{P}|_j)  + \max_{j} \operatorname{Rad}_n(\mathcal{G}|_j) \right] \right)$.
\end{restatable}
}
\vspace{-1mm}
\begin{remark} The first term is the population risk of $g$ on $\mu$, which has been upper bounded by \cref{thm:main_distill}.  The second term is the prediction difference between $g$  and personalized models. \chulin{Therefore, the generalization  of personalized model is intrinsically related to the performance of global model.} In \cref{sec:evaluation-results},  we empirically show that moderately increasing the regularization strength $\lambda$ in (\ref{eq:per-obj}) could improve the generalization of $p_m$, by reducing such prediction distance. 
\end{remark}

 \vspace{-4mm}
\section{Convergence Guarantees of \name}\label{sec:convergenceguarantees}
 \vspace{-2mm}
In this section, we aim to provide the convergence analysis.
We outline the analysis challenges for \name, arising from the bi-level optimization between server distillation and local training, as well as the personalization regularized by the global model. Then, we present the convergence analysis for \name global model and personalized model.
For notation simplicity,  we will omit the frozen parameters $u$ and use $w/\theta_m/v_m$ to represent corresponding models.

 To convey the salient ideas, we consider full client participation (i.e., $|\gS_t|=M$) for convergence analysis following~\cite{reddi2021adaptive,ozkara2021quped}; thus, the stochasticity comes from mini-batch samplings during client and server training. Below, we first give several necessary assumptions.
 \vspace{-1mm}
\begin{assumption}(Smoothness).\label{asp:smooth}
 $\gL_m(\theta)$ is $L$-Lipschitz smooth $\forall m\in[M]$
and $\gR(\{\theta_m\}, w)$ is $L_R$-Lipschitz smooth.
\end{assumption}
\vspace{-4mm}
\begin{assumption}(Bounded Variance).\label{asp:stograd}
The stochastic gradients are unbiased 
and variance is bounded $\forall m\in[M]$:  
$\mathbb{E}\|\widetilde{\nabla} \gL_m \left(\theta \right) -  \nabla \gL_m \left(\theta \right) \|^2 \leq \sigma^2$,
$\mathbb{E}\|\widetilde{\nabla}_w \gR( \{\theta_m\} , w) -\nabla_w \gR( \{\theta_m\} , w )  \|^2 \leq \sigma_R^2$.
\end{assumption}
\vspace{-4mm}
\begin{assumption}(Bounded Diversity).\label{asp:boundeddiv}	
The variance of local gradients to global gradient is bounded
\add{$\frac{1}{M}\sum_{m=1}^M  \|   \nabla\gL_m( {w})  - \frac{1}{M}\sum_{i=1}^M \nabla \gL_i( {w})     \|^2 \leq  \bar \gamma $}.
\end{assumption}
\vspace{-4mm}
\begin{assumption}(Bounded Gradients).\label{asp:boundedgrad}
\add{The functions $ \gL_m, \gR, P_m, \forall m\in[M]$ have bounded gradients: $  \left\|  \nabla \gL_m(\theta)\right\|  \leq G$,
$  \left\|  \nabla_w \gR(\{\theta_{m}\},  w)\right\|  \leq G_R$,
$ \left\|  \nabla_{w} P_m(v_m, w)\right\|  \leq G_P$.
}
\end{assumption}
\vspace{-2mm}
We defer more discussions on the assumptions to \cref{app:convergence-add-results}. 
Next, we discuss the challenges and present the main results. All proofs are relegated to \cref{app:convergence-proof}.

\noindent\textbf{Global Model Convergence with Ensemble Distillation.}  
Despite the wide applications of knowledge distillation in FL~\cite{zhu2021data,lee2021preservation,zhang2021parameterized}, its convergence analysis is less explored.
To the best of our knowledge,  there is no convergence guarantee under server-side ensemble distillation~\cite{lin2020ensemble,chen2020fedbe,li2019fedmd,zhang2022fine}.
This lack of research 
is likely because (1) the complexity of bi-level optimization between server distillation for $w^{t}$ and client training for $\{\theta_m^{t}\}$, which incorporates two objectives (i.e., minimizing distillation loss and local loss respectively); (2) at each round, the global model is initialized by averaged local models before distillation, and local models are initialized by the global model before local training. Such mutual initializations intervene in the model updating trajectories of $w^{t}$ and $\{\theta_m^{t}\}$  w.r.t their training objectives, making the convergence even harder to analyze. 
\chulin{On the other hand, it has been empirically shown that ensemble distillation can improve the global model performance by incorporating diverse knowledge from clients (e.g., low $\gL(w^{t})$ measured on all clients' data)~\cite{lin2020ensemble,chen2020fedbe,li2019fedmd,zhang2022fine}. }
Therefore, we aim to \textit{understand the global model convergence w.r.t $\gL(w^{t})$ as a function of ensemble distillation}. 
To overcome the aforementioned challenges, we regard $\{\theta_m^{t}\}$ as the intermediate models to update $w^{t+1}$, and quantify the effects of local client training and server distillation on optimizing FL global objective:

{
\abovedisplayskip=1pt
\abovedisplayshortskip=1pt
\belowdisplayskip=1pt
\belowdisplayshortskip=1pt
 \vspace{-2mm}
 \begin{restatable}[Convergence of \name global model]{theorem}{globalconv}
\label{thm:convergence1}
 \chulin{
Let Assumptions  \ref{asp:smooth} to \ref{asp:boundedgrad} hold, and $ \eta_l = \frac{1}{EL\sqrt{T}} $, $\eta_g = \frac{1}{L_R RT}$, denote  $\bar w^{t,e}=\frac{1}{M}\sum_{m=1}^M{\theta_m^{t,e}}$, 
then the algorithm satisfies}
\chulin{
\begin{small}
\begin{equation}
     \sum_{t=0}^{T-1}\sum_{e=0}^{E-1} \frac{ \mathbb E \|\nabla \gL(\bar w^{t,e})\|^2}{ET} \leq \gO \Big(\frac{L\Delta_{\gL}+\psi_1}{\sqrt T} + \frac{\bar \gamma^2}{T} +  \frac{L^2 \psi_2}{T\sqrt T L_R^2E}\Big), \nonumber
\end{equation}
\end{small}
\vspace{-1mm}
where 
$\Delta_\gL =  \gL(w^{0} ) -  \gL(w^{T})$  ,  $\psi_1 = \frac{\sigma^2}{EM} + \frac{L(G^2+\psi_2)}{EL_R}$, and $\psi_2 =  4 \sigma_R^2 + 32 (3 G_R^2 + \frac{2\sigma_R^2}{R})/T^2 + 2G_R^2$. In particular, $\bar{w}^{t+1,0}=w^{t}$ and  $\bar{w}^{t+1,E-1}=\bar\theta^{t+1}$.}
\end{restatable}
 \vspace{-4mm}
 }
 \begin{remark} 
(1) \textbf{Convergence rate} is $\gO(1/\sqrt{T})$ as it is the dominant term,  matching the rate of the general FL non-convex settings of our interest~\cite{t2020personalized,ozkara2021quped}.
(2) \textbf{Local steps \& distillation steps.} With more local updating steps $E$ and distillation steps $R$, the terms $\psi_1$ and $\psi_2$ decrease. It means that a larger $E$ and $R$ can reduce the required communication rounds $T$ to converge, thus lowering communication costs.
(3) \textbf{Client heterogeneity} is reflected in $\bar\gamma$, whose effect can be mitigated by  larger $T$.
(4)  \chulin{\textbf{Ensemble distillation} is mainly reflected in  $\psi_2$ where $\sigma_R^2$ are inherent data sampling noise when using stochastic gradients~\cite{t2020personalized,fallah2020personalized}, and  $G_R$ is from the bounded gradient assumption for distillation.
The distillation gradient  can be small when the averaged logits of local models (teacher) and the logits of the global model (student) are close (See \cref{eq:nabla_r_closed-form} and more discussion in \cref{app:convergence-add-results}). 
Notably, the convergence bound remains valid for any distillation data, even if it is \textit{out-of-domain}.
}

\end{remark}

\noindent\textbf{Personalized Model Convergence.}  
Regarding personalization,  unlike \cite{t2020personalized}, to preserve generalization, the global model $w^{t}$ of  \name is not updated based on the personalized objective $P(v_m^{t}, w^{t})$.
Thus, it remains unclear \textit{how the global model $w^{t}$ learned from the ensemble distillation impacts the convergence of personalized models w.r.t $P(v_m^{t}, w^{t})$.} 
In \cref{thm:convergence2} (\cref{app:convergence-add-results}), we analyze such impacts and show the convergence rate of personalized models.

\vspace{-3mm}
\section{Experiments}
\vspace{-2mm}
We empirically compare \name to existing pFL methods. We defer the details of experiments and hyperparameter as well as the additional experimental results to \cref{app:ep-details}.

{
\begin{table*}
\vspace{-1mm}
\centering
\caption{\small Parameter-efficiency and averaged test accuracy across all clients' personalized models. \name achieves higher personalized  performance and generalization with a smallest $\#$ of trainable parameters. \textbf{bold}/\underline{Underline} fonts highlight the best/runner-up approach.}
\vspace{-4mm}
\label{tab:performance}
 \resizebox{0.98\textwidth}{!}{%
 \begin{tabular}{c c c c 
 *{12}{S[table-format=2.2]@{\tiny$\pm$}>{\tiny}S[table-format=2.2]<{\endcollectcell}
  }
}
\toprule
\multirow{2}{*}{\textbf{Algorithm}} &
\multirow{2}{*}{\makecell{\textbf{Personalized}\\\textbf{Params}}} &
\multirow{2}{*}{\makecell{\textbf{$\#$ Trained}\\\textbf{Params}}} &
\multirow{2}{*}{\makecell{\textbf{$\#$ Comm.}\\\textbf{Params}}} &
\multicolumn{7}{c}{\textbf{\cifar}} &  
\multicolumn{4}{c}{\textbf{\office}} &  
\multicolumn{4}{c}{\textbf{\chexpert}}   \\  
\cmidrule(lr){5-12} \cmidrule(lr){13-16} \cmidrule(lr){17-20}
&& &&  \multicolumn{2}{l}{\textbf{\localtest}} & \multicolumn{2}{l}{\textbf{\globaltest}}  & \multicolumn{2}{l}{{\cifarpointone}} & \multicolumn{2}{l}{{\cifarc}}  & \multicolumn{2}{l}{\textbf{\localtest}}  & \multicolumn{2}{l}{\textbf{\globaltest}}   & \multicolumn{2}{l}{\textbf{\localtest}}  & \multicolumn{2}{l}{\textbf{\globaltest}}  \\ \midrule
\standalone & Full model & 11.18$M$ & \textbf{0$M$} & 85.94 & 8.82 & 29.77 & 8.09 & 25.82 & 6.27 & 26.67 & 7.07 & 81.64 & 6.08 & 59.15 & 3.32  & 65.06      & 1.88 &  65.45 & 2.3    \\\midrule
\mtl~\cite{smith2017federated} & Full model & 11.18$M$ & 11.18$M$ & 86.24 & 8.45 & 29.46 & 8.33 & 25.64 & 6.42 & 26.4 & 7.29 & 81.82 & 5.53 & 59.25 & 2.84  & 65.15      & 1.95 & 65.48 & 2.3    \\
\textsc{FedAvg+FT}~\cite{yu2020salvaging} & Full model & 11.18$M$ & 11.18$M$* & 88.91	&5.71&	43.99&	9.57&	35.49&	8.02	&36.51&	8.36 & 79.42	& 5.62	&\underline{77.19}& 0.56 & 70.16	& 0.78 &	70.6 &	0.31   \\
\pfedme~\cite{t2020personalized} & Full model & 22.36$M$ & 11.18$M$ & 90.73 & 4.67 & 45.06 & 8.65 & 36.51 & 7.2 & 37.65 & 7.6 & 80.21 & 5.32 & 75.69 & 0.69  & 65.07      & 1.2  &  64.86 & 1.22 \\
\apfl~\cite{deng2020adaptive} & Full model & 22.36$M$ & 11.18$M$ & 90.74 & 4.75 & 43.92 & 9.18 & 35.83 & 7.5 & 36.51 & 7.94 & 81.24 & 4.51 & {76.98} & 1.39  & 68.98      & 1.04 &  68.96 & 1.1   \\
\ditto~\cite{li2021ditto} & Full model & 22.36$M$ & 11.18$M$ & 90.21 & 4.61 & \underline{53.82} & 6.35 & \underline{42.72} & 5.68 & {44.32} & 5.73 & 81.77 & 4.31 & 75.66 & 1.01   & 68.79      & 1.4  &  68.86 & 1.22 \\ \midrule 
\fedbn~\cite{li2021fedbn} & Batch norm. & 11.18$M$ & 11.17$M$ & 90.37 & 5.19 & 43.18 & 8.67 & 35.01 & 7.24 & 36.29 & 7.43 & 81.86 & 5.13 & 74.26 & 0.52& 68.74      & 1.17 & 68.83 & 1.08  \\
\fedalt~\cite{pillutla2022federated} & Input layer & 11.18$M$ & 6.45$M$ & 87.07 & 6.54 & 32.23 & 8.23 & 27.49 & 6.41 & 28.51 & 7.11 & 81.07 & 5.59 & 65.85 & 0.9  & 67.63      & 1.18 & 67.74 & 1.1  \\
\fedsim~\cite{pillutla2022federated} & Input layer & 11.18$M$ & 6.45$M$ & 87.93 & 6.25 & 33.07 & 8.16 & 28.21 & 6.41 & 29.15 & 7.16 & 82.45 & 5.03 & 67.66 & 0.82 & 67.49      & 1.32 &  67.54 & 1.24 \\
\lgfedavg~\cite{liang2020think} & Feat. extractor & 11.18$M$ & 0.005$M$ & 86.7 & 8.01 & 29.96 & 8 & 25.97 & 6.21 & 26.83 & 6.95 & 82.04 & 5.96 & 63.57 & 2.32  & 65.78      & 1.62 & 66.23 & 1.75\\
\fedrep~\cite{collins2021exploiting} & Output layer & 11.18$M$ & 11.17$M$ & 87.76 & 6.46 & 35.19 & 6.97 & 30.15 & 5.89 & 30.68 & 6.31 & 79.05 & 5.88 & 74.17 & 2.02 & 66.66      & 1.82  & 66.52 & 1.47 \\
\fedalt~\cite{pillutla2022federated} & Output layer & 11.18$M$ & 11.17$M$ & 89.68 & 5.4 & 40.68 & 7.3 & 33.61 & 6.12 & 34.3 & 6.5 & {83.24} & 3.96 & 70.62 & 1.46  & 68.27      & 1.3  &  68.36 & 1.31  \\
\fedsim~\cite{pillutla2022federated} & Output layer & 11.18$M$ & 11.17$M$ & 89.75 & 5.51 & 41.98 & 7.66 & 34.21 & 6.22 & 35.31 & 6.79 & 82.91 & 4.46 & 72.34 & 0.51  & 68.22      & 1.34 &  68.12 & 1.24\\
\fedalt~\cite{pillutla2022federated} & Adapter & 12.59$M$ & 11.18$M$ & 87.26 & 7.78 & 31.51 & 8.55 & 27.38 & 6.65 & 27.77 & 7.19 & 81.41 & 6.5 & 57.88 & 3.57  & {72.13}      & 1.34  & {74.67} & 1.57 \\
\fedsim~\cite{pillutla2022federated} & Adapter & 12.59$M$ & 11.18$M$ & 87.76 & 7.57 & 31.97 & 7.44 & 27.76 & 5.78 & 28.1 & 6.46 & 82.14 & 5.46 & 58.62 & 3.24 & 71.75      & 1.4  &  74.09 & 1.55 \\ 
{\namewokd} & Adapter & 2.82$M$ & 1.41$M$ & \underline{91.27} & 5.15 & 53.81 & 6.27 & 42.5 & 5.06 & \underline{44.45} & 5.48 & \underline{83.31} & 5.54 & 76.55 & 2.47  & \underline{76.77}      & 2.24 &  \underline{77.59} & 2.18 \\
{\name} & Adapter & \textbf{2.82$M$} & 1.41$M$ & \textbf{91.82} & 4.43 & \textbf{59.05} & 5.24 & \textbf{47.25} & 4.48 &\textbf{ 48.53} & 4.74 & \textbf{83.58} & 4.74 & \textbf{77.2 }& 1.63  & \textbf{76.98}         & 3.87 & \textbf{77.88} & 1.55\\\bottomrule
\multicolumn{18}{l}{\small *{\textsc{FedAvg+FT} requires full model communciation during \fedavg training and there is no communciation during local finetuning.}} \\
\end{tabular}%
}
\vspace{-4mm}
\end{table*}

}

\noindent\textbf{Data and Model.}
We use \cifar~\cite{krizhevsky2009learning}, \office~\cite{venkateswara2017deep}, and medical image data \chexpert~\cite{irvin2019chexpert}. 
We simulate pFL setting for (1) \textit{label Non-IID} using Dirichlet distribution $\operatorname{Dir}(\alpha)$~\cite{hsu2019measuring} with $\alpha=0.1/0.3$ on \cifar/\chexpert , creating different local data size and label distributions for $M$ clients; and 
(2) \textit{feature Non-IID} on \office by distributing the data from 4 domains (Art, Clipart, Product, and Real Word) to 4 clients respectively~\cite{sun2021partialfed}.   
We use $M=20$ for \cifar/\chexpert, and sample 40\% clients at every round following~\cite{chen2021bridging,lin2020ensemble}, and use full client participation for \office following~\cite{sun2021partialfed}. 
We use ResNet-18 pretrained on ImageNet-1K~\cite{russakovsky2015imagenet} for all datasets.
For \name\footnote{We follow \cite{pillutla2022federated} to implement Adapter, which includes prediction head.}, we use out-of-domain distillation dataset CIFAR-100 for \cifar, and use \cifar for \office/\chexpert.  

\noindent\textbf{Baselines.}
We evaluate full model pFL methods  \textsc{FedAvg+FT}~\cite{yu2020salvaging},
\ditto~\cite{li2021ditto}, \apfl~\cite{deng2020adaptive}, \mtl~\cite{smith2017federated}, \pfedme~\cite{t2020personalized}, 
and partial model pFL methods \chulin{with decoupled personalized/global parameters}, including \fedbn~\cite{li2021fedbn},
\lgfedavg~\cite{liang2020think}, \fedrep~\cite{collins2021exploiting}, \fedsim~\cite{pillutla2022federated}, \fedalt~\cite{pillutla2022federated}. We also include \namewokd, which is \name without Line~\ref{algoline:server-update} server-side knowledge distillation (i.e., using \fedavg to aggregate global adapter).
\chulin{Note that we use the \textit{same pretrained ResNet as initialization} for all methods for fair comparisons. }

\noindent\textbf{Evaluation Metrics.}
We report the averaged test accuracy (\textbf{pFL accuracy}) and standard deviation over all clients' \textit{personalized models}. For \chexpert, we report the AUC score since it is a multi-label classification task. 
We evaluate pFL accuracy mainly under two metrics: \localtest (i.e., clients' corresponding local test data) and \globaltest (i.e., the union of clients'  local test data), to study the \textit{personalized performance} and  \textit{generalization} (against label or covariate shifts), respectively.
In addition, for \cifar, we evaluate pFL generalization against distribution shifts on \cifarpointone~\cite{recht2018cifar}  and common
image corruptions (e.g. Blur, Gaussian Noise) on \cifarc~\cite{hendrycks2019robustness}.

\vspace{-1mm}
\subsection{Evaluation Results}\label{sec:evaluation-results}
\vspace{-2mm}
\noindent\textbf{\name is parameter-efficient.} 
ResNet-18 
model 
consists of 11.18 million (M) parameters, and the adapter has 1.41M (12.6\%) parameters.  
\cref{tab:performance} reports each client's \# trainable parameters and \#  communicated parameters to the server. We see that \name is most parameter-efficient by locally training two adapters and communicating one adapter. Most full model pFL requires training two full models (\pfedme, \apfl, \ditto), and sends one full model to the server. Partial model pFL requires training one full model and communicating its shared parameter. Note that adapter-based partial model pFL in \fedalt and \fedsim are more expensive than \name because they still need to train both a personalized adapter plus a shared full model (12.59M), and communicate the full model. Additional comparison under ResNet-34 shows similar conclusions in \cref{fig:modelarc}.  

\noindent\textbf{\name achieves competitive personalized performance and better generalization than baselines.} 
\cref{tab:performance} shows that even with the smallest number of trainable parameters, \name achieves the comparable personalized performance (+1.08\%, 0.34\%, 4.85\%  on \cifar, \office, \chexpert) and better generalization (+5.23\%, 4.53\%, 4.21\%, 0.22\%, 3.21\% on \cifar, \cifarpointone, \cifarc, \office, \chexpert). 
Specifically, \textbf{(a)} \namewokd already achieves favorable performance compared to the best baseline, which shows that the plug-in module adapter can adapt the pretrained model to FL data distributions, and  
personalized adapter can successfully encode both local knowledges (with local empirical risk) and generalized knowledge (with regularization). 
\textbf{(b)} \name outperforms \namewokd,  which shows  that KD improves the generalization of personalized models (\cref{thm:main_distill_personal}). \chulin{We present the convergence curves in   \cref{fig:convergence} (\cref{app:ep-more-results}) to show the learning performance from the convergence
perspective, where \name achieves the best convergence speed.}

To verify that such improvement of pFL is due to an improved global model (\cref{thm:main_distill}), we compare the performance of the \textit{global model} of \name to the global model of state-of-the-art methods in pFL (\pfedme, \apfl, \ditto) and generic FL  (\fedavg, \fedprox~\cite{li2020federated}, \feddyn~\cite{acar2020federated}, \feddf~\cite{lin2020ensemble}). \chulin{Note that \feddf~\cite{lin2020ensemble} also uses ensemble knowledge distillation for global model aggregation, but updates the full model.} \cref{tab:performance-globalmodel} shows that the generalization of \name \textit{global} model with adapter also outperforms baselines, and KD indeed improves our global model.  
\begin{table}[]
\vspace{-1mm}
\caption{Generalization comparison of the \textit{global} model from different  generic FL and pFL methods on \cifar.}
\vspace{-2mm}
\label{tab:performance-globalmodel}
\centering
 \resizebox{0.95\columnwidth}{!}{%
\begin{tabular}{lccccc}
\toprule
\textbf{Algorithm} &\textbf{Algorithm Type} & \textbf{Trained Params}  & \globaltest &  \cifarpointone & \cifarc\\ \midrule
\fedavg~\cite{mcmahan2016communication} &  generic FL  & Full  & 69.34& 	54.95&	57.07   \\
\fedprox~\cite{li2020federated} & generic FL &  Full &69.64& 	54.75& 	56.84   \\
\feddyn~\cite{acar2020federated} & generic FL & Full & 70.36 &	56.3	&55.91   \\
\feddf~\cite{lin2020ensemble} (w/ KD) & generic FL  & Full & 74.83	& 60.95	& 61.23   \\
\midrule
\pfedme~\cite{t2020personalized}& pFL   & Full & 68.25 & 52.55 & 56.33 \\
\apfl~\cite{deng2020adaptive} & pFL  & Full & 69.79 & 53.6 & 57.06 \\
\ditto~\cite{li2021ditto} & pFL  & Full  & 69.95 & 55.25 & 57.33 \\ 
\namewokd & pFL  & Adapter & 74.22 & 57.6 & 61.40 \\
\name & pFL  & Adapter & \textbf{76.77} & \textbf{62.5} & \textbf{64.47}\\
\bottomrule
\end{tabular}
}
\vspace{-8mm}
\end{table}

\noindent\textbf{Existing partial model pFL can have poor generalization to out-of-distribution shifts.}
As shown in \cref{tab:performance}, these methods, while showing promising personalized accuracy on \cifar and sometimes outperform full model pFL on \office and \chexpert by personalizing the right model component, \chulin{they significantly lag in generalizing to test-time distribution shifts. \textbf{(a)} Compared to full model pFL, the root causes of this inferior generalization in existing partial model pFL methods are twofold:} \textbf{(i)} a smaller number of shared parameters prevents them from effectively learning global information; \textbf{(ii)} personalized parameters can predominately encode local information for the partially personalized model.
\name circumvents such issues by regularization, which enforces personalized adapters to learn \textit{both} local and global information.
\chulin{\textbf{(b)} Moreover, the fact that \name even w/o KD has better generalization than existing partial pFL methods suggests that  updating the shared parameters globally via FL on heterogeneous data can compromise the pretrained feature exactor. Our findings indicate that maintaining frozen parameters, as done in \name without KD, is more effective in preserving the capabilities of the pre-trained model.}

\noindent\textbf{Adapter-based personalization methods are generally effective on \chexpert.}
\cref{tab:performance} shows that adapter-based personalization, including \fedalt, \fedsim, \name, are especially effective on the X-ray data \chexpert. This conclusion holds under different degrees of data heterogeneity $\operatorname{Dir}(0.3)$ and $\operatorname{Dir}(1)$ in \cref{tab:chexpert-dir1}.
It indicates that when adapting to FL domains that have a large domain gap for ImageNet pre-trained models, e.g., medical domains,  adapter personalization may be preferable to input/output/batch-norm pFL.

{
\begin{table}[]
\centering
\caption{\small Averaged test accuracy across personalized models with data heterogeneity degrees $\operatorname{Dir}(1)$ and $\operatorname{Dir}(0.3)$  on \chexpert. \name achieves best personalized performance and generalization.}
\vspace{-3mm}
\label{tab:chexpert-dir1}
 \resizebox{0.9\columnwidth}{!}{%
 \begin{tabular}{cc
 *{4}{S[table-format=2.2]@{\tiny${}\pm{}$}
  >{\tiny}S[table-format=2.2]<{\endcollectcell}
  }
}
\toprule
\multirow{2}{*}{\textbf{Algorithm}} &\multirow{2}{*}{\textbf{Personalization}} &  \multicolumn{4}{c}{\textbf{\localtest}} & \multicolumn{4}{c}{\textbf{\globaltest}} \\ \cmidrule(lr){3-6}  \cmidrule(lr){7-10} 
&&   \multicolumn{2}{l}{\textbf{$\operatorname{Dir}(1)$ }}   & \multicolumn{2}{l}{\textbf{$\operatorname{Dir}(0.3)$ }}  &
\multicolumn{2}{l}{\textbf{$\operatorname{Dir}(1)$}}  &
\multicolumn{2}{l}{\textbf{$\operatorname{Dir}(0.3)$ }}  \\ \midrule
\standalone & Full & 64.69 & 1.63 &65.06      & 1.88 & 65.32 & 1.7 & 65.45 & 2.3 \\\midrule
\mtl & Full & 65.18 & 1.95 & 65.15 & 1.95 & 65.67 & 1.72 & 65.48 & 2.3  \\
\pfedme & Full & 64.8 & 1.4 &65.07      & 1.2 & 64.85 & 1.25 & 64.86 & 1.22  \\
\apfl & Full & 69.21 & 1.23 & 68.98 & 1.04  & 69.21 & 1.05 & 68.96 & 1.1   \\
\ditto & Full & 68.65 & 0.82 & 68.79 & 1.4  & 68.72 & 0.58 & 75.55 & 0.34 \\\midrule
\fedbn & BN & 69.09 & 0.79 &  68.74      & 1.17 & 69.03 & 0.57 & 68.83 & 1.08  \\
\fedalt & Input & 67.74 & 0.85 & 67.63      & 1.18  & 67.88 & 0.6 &  67.74 & 1.1 \\
\fedsim & Input & 67.65 & 0.88 & 67.49      & 1.32  & 67.82 & 0.61 &  67.54 & 1.24 \\
\lgfedavg & Feat. extractor & 65.77 & 1.48 &  65.78 & 1.62 & 66.33 & 1.38 & 66.23 & 1.75  \\
\fedrep & Output & 66.42 & 1.62 &66.66   & 1.82 & 66.49 & 1.53 & 66.52 & 1.47 \\
\fedalt & Output & 68.31 & 0.79 & 68.27      & 1.3 & 68.41 & 0.47 &68.36 & 1.31\\
\fedsim & Output & 68.51 & 0.82 & 68.22  & 1.34  & 68.63 & 0.57 &  68.12 & 1.24  \\
\fedalt & Adapter & 72.52 & 0.99 & 72.13  & 1.34  & 74.79 & 1.21 & 74.67 & 1.57  \\
\fedsim & Adapter & 72 & 1.26 & 71.75      & 1.4 & 74.3 & 1.51 &  74.09 & 1.55  \\
\namewokd & Adapter & {77.45} & 1.21 &  76.77      & 2.24  & \textbf{78.02} & 1.36 & 77.59 & 2.18  \\
\name & Adapter & \textbf{77.47} & 1.54 &  \textbf{76.98}      & 1.81  & \textbf{78.02} &1.55 & \textbf{77.88} & 1.55   \\
\bottomrule
\end{tabular}%
}
\vspace{-5mm}
\end{table}
}

\noindent\textbf{Effects of KD.}
We use CIFAR-100 as the distillation dataset on \cifar, and \cref{fig:distillation} shows that more distillation steps and distillation data samples are better for pFL generalization. These results echo our theoretical analysis in \cref{thm:main_distill} that smaller KD optimization error $\Phi_{\mu_{\mathtt{aux}}, n_{\mathtt{aux}}}$ and a larger number of samples can tighten the generalization bounds. We also evaluate different distillation datasets, and  \cref{fig:distillation} shows that out-of-domain datasets (STL-10, CIFAR100) can improve generalization compared to the one without KD (None) by a margin, and achieve comparable performance compared to in-domain CIFAR10 validation data.
\textit{The flexibility of choosing distillation} datasets makes it practical for the server to leverage public data for KD. 

Another potential way to improve generalization is by moderately increasing regularization strength $\lambda$ for less personalization. However, \cref{fig:lambda} (\cref{app:ep-more-results})  show that an overly large $\lambda$ degrades the personalized performance, which matches the observation for $\ell_2$ regularization-based pFL methods in \cite{pillutla2022federated}.  Notably, KD does not have such a negative impact on personalized performance (in \cref{fig:distillation}).

\begin{figure}[t]
\setlength{\tabcolsep}{0pt}
\begin{subtable}{\linewidth}
\centering
\begin{tabular}{c@{}c@{}c@{}c@{}}
        \vspace{-3pt}\\
\includegraphics[height=0.24\linewidth]{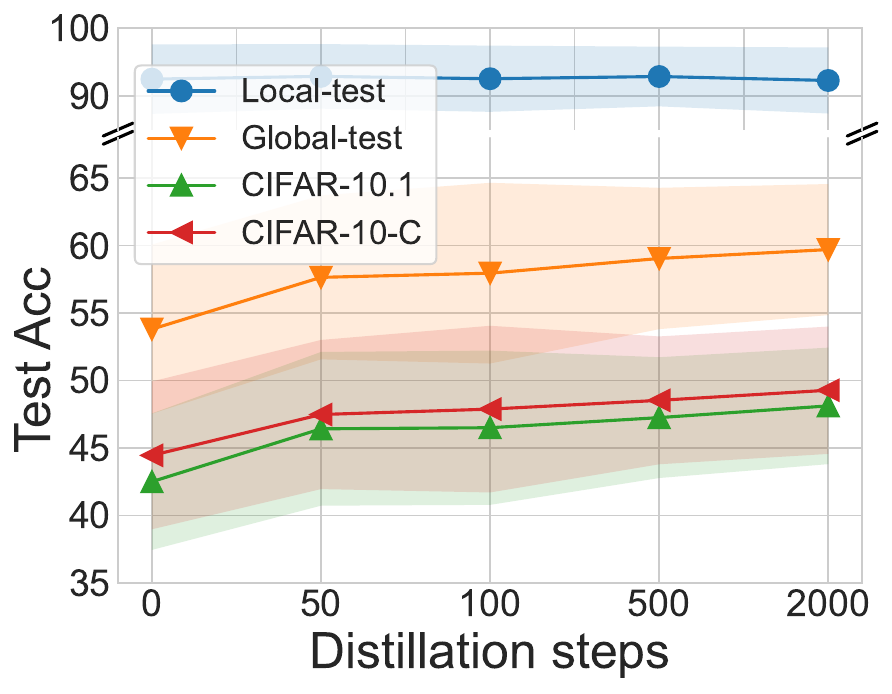}&
\includegraphics[height=0.24\linewidth]{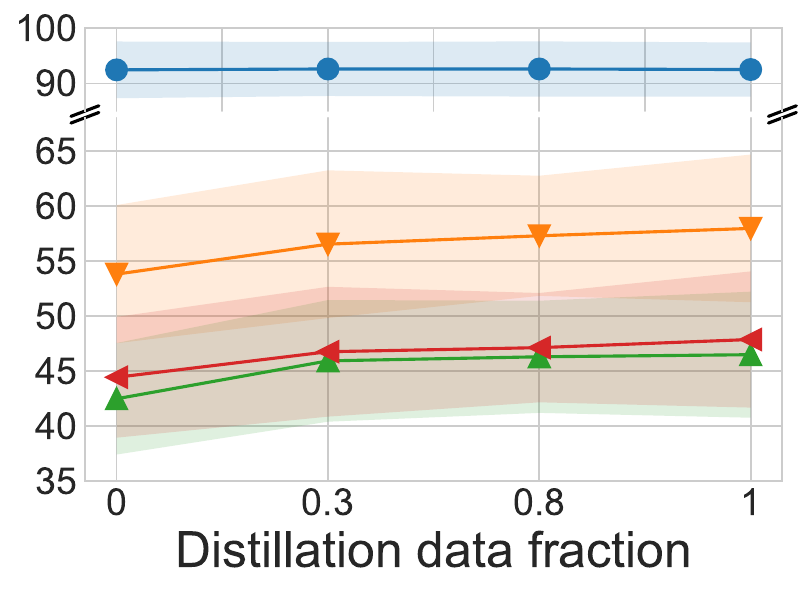}&
\includegraphics[height=0.24\linewidth]{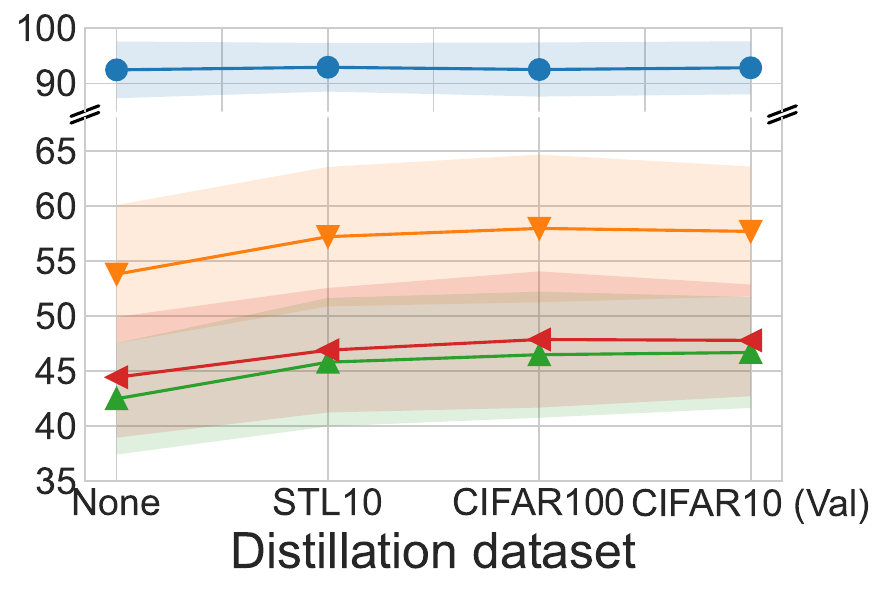}&
\\[-1.2ex]
\end{tabular}
\end{subtable}
\vspace{-2mm}
\caption{
\small Effect of KD on \name evaluated on \cifar.  More distillation steps and data samples lead to better generalization and out-of-domain distillation data  (STL-10, CIFAR-100)  achieve similar performance as in-domain (validation) data.   
}
\label{fig:distillation}
\vspace{-3mm}
\end{figure}

{
\begin{figure}[t]
\centering
\vspace{-6mm}
{
\begin{subtable}{\linewidth}
\centering
\resizebox{\linewidth}{!}{%
\begin{tabular}{c@{}c@{}c@{}c@{}}
        \vspace{-3pt}\\
\includegraphics[height=0.22\linewidth]{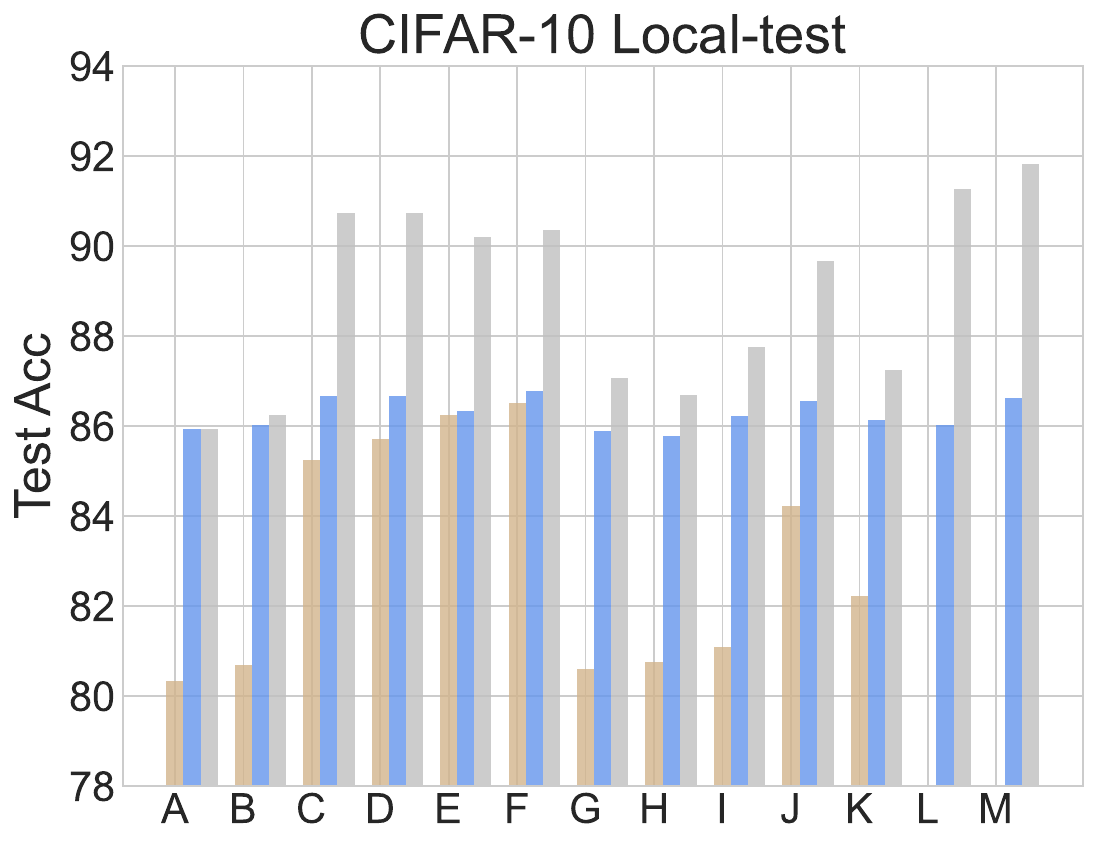} &
\includegraphics[height=0.22\linewidth]{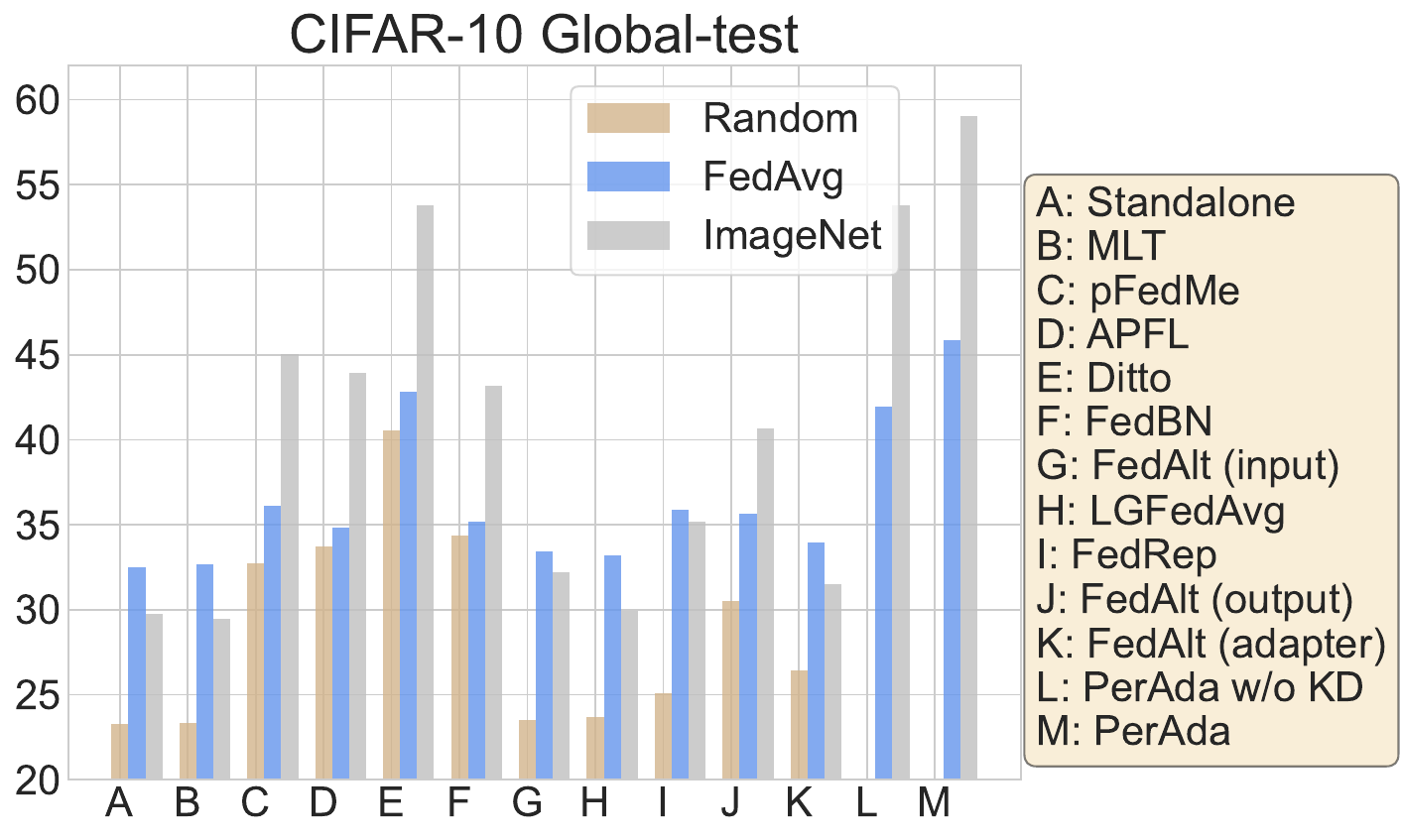}&
\\[-1.2ex]
\end{tabular}
}
\end{subtable}
}
\caption{\small Effect of  different initializations (Random, \fedavg model, and ImageNet pretrained model).}%
\label{fig:initialization}
\vspace{-6mm}
\end{figure}

}

\noindent\textbf{Effects of pretrained models.}
Starting personalization from a pretrained model, such as \fedavg global model, is commonly considered in pFL~\cite{pillutla2022federated,marfoq2022personalized}. Therefore, we first train a ResNet-18 global model on FL data from scratch using \fedavg and utilize it as initialization for pFL.
Results in \cref{fig:initialization} show that \name also achieves comparable personalized performance and higher generalization than baselines with \fedavg pretrained model. 
Moreover, ImageNet-pretraining leads to better generalization than \fedavg-pretraining for \name,  which echos \cref{thm:main_distill} that high-quality local models (enabled by good pretrained model) can further improve generalization.

\noindent\textbf{Utility under differential privacy guarantees.}
To further protect local data privacy, we train our method under \textit{sample-level} $(\epsilon,\delta)$ -differential privacy (DP)~\cite{dwork2014algorithmic} on \cifar with a ViT-S/16-224 model
\footnote{As batch normalization layer in ResNet creates dependencies between samples and violates DP, we use ViT model~\cite{wu2020visual} for DP experiments.}. Following~\cite{liu2022privacy}, we consider full client participation and  perform local training with DP-SGD~\cite{abadi2016deep} for \textit{both} personalized models and the global model (see experimental details in \cref{app:ep-details});
We set $\delta=10^{-5}$ and  report averaged $\epsilon$ across all clients and averaged pFL accuracy under \localtest. \cref{tab:dp} shows that  (1) \namewokd retains higher utility than full model personalization \ditto under reasonable privacy guarantees due to a smaller number of trainable parameters and the whole model is less impacted by DP noise. (2) KD with unlabeled \textit{public} data in \name can further improve the utility without consuming additional privacy budgets.

\begin{table}[H]
\vspace{-3mm}
\caption{\small \name retains high personalized utility under DP guarantee on CIFAR-10 with ViT-S/16-224 model.
}
\vspace{-3mm}
    \label{tab:dp}
    \centering
    \resizebox{1.02\linewidth}{!}{%
\begin{tabular}{cccccccc}
    \toprule
Algorithm & Personalization &  \multirow{1}{*}{\makecell{$\epsilon=\infty$ }} &
\multirow{1}{*}{\makecell{$\epsilon=5.99\pm	3.03  $  }} &
\multirow{1}{*}{\makecell{$\epsilon= 3.7 \pm 2.12   $ }} &
\multirow{1}{*}{\makecell{$\epsilon= 1.81 \pm  1.12 $ }} 
\\ \midrule
Ditto & Full  & \textbf{98.59} $\pm$ 1.63   & 76.76	 $\pm$  24.14   &  76.75 	$\pm$	24.13 &76.67		$\pm$ 24.12  \\
\namewokd & Adapter  &  97.69  $\pm$ 1.79 &  77.49		$\pm$ 21.21   & 77.32	$\pm$	21.16 &76.68		$\pm$ 21 \\
\name & Adapter &  98.08	$\pm$ 1.28 & \textbf{ 80.33}$\pm$	20.76  &\textbf{79.79}	$\pm$20.45   &  \textbf{77.83}$\pm$	19.58   \\
\bottomrule
\end{tabular}
}
\vspace{-2mm}
\end{table}

\section{Conclusion}
\vspace{-1mm}
We propose a pFL framework \name based on global/personalized adapter and knowledge distillation with convergence and generalization guarantees, and show that it reduces computation and communication costs and achieves higher personalized performance and generalization.

\newpage
\textbf{Acknowledgement.} 
The authors thank Zinan Lin, Wenxuan Bao, Liliang Ren, and anonymous reviewers for their valuable feedback and suggestions.
This work is partially supported by the National Science Foundation under grant No. 1910100, No. 2046726, No. 2229876, the Alfred P. Sloan Fellowship, the Amazon research award, and the eBay research award.

{
    \small
    \bibliographystyle{ieeenat_fullname}
    \bibliography{egbib}
}

\appendix
\onecolumn

{\large \textbf{Appendix}}

The Appendix is organized as follows:
\begin{itemize}[noitemsep,leftmargin=*]
  \item Appendix~\ref{app:ep-details} provides detailed setup and hyperparameters for experiments. 
    \item Appendix~\ref{app:ep-more-results} provides additional experiment results.
     \item Appendix~\ref{app:generalization-proof} provides the generalization analysis of \name and the full proofs for \cref{thm:main_distill} and \cref{thm:main_distill_personal}.
      \item Appendix~\ref{app:convergence-proof} provides the convergence analysis of \name and the full proofs for \cref{thm:convergence1} and \cref{thm:convergence2}.
\end{itemize}

\section{Experimental Details}\label{app:ep-details}

\subsection{Datasets and Model}

\begin{table}[ht]
\centering
\caption{Summary of datasets.}
\label{tab:data}
\resizebox{\linewidth}{!}{%
    \begin{tabular}{cccccccccccccccc}
    \toprule
     Dataset  & Task & $\#$ Training Samples & $\#$ Test Samples & $\#$ Validation Samples   & $\#$ Clients & Data Partition & $\#$ Classes  \\\midrule
    \cifar & image classification &  45000 & 10000 & 5000 &  20 & label-shift non-IID (synthetic) & 10 \\
     \office & image classification  &  12541 & 1656  &1391  &  4 & covariate-shift non-IID (nature) & 65 \\
     \chexpert & multi-label image classification & 180973 & 20099  & 22342 & 20& label-shift non-IID (synthetic) &  5  \\
    \bottomrule
\end{tabular}%
}
\end{table}
\paragraph{FL datasets}
We summarize our FL datasets in \cref{tab:data}. 
\begin{itemize}
    \item \textbf{\cifar}~\cite{krizhevsky2009learning} contains nature images for 10 classes, such as cat, bird, dog.  We simulate {label non-IID} on \cifar using Dirichlet distribution $\operatorname{Dir}(\alpha)$~\cite{hsu2019measuring} with $\alpha=0.1$, creating different local data size and label distributions for $M=20$ clients. 
    \item \textbf{\office}~\cite{venkateswara2017deep} contains images from four domains, i.e., Art, Clipart, Product, and Real Word.   All domains share the same 65 typical classes in office and home. We simulate the  feature non-IID by distributing the data from 4 domains to 4 clients, respectively~\cite{sun2021partialfed}.
    \item \textbf{\chexpert}~\cite{irvin2019chexpert} is a dataset of chest X-rays that contains 224k chest radiographs of 65,240 patients, and each radiograph is labeled for the presence of 14 diseases  as positive, negative, and uncertain. We map all uncertainty labels to positive ($\operatorname{U-Ones}$~\cite{irvin2019chexpert}).  We follow the original \chexpert paper to report the AUC score as a utility metric on five selected diseases, i.e., Cardiomegaly, Edema, Consolidation, Atelectasis, Pleural Effusion.  To create the label-shift non-IID on \chexpert, we view each possible multi-class combination as a ``meta-category'' and group all combinations that have less than 2000 training samples into a new meta-category, which results in a total of 19 meta-categories.  Then we use Dirichlet distribution $\operatorname{Dir}(\alpha)$ with $\alpha=0.3$ to create label-shift non-IID based on the 19 meta-categories for $M=20$ clients. Such FL data partition simulates a scenario where different hospitals (clients) have different majority diseases among their patients. Note that such meta-categories are only used to create FL non-IID data partition, and our utility metric AUC score is always calculated based on the five diseases, i.e., a 5-label image classification task.   
\end{itemize}
The number of samples for each dataset is shown in \cref{tab:data}, where 
we use a ratio of 9:1 to split the original training data into training data and validation data for each dataset. 

\paragraph{Distillation datasets}

We summarize our out-of-domain distillation dataset as below:
\begin{itemize}
    \item {\cifar}: we use 50k (unlabeled) samples from the CIFAR-100 training dataset.
    \item {\office} and \chexpert: we use 50k (unlabeled) samples from the  CIFAR-10 training dataset.
\end{itemize}

In \cref{fig:distillation}, we conduct the ablation study of distillation on \cifar. 
\begin{itemize}
    \item Distillation steps:  we fix the distillation data fraction as 1 and increase steps.
    \item Distillation data fraction: we fix the distillation steps as 100  and increase the data fraction.
    \item Distillation datasets: we fix the distillation steps as 100, data fraction as 1, and use different distillation datasets. Specifically, we use 100.5k  samples from the STL-10 unlabeled+training dataset, 50k samples from the CIFAR-100 training dataset, and 5k samples from the CIFAR-10 validation dataset. 
\end{itemize}

\paragraph{Evaluation datasets} 
As mentioned in \cref{sec:evaluation-results}, we evaluate pFL accuracy mainly under two metrics: \localtest (i.e., clients' corresponding local test data) and \globaltest (i.e., the union of clients'  local test data), to study the \textit{personalized performance} and  \textit{generalization} (against label or covariate shifts), respectively.
In addition, for \cifar, we evaluate pFL generalization against distribution shifts on \cifarpointone~\cite{recht2018cifar}  and \cifarc~\cite{hendrycks2019robustness}. 
\cifarpointone contains roughly 2,000 new test images that share the same categories as \cifar, and the samples in \cifarpointone are a subset of the TinyImages dataset~\cite{torralba200880}.
\cifarc~\cite{hendrycks2019robustness} is natural corruption benchmark for test-time distribution shits, containing common image corruptions such as Blur, Gaussian Noise, and Pixelate. It is generated by adding 15 common corruptions plus 4 extra corruptions to the test images in the \cifar dataset.

\paragraph{Model} 
We use a ResNet-18~\cite{he2016deep} pretrained on ImageNet-1K~\cite{russakovsky2015imagenet} for all tasks. We additionally evaluate \office on ResNet-34~\cite{he2016deep} pretrained on ImageNet-1K.  The pretrained models are downloaded from PyTorch~\cite{paszke2019pytorch}.

\cref{tab:model-resnet18-oh} and \cref{tab:model-resnet34-oh} show the detailed model architectures of ResNet-18 and ResNet-34 model used for personalization on \office, respectively.
We use the number of parameters in the corresponding layers and  the  number of parameters in the full model  to calculate the total number of $\#$ trainable parameters for different full model pFL and partial model pFL in \cref{fig:modelarc}. 

Since we use ResNet-18 for all datasets, the number of parameters  of different kinds of layers for \cifar and \chexpert are the same, except for the output layer. This is because different datasets have different numbers of classes, which decide the size of the output layer.
In \cref{tab:performance}, we report the parameters of the ResNet-18 model on \cifar, where the output layer consists of 0.0051M parameters.

\begin{table}[ht]
\centering
\caption{Summary of model architectures of ResNet-18 model used for personalization on \office.}
\label{tab:model-resnet18-oh}
\resizebox{0.6\linewidth}{!}{%
    \begin{tabular}{cccccccccccccccc}
    \toprule
    Type  &  Detailed layers & $\#$ Params. in the layers  \\\midrule
    Full model &  full model &11.21 M   \\
    Input layer &  1st Conv. layer & 4.73 M \\
     Feature extractor &  the model except last fully connected layer  &  11.16 M \\
    Batch norm  & batch normalization layers &  0.0078M \\ 
    Output layer &  last fully connected layer&  0.033 M \\
    Adapter &residual adapter modules & 1.44 M \\
    \bottomrule
\end{tabular}%
}
\end{table}

\begin{table}[ht]
\centering
\caption{Summary of model architectures of ResNet-34 model used for personalization on \office.}
\label{tab:model-resnet34-oh}
\resizebox{0.6\linewidth}{!}{%
    \begin{tabular}{cccccccccccccccc}
    \toprule
    Type  &  Detailed layers & $\#$ Params. in the layers  \\\midrule
    Full model &  full model & 21.32 M   \\
    Input layer &  1st Conv. layer & 9.78 M \\
     Feature extractor & the model except last fully connected layer  &  11.16 M \\
    Batch norm  & batch normalization layers &  0.015 M \\ 
    Output layer &  last fully connected layer&  0.033 M \\
    Adapter &residual adapter modules & 2.57 M \\
    \bottomrule
\end{tabular}%
}
\end{table}

\subsection{Training Details}
We tuned the hyperparameters according to the {personalized performance} evaluated on the local validation data. 
We use SGD as the client optimizer.
For each baseline method as well as our method, we tuned the (client) learning rate via grid search on the values \{5e-4,1e-3, 5e-3, 1e-2\} for \cifar and \chexpert, and \{5e-4, 1e-3, 5e-3, 1e-2, 5e-2\} for \office.
For \name, we use Adam as the server optimizer. We tuned the server learning rate via grid search on the values \{1e-5,1e-4, 1e-3, 1e-2\} for all datasets.  The strength of regularization $\lambda$ is selected from  \{0.1, 1\} following \cite{li2021ditto} and we use the same $\lambda$ for \name, \ditto, \pfedme.
For \pfedme, we use the inner step of $K=3$ as suggested in \cite{t2020personalized}.  
For \apfl, the mixing  parameter $\alpha$ is selected from \{0.1, 0.3, 0.5, 0.7\}. 
The final hyperparameters we used for \name are given in \cref{tab:our-hyperparam}.

\begin{table}[ht]
\centering
\caption{Hyperparameters of \name for each dataset.}
\label{tab:our-hyperparam}
\resizebox{0.5\linewidth}{!}{%
    \begin{tabular}{cccccccccccccccc}
    \toprule
    Hyperparameter  &  \cifar  &  \office & \chexpert  \\\midrule
    Batch size &   64  & 128 & 256 \\
    Clients per round  & 8  & 4 & 8\\
    Local epochs &  10 & 1 &  1 \\
    $\#$ training rounds  &   200&  100 & 30\\
    Regularization strength $\lambda$ & 1 & 0.1 & 1  \\
     Client learning rate  &  0.01 & 0.05 &  0.01\\
    Server learning rate  &   1e-3   &  1e-4 &  1e-5 \\
     Distillation step  &  500 &  100 & 50 \\
     Distillation batch size &   2048 & 256& 128 \\
    \bottomrule
\end{tabular}%
}
\end{table}

\subsection{Experimental Setups for DP Experiments}

Since the batch normalization layer in ResNet-18 requires computing the mean and variance of inputs in each mini-batch, creating dependencies between training samples and violating the DP guarantees, it is not supported in differentially private models.
Thus, we turn to conduct DP experiments with a ViT-S/16-224 model~\cite{wu2020visual}, which is pretrained on ImageNet-21k~\cite{russakovsky2015imagenet}. We download the pretrained model from Hugging Face~\cite{wolf-etal-2020-transformers}.  

Following~\cite{liu2022privacy}, we consider full client participation and  perform local training with DP-SGD~\cite{abadi2016deep} for personalized models and the global model. 
On \cifar, the local epoch is 1, and we run all methods for 10 communication rounds.
We tuned the (client) learning rate via grid search on the values \{0.01, 0.05,0.1, 0.2, 0.3 \} for \ditto, \namewokd, and \name. The optimal learning rate for \ditto, \namewokd, and \name are 0.05, 0.1, and 0.2, respectively. 
For \name, we set the distillation  batch size as 32. We select the sever learning rate from \{0.005, 0.003, 0.001\}, and the optimal server learning rate is 0.005. 

We set the DP parameter $\delta=10^{-5}$ and evaluate the averaged pFL accuracy under \localtest. 
We set the noise level $\sigma$ as $0.8,  1, 1.5$ for DP-SGD training  to obtain the privacy budgets $\epsilon=5.99\pm	3.03,  3.7 \pm 2.12, 1.81 \pm  1.12 $  used in \cref{tab:dp}, respectively. 
Under each privacy budget,  we tuned the clipping threshold  via grid search from \{1, 2, \ldots, 10 \} for each method.

\section{Additional Experimental Results and Analysis}\label{app:ep-more-results}
In this section, we provide additional  experimental results and analysis, including 
(1) Convergence analysis; 
(2) analysis of pFL performance under different model architectures \office;
(3) pFL performance under different data heterogeneity degrees on  \chexpert;
(4) generalization comparison of the global model of different pFL methods;
(5) effect of the pretrained model;
(6) effect of regularization strength $\lambda$.

\paragraph{Convergence}

{
\begin{figure*}[t]
\centering
\begin{tabular}{l}
\includegraphics[width=0.9\linewidth]{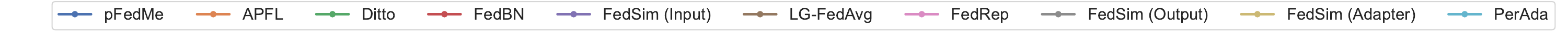}
\vspace{-8pt}\\
\end{tabular}
{
\begin{subtable}{\linewidth}
\centering
\resizebox{\linewidth}{!}{%
\begin{tabular}{c@{}c@{}c@{}c@{}c@{}c@{}c@{}c@{}}
        \vspace{-3pt}\\
\includegraphics[height=0.4\linewidth]{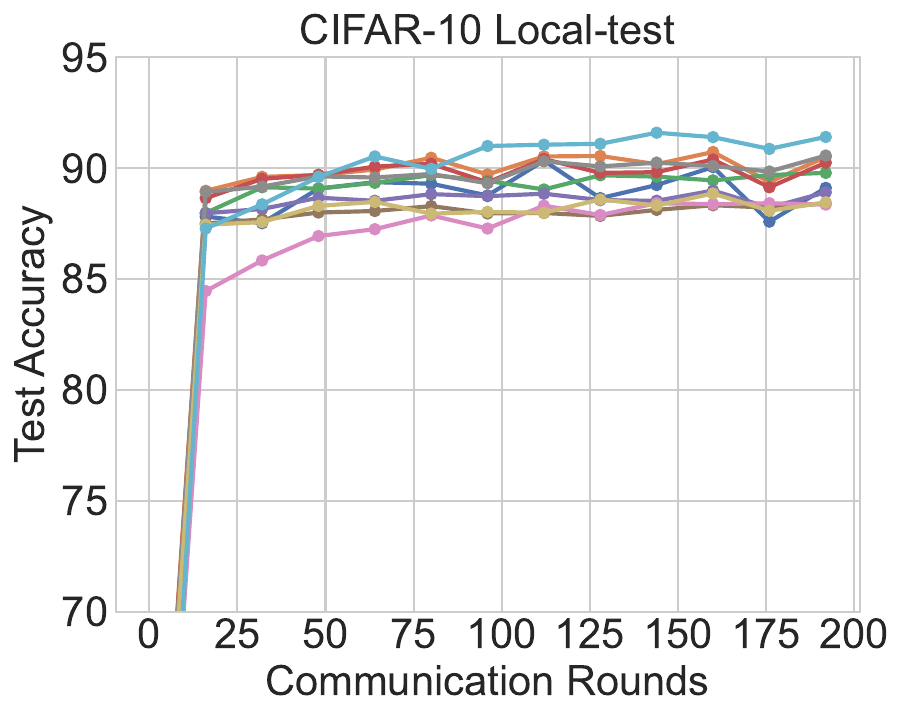} &
\includegraphics[height=0.4\linewidth]{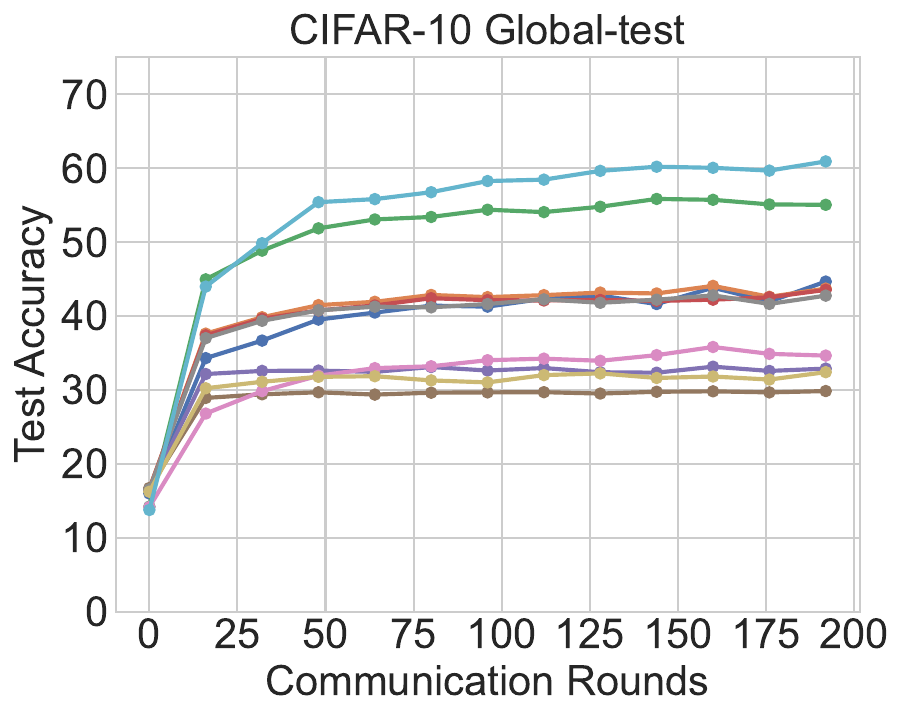}&
\includegraphics[height=0.4\linewidth]{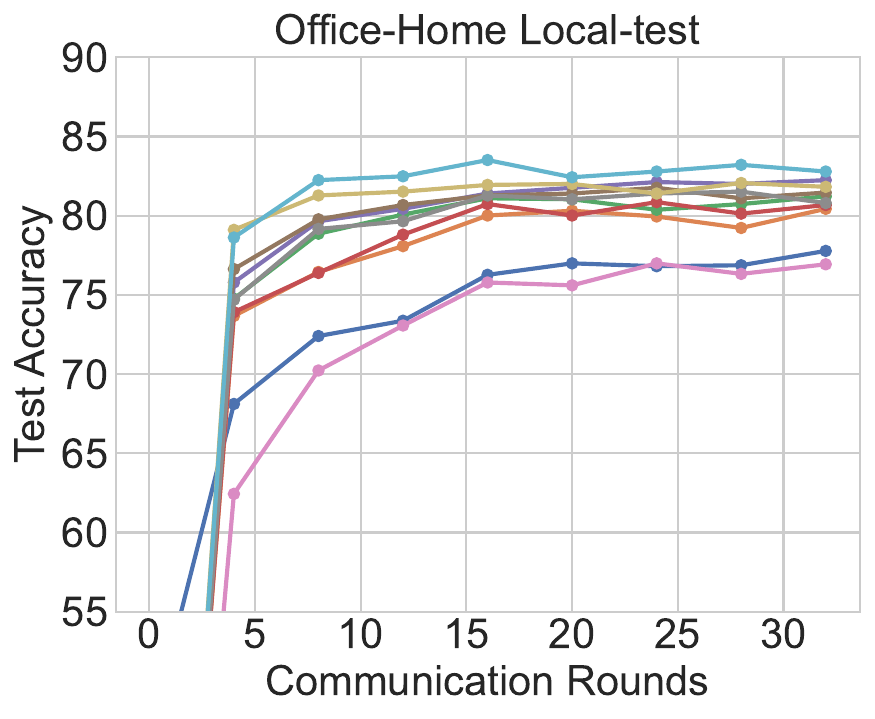} &
\includegraphics[height=0.4\linewidth]{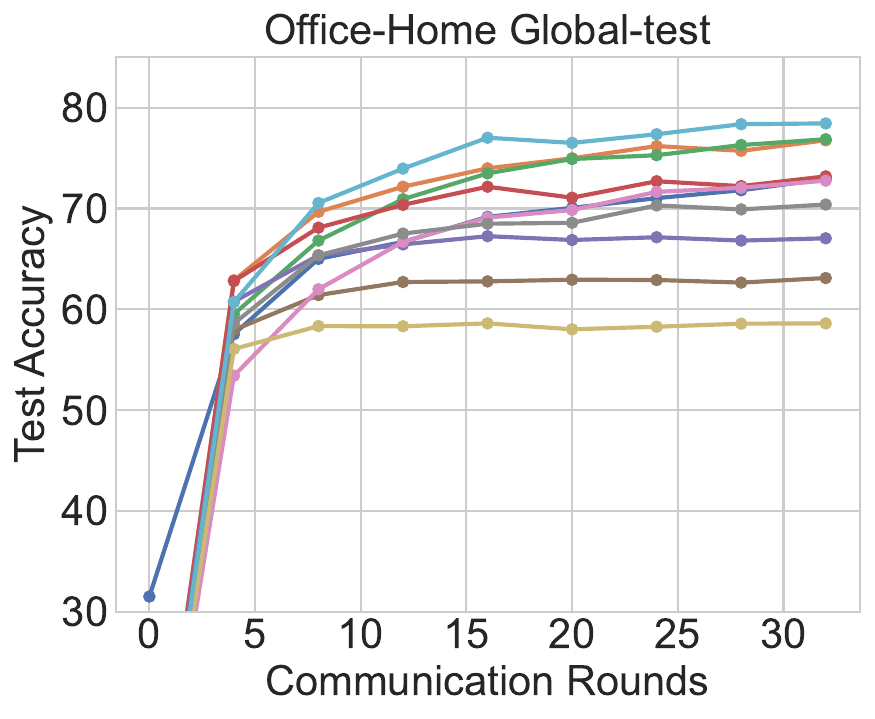}&
\includegraphics[height=0.4\linewidth]{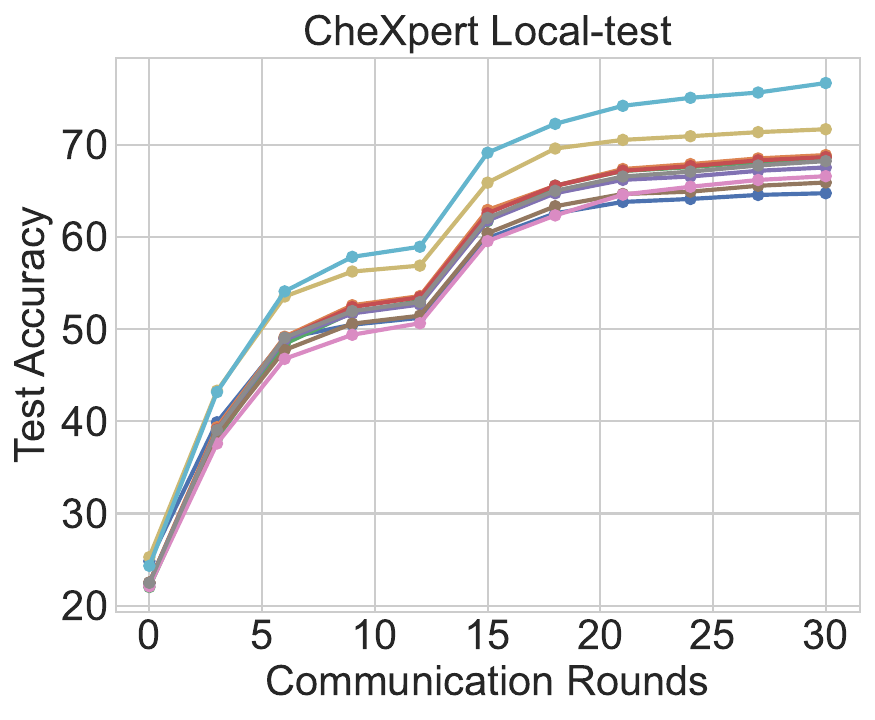} &
\includegraphics[height=0.4\linewidth]{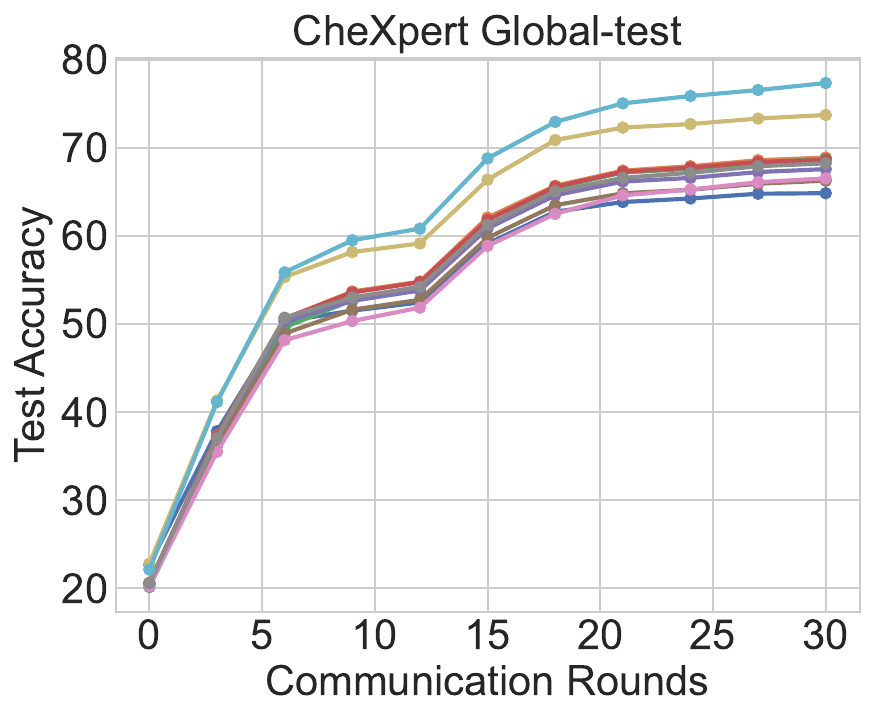}&
\\[-1.2ex]
\end{tabular}
}
\end{subtable}
}
\vspace{-4mm}
\caption{\small \chulin{Averaged test accuracy of personalized models from participating clients at each communication round.}}%
\label{fig:convergence}
\vspace{-6mm}
\end{figure*}

}
We present the learning performance from the convergence perspective in   \cref{fig:convergence}, where we report the averaged test accuracy of personalized models from the participated clients at each communication round. It shows that \name achieves the best convergence speed and converges to a higher personalized performance (local-test) and generalization performance (global-test).

\paragraph{Performance under different model architectures (ResNet-18 and ResNet-34) on \office.}

\cref{fig:modelarc} shows the performance of different pFL under ResNet-18 and ResNet-34. Cross different network architecture, \name is able to achieve the best personalized performance and generalization with the fewest number of trainable parameters.  For larger model, the number of updated parameters difference between full model personalization and our adapter personalization will be larger, reflecting our efficiency.

\paragraph{Performance under different data heterogeneity degrees on  \chexpert.}
\cref{tab:chexpert-dir1} shows under different data heterogeneity degrees $\operatorname{Dir}(1)$ and $\operatorname{Dir}(0.3)$  on \chexpert, \name achieves the best personalized performance and generalization. It also verifies that adapter-based personalization methods, including \fedalt, \fedsim, \name are especially effective on the X-ray data \chexpert.

\paragraph{Generalization comparison of the global model of different pFL methods.}
\cref{tab:performance-globalmodel} compare the generalization performance of the global model in  our method to the global model in other full model  pFL methods (\pfedme, \apfl, \ditto)  and generic FL methods (\fedavg, \fedprox~\cite{li2020federated}, \feddyn~\cite{acar2020federated}, \feddf~\cite{lin2020ensemble}) on \cifar.  
\mtl and partial model pFL methods are excluded from the compression because they do not train a complete global model. 
We use the same distillation dataset and distillation steps and data size for \feddf and \name to ensure a fair comparison. 

The results show that the global model of  \name outperforms these baselines, which verifies that KD improves our global model, and the improved performance of personalized models is due to a well-generalized global model.

\paragraph{Effect of pretrained models.}
Starting personalization from a pretrained model, such as \fedavg model~\cite{pillutla2022federated,marfoq2022personalized}, is common in pFL, so we report the results with \fedavg pretrained model (on FL data from scratch) for all methods\footnote{\fedsim is omitted here because its results are similar to \fedalt~\cite{pillutla2022federated}} on \cifar. The results in \cref{fig:initialization} show that \name also achieves comparable personalized performance and higher generalization than baselines with \fedavg pretrained model. 
Moreover, \cref{thm:main_distill} shows that high-quality local models (enabled by good pretrained model) can further improve generalization. Here, we use ImageNet as an example of high-quality pretrained models, which leads to even higher personalized performance and generalization for \name.
Additionally,  pretrained models lead to significantly higher pFL accuracy than random initialization for all existing methods; therefore, leveraging a pretrained model,  which is often available for modern deep neural networks~\cite{bommasani2021opportunities}, is practical and beneficial not only for \name but also for existing pFL methods.

\paragraph{Effect of $\lambda$.}

{
\begin{figure}[t]
\vspace{-5mm}
\centering
{
\setlength{\tabcolsep}{0pt}
\begin{subtable}{0.8\linewidth}
\centering
\resizebox{\linewidth}{!}{%
\begin{tabular}{c@{}c@{}c@{}c@{}}
        \vspace{-3pt}\\
\includegraphics[width=\linewidth]{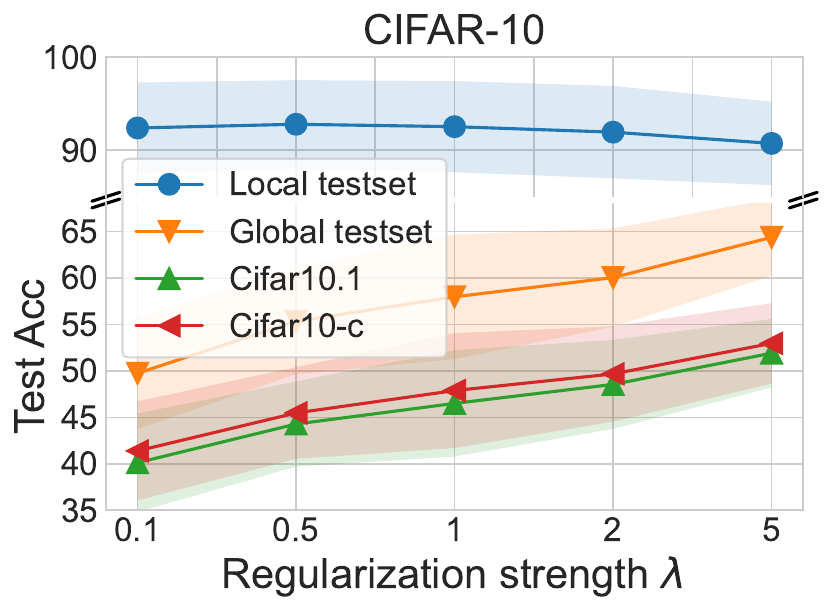}&
\includegraphics[width=\linewidth]{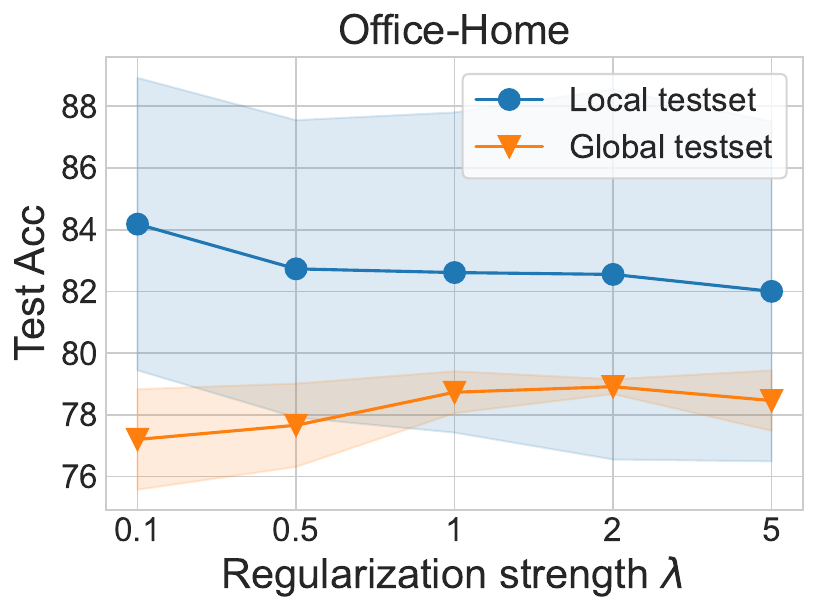}
\\[-1.2ex]
\end{tabular}
}
\end{subtable}
}
\caption{\small {Effect of $\lambda$ on \name} on \cifar and \office.}%
\label{fig:lambda}
\end{figure}

}

Results on \cifar and \office in \cref{fig:lambda} shows that moderately increasing regularization strength $\lambda$ can improve generalization, but it also degrades the personalized performance, which matches the observation for $\ell_2$ regularization-based pFL methods in \cite{pillutla2022federated}.

\newpage
\section{Generalization Analysis}\label{app:generalization-proof}
We give the discussions and analysis for our generalization bounds. The outline of this section is as follows:
\begin{itemize}[noitemsep,leftmargin=*]
    \item \cref{app:generalization-add-results} provides more discussions on  \cref{thm:main_distill}.
    \item \cref{app:generalization-preliminary} provides the peliminaries for generalization bounds and introduces several useful lemmas.  
     \item \cref{app:proof-g-model-generalization} provides the proofs for generalization bounds of global model in \cref{thm:main_distill}.
        \item \cref{app:proof-per-model-generalization} provides the proofs for generalization bounds of personalized model in \cref{thm:main_distill_personal}.
\end{itemize}

\subsection{Additional Discussion}\label{app:generalization-add-results}
\paragraph{Additional Discussion on \cref{thm:main_distill}.} From \cref{thm:main_distill}, we can have the additional observations:
(i) \textbf{Client heterogeneity.} Larger heterogeneity, i.e., higher distribution divergence  ${\hat d}_{\gH \Delta \gH} (\sD_m, \sD)$ between local and global datasets, could undermine the generalization  of  $g$, echoing the implications in~\cite{lin2020ensemble,zhu2021data} 
(ii) \textbf{Number of classes}. The smaller number of classes $k$ is favorable to generalization, as the classification task with fewer classes is easier to learn. We note that previous FL generalization bounds~\cite{lin2020ensemble,zhu2021data,marfoq2022personalized} are limited to binary classification cases.

\subsection{Peliminaries for Generalization Bounds}\label{app:generalization-preliminary}
Here we introduce several existing definitions and lemmas from learning theory. 
\begin{lemma}[Empirical Rademacher complexity~\cite{rademacher}]\label{lm:rademacher}
$\gG $ be a set of functions $\gZ \rightarrow [a,b] $,	$\forall \delta>0$.  Let $Z_1, \ldots, Z_n$ be i.i.d. random variables on $\gZ$ following some distribution $P$. The \textit{empirical Rademacher complexity} of $\gG$ with respect to the sample $(Z_1, \ldots, Z_n)$ is
\begin{equation}
	\widehat{\Re}_S(\gG):= \E_\sigma \left[\sup_{g \in \gG} \frac{1}{n} \sum_{i=1}^n \sigma_i g(x_i)\right]
\end{equation}
where $\sigma = (\sigma_1, \ldots, \sigma_n)^\top $ with $\sigma_i \sim \operatorname{unif}\{-1,1\}$, which is are known as Rademacher random variables.

Moreover, with probability at least $1-\delta$, we have w.r.t the draw of $S$ that 
	\begin{equation}
		\forall g \in \gG, \E[g(\gZ )] \leq \frac{1}{n} \sum_{i=1}^n g(x_i) + 2 \widehat{\Re}_S(\gG) + 3(b-a) \sqrt{\frac{\log(2/\delta)}{2n}}
	\end{equation}
\end{lemma}

\begin{definition}
[Risk~\cite{ben2010theory}] 
\label{def:vareps}
We define a domain as a pair consisting of a distribution $\mu_S$ on inputs $\gX$  and a labeling function $h_S^*: \gX \to \Delta^k$. The probability according to the distribution $\mu_S$ that a hypothesis h disagrees with a labeling function $h_S^*$ (which can also be a hypothesis) is defined as
	\begin{equation}
\varepsilon_{\mu_S} (h)		=	\varepsilon_{\mu_S} (h,h_S^*)= \E_{(x,y)\sim \mu_S } |h(x)_y - h_S^*(x)_y|
	\end{equation}
\end{definition}

\begin{definition}
[$\gH$ -divergence~\cite{ben2010theory}]  
Given a domain  $\gX$ with $\mu$ and $\mu'$ probability distributions over $\gX$, 
let $\gH$ be a hypothesis class on  $\gX$  and denote by $I(h)$ the set for which $h\in \gH$ is the characteristic function; that is,
   where $(x,y) \in I(h) \Leftrightarrow  h(x)_y =1$. The $\gH$ -divergence between $\mu$ and $\mu'$ is
\begin{equation}
 d_{\gH \Delta \gH} (\mu, \mu') = 2 \sup_{h \in \gH } |\Pr_{\mu}(I(h)) - \Pr_{\mu'}(I(h)) | 
\end{equation}
\end{definition}

\begin{lemma}
[Domain adaptation~\cite{ben2010theory}]  
\label{lm:domainadapt}
Let $\gH$ be a hypothesis space on $\gX$ with VC dimension $d$. Considering the distributions $\mu_S$ and $\mu_{T}$. 
If $\mathcal{D}_S'$ and $\mathcal{D}_T'$  are samples of size $n$ from   $\mu_S$ and $\mu_{T}$  respectively and ${\hat d}_{\gH \Delta \gH} (\mathcal{D}_S', \mathcal{D}_T', n)$ is the empirical $\gH$ -divergence between samples, then
for every $h \in  \gH$ and any $\delta \in (0,1)$, with probability at least $1-\delta$ (over the choice of samples) , there exists,
$$
\varepsilon_{\mu_T}(h) \leq  \varepsilon_{\mu_S}(h) +
 \frac{1}{2} {\hat d}_{\gH \Delta \gH} (\mathcal{D}_S', \mathcal{D}_T') + 4\sqrt{\frac{2d \log (2n)+ \log(2/\delta)}{n}} + \lambda
$$
where $\lambda = \varepsilon_{\mu_T}(h^*) +  \varepsilon_{\mu_S}(h^*)  $ and  $h^* := \arg\min_{h \in\gH } \varepsilon_{\mu_T} (h) +  \varepsilon_{\mu_S} (h)$ corresponds to ideal joint hypothesis that minimizes the combined error. 
\end{lemma}

\subsubsection{Useful Lemmas} \label{app:deferred-generalization-lemmas}
Then, we introduce several useful lemmas.
\begin{lemma}\cite{hsu2021generalization}\label{lm:hsuphi}
For any $v \in \mathbb{R}^k$ and $y \in [k]$,
$$2\left(1- v \right)_y \geq \mathds{1} \left[y \neq \underset{i}{\arg \max}  v_i\right].$$
\end{lemma}
\begin{proof}
Let $v \in \mathbb{R}^k$ be given, and consider two cases. For the first case, if $y=\underset{i}{\arg \max}  v_i$, then $v \in [0,1]^k$  implies
$2(1-v) \geq 0 =\mathds{1} \left[y \neq \underset{i}{\arg \max}  v_i\right]$. 
For the second case, if $y\neq \underset{i}{\arg \max}  v_i$, then $v_y \leq 1/2 $  and 
$2(1-v) \geq 1 =\mathds{1} \left[y \neq \underset{i}{\arg \max}  v_i\right]$.
Combining the two cases together, we prove the lemma. 
\end{proof}

\begin{lemma}  \label{lm:phi_rad}
For any functions $\mathcal{H}$ with $\mathcal{H} \ni h: \mathcal{X} \rightarrow \mathbb{R}^k$, since $\gH $ takes values in $ \mathbb{R}^k$, let $H|_j$ denote the Rademacher complexity of each class $j$, 
$$
\operatorname{Rad}_n\left(\left\{(x, y) \mapsto 1- h(x)_y: h \in \mathcal{H}\right\}\right)=\mathcal{O}\left(\sqrt{k} \max_{j} \operatorname{Rad}_n(\mathcal{H}|_j)\right)
$$
where $ \max_{k} \operatorname{Rad}_n(\mathcal{H}|_k)$ is the worst-case per class Rademacher complexity. 
\end{lemma}
\begin{proof}
The proof follows from a multivariate Lipschitz composition lemma for Rademacher complexity due to  \citep[Theorem 1]{foster2019vector}; it remains to prove that $v \mapsto \psi (v)_y$ is $1$-Lipschitz with respect to the $\ell_{\infty}$ norm for any $v \in \mathbb{R}^k$ and $y \in [k]$. 
$$ 
\frac{\mathrm{d}}{\mathrm{d} v_y} \psi (v)_y  = \frac{\mathrm{d}}{\mathrm{d} v_y} (1-v_y)= -1 , \quad 
\frac{\mathrm{d}}{\mathrm{d}  v_{i \neq y}} \psi (v)_y = \frac{\mathrm{d}}{\mathrm{d}  v_{i \neq y}} (1- v_y )=0
$$
and therefore $\|\nabla \psi (v)_y \|_1 = 1 $ and thus, by the  mean value theorem, for any $u \in \mathbb{R}^k$ and $v \in \mathbb{R}^k$, there exists $z \in [u, v]$ such that 
$$| \psi (v)_y - \psi (u)_y| =| \left\langle  \nabla \psi (z)_y, v- u  \right \rangle  \leq \|v-u\|_{\infty} \|  \nabla \psi (v)_y  \|_1 \leq  \|v-u\|_{\infty}.$$
In particular, $v \mapsto \psi (v)_y $ is 1-Lipschitz with respect to the $\ell_\infty$ norm. Applying the aforementioned Lipschitz composition rule \citep[Theorem 1]{foster2019vector},
$$
\operatorname{Rad}_n\left(\left\{(x, y) \mapsto 1- h(x)_y: h \in \mathcal{H}\right\}\right)=\operatorname{Rad}_n\left(\left\{(x, y) \mapsto \psi(h(x))_y: h \in \mathcal{H}\right\}\right)=\mathcal{O}\left(\sqrt{k} \max_{j} \operatorname{Rad}_n(\mathcal{H}|_j)\right)
$$
\end{proof}

\begin{lemma}\label{lm:fm_rad}
For any functions $\mathcal{H}_m$ with $\mathcal{H}_m \ni h_m: \mathcal{X} \rightarrow \mathbb{R}^k$ with any $m \in [M]$, and  $h \in \gH $ where $h(x):=\frac{1}{M}\sum_{m=1}^M h_m (x)$ for any $x \in \mathcal{X}$
\begin{equation}
\operatorname{Rad}_n(\mathcal{H}|_j) 	= \frac{1}{M} \sum_{m=1}^M \operatorname{Rad}_n(\mathcal{H}_m|_j)	
\end{equation}
\end{lemma}
\begin{proof}
\begin{align*}
\operatorname{Rad}_n(\mathcal{H}|_j) &=\frac{1}{n} \mathbb{E}_\epsilon \sup _{h \in \mathcal{H}} \sum_{i=1}^n \epsilon_i h\left(x_i\right)_{j} \\
&= \frac{1}{n} \mathbb{E}_\epsilon \sup _{h_1, \ldots, h_M \in \mathcal{H}} \sum_{i=1}^n \epsilon_i \left(\frac{1}{M}\sum_{m=1}^M h_m (x_i)\right)_{j}  \\ 
& =\frac{1}{M} \sum_{m=1}^M  \frac{1}{n} \mathbb{E}_\epsilon \sup _{h_m \in \mathcal{H}} \sum_{i=1}^n \epsilon_i h_m (x_i)_{j} \\
& =\frac{1}{M} \sum_{m=1}^M \operatorname{Rad}_n(\mathcal{H}_m|_j)	
\end{align*}
\end{proof}

\subsection{Proofs for Generalization Bounds of Global Model \cref{thm:main_distill}} \label{app:proof-g-model-generalization}

\paragraph{Overview}
Recall the definition of distillation distance:
 \begin{align}\label{eq:def_phi}
\Phi_{\mu, n} (h_1,\ldots, h_M; g ):= \frac{1}{n} \sum_{i=1}^n \|  g(x_i) -  \frac{1}{M}\sum_{m=1}^M h_m(x_i) \|_1 
\end{align}
which measures the output difference between the global model and the ensemble of local models.
The server distillation (Line \ref{algoline:server-update} in \cref{algo:fl}) essentially finds the global model $g$ with a small distillation distance $\Phi_{\mu_{\mathtt{aux}}, n_{\mathtt{aux}}}$, meaning that its outputs are close to the ensemble outputs of local models $f_1,\ldots, f_M$ on the out-of-domain distillation dataset $\sD_{\mathtt{aux}}$. 

For the generalization bounds of the global model,  we aim to show  $g$ can have good generalization bounds on $\mu$ with KD if it (1) distills knowledge accurately from teachers $\{f_m\}$ and (2) the teachers $\{f_m\}$ performs well on their local distributions $\{\mu_m\}$. 
To sketch the idea, by  \cref{lm:hsuphi}, we can upper bound error probabilities of $g$ with the expected distillation distances and errors of local models (i.e., teachers) on $\mu$: 
\begin{align}
&{\operatorname{Pr}_{(x, y)\sim\mu}\left[\underset{y^{\prime}}{\arg \max } g(x)_{y^{\prime}} \neq y\right]   = \mathbb{E}_{(x, y)\sim\mu}  \mathds{1} \left[\underset{y^{\prime}}{\arg \max } g(x)_{y^{\prime}} \neq y\right]  } \\
&\leq {2  \underbrace{\underset{x\sim \mu}{\mathbb{E}}\left\|g(x) -\frac{1}{M}\sum_{m=1}^M h_m(x) \right\|_1}_{\text{ensemble distillation distance}}  +2 \underbrace{ \underset{(x, y)\sim \mu} {\mathbb{E}}\left(1-  \frac{1}{M}\sum_{m=1}^M h_m(x)_y\right)}_{\text{errors of teacher models}}}
\end{align}
Then, we can relate the errors of local models $h_m$ on $\mu$ to $\mu_m$ with prior arts from domain adaptation~\cite{ben2010theory}.

To simply our notations, we define \textbf{ ``virtual hypothesis'' $h \in \gH: \gX \rightarrow [0, 1]^k$ , whose outputs are the ensemble outputs from all local models}:
$$h(x):=\frac{1}{M}\sum_{m=1}^M h_m (x).$$

\paragraph{Main Analysis}
Let us recall \cref{thm:main_distill}. 
\globalgen*
To prove the generalization bounds of the global model \cref{thm:main_distill}, we use \cref{lm:part_one} as a bridge.

\begin{lemma}
\label{lm:part_one}
Let classes of bounded functions $\gH$ and $\gG$ be given with  $h \in \gH: \gX \rightarrow [0, 1]^k$ and  $g \in \gG : \gX \rightarrow [0, 1]^k$. Suppose $\{x_i\}_{i=1}^{n_{\mathtt{aux}}}$ is sampled from a distribution $\mu_{\mathtt{aux}}$.
For every $h \in \gH$ and every  $g \in\gG $, with probability at least $1-\delta$,
\begin{align*}
\E_{(x,y)\sim \mu }   g(x)_y   & \leq    \E_{(x,y)\sim \mu }   h(x)_y +  \frac{1}{n_{\mathtt{aux}}} \sum_{i=1}^{n_{\mathtt{aux}}}\min \left\{1,\|g(x_i)-h(x_i)\|_1\right\} + 2\mathbb{TV}(\mu,\mu_{\mathtt{aux}}) + 3\sqrt{\frac{\log(2/\delta)}{2{n_{\mathtt{aux}}}}} \\
&  + 2 \sum_{y^{\prime}=1}^k  \left(\operatorname{Rad}_{n_{\mathtt{aux}}}\left(\left\{x \mapsto h(x)_{y^{\prime}}:  h \in \mathcal{H} \right\}\right)+ \operatorname{Rad}_{n_{\mathtt{aux}}}\left(\left\{x \mapsto g(x)_{y^{\prime}}: g \in \mathcal{G}\right\}\right)\right)
\end{align*}
\end{lemma}

\begin{proof}
To start, for any $h \in \gH, g \in \gG$, write 
\begin{align*}
	\E_{x,y}  g(x)_y = \E_{x,y} (g(x)-h(x))_y + \E_{x,y}  h(x)_y   
\end{align*}
For the first term, since $h : \gX \rightarrow [0, 1]^k$ and  $g : \gX \rightarrow [0, 1]^k$, by Holder’s inequality
\begin{align*}
\E_{x,y} (g(x)-h(x))_y &=\int \min \left\{1,(g(x)-h(x))_y\right\} \mathrm{d} \mu(x, y) \\
& \leq \int \min \left\{1,\|g(x)-h(x)\|_1\right\} \mathrm{d} \mu_{\mathcal{X}}(x) 
\end{align*}
Here we need $1$ in  $\min \left\{1,(g(x)-h(x))_y\right\} $ to make the upper bound tighter, since  $(g(x)-h(x))_y \leq 1$ always hold.

Note that for any two measures $\mu$ and $\nu$, and for any continuous function $f(x)$ in $[0,1]$,
\begin{align*}
    \int h(x) (\mathrm d\mu(x) - \mathrm d\nu(x)) &= \int_{x\in A} f(x) (\mathrm d\mu(x) - \mathrm d\nu(x)) + \int_{x\in B} f(x) (\mathrm d\mu(x) - \mathrm d\nu(x))\\
    &\leq |\mu(A) - \nu(A)| + |\mu(B) - \nu(B)|\\
    & \leq 2 \sup_{\text{measurable}\ S}|\mu(S) - \nu(S)| = 2\mathbb {TV}(\mu,\nu),
\end{align*}
where $A = \{x: \mathrm d\mu(x)\geq \mathrm d\nu(x)\}$ and $B = \{x: \mathrm d\mu(x)< \mathrm d\nu(x)\}$.

Once again invoking standard Rademacher complexity arguments~\cref{lm:rademacher}, with probability at least $1-\delta$, every mapping $x \mapsto \min \left\{1,\|g(x)-h(x)\|_1\right\}$  where $h \in \gH $ and  $g \in\gG $ satisfies

\begin{align*}
    \int \min \left\{1,\|g(x)-h(x)\|_1\right\} \mathrm{d} \mu_{\mathcal{X}}(x) & \leq \int \min\{1,\|g(x)-h(x)\|_1\} \mathrm d\mu_{\mathtt{aux}}(x) \\
    & \quad + \int \min\{1,\|g(x)-h(x)\|_1\} (\mathrm d\mu(x)-\mathrm d\mu_{\mathtt{aux}}(x))\\
    &\leq \int \min\{1,\|g(x)-h(x)\|_1\}\mathrm d\mu_{\mathtt{aux}}(x) + 2 \mathbb{TV}(\mu,\mu_{aux})\\
    &\leq \frac{1}{n_{\mathtt{aux}}} \sum_{i=1}^{n_{\mathtt{aux}}}\min \left\{1,\|g(x_i)-h(x_i)\|_1\right\} + 2\mathbb{TV}(\mu,\mu_{\mathtt{aux}}) + 3\sqrt{\frac{\log(2/\delta)}{2{n_{\mathtt{aux}}}}} \\
&  + 2 \operatorname{Rad}_{n_{\mathtt{aux}}}\left(\left\{ x \mapsto \min \left\{1,\|g(x)-h(x)\|_1\right\} : h \in \mathcal{H} , g \in \gG \right\}\right) 
\end{align*}

For the final Rademacher complexity estimate, first note $r \mapsto  \min \{ 1, r \}$ is $1$-Lipschitz and can be peeled off, and we use the definition of the empirical Rademacher complexity~(\cref{lm:rademacher}), thus
\begin{align*}
&\operatorname{Rad}_{n_{\mathtt{aux}}}\left(\left\{ x \mapsto \min \left\{1,\|g(x)-h(x)\|_1\right\} : h \in \mathcal{H} , g \in \gG \right\}\right)\\
&\leq  \operatorname{Rad}_{n_{\mathtt{aux}}}\left(\left\{ x \mapsto  \|g(x)-h(x)\|_1 : h \in \mathcal{H} , g \in \gG \right\}\right)\\
&= \mathbb{E}_\epsilon \sup _{\substack{h \in \mathcal{H} \\
g \in \mathcal{G}}} \frac{1}{n_{\mathtt{aux}}} \sum_{i=1}^{n_{\mathtt{aux}}} \epsilon_i\left\|g\left(x_i\right)-h\left(x_i\right)\right\|_1 \\
&\leq \sum_{y^{\prime}=1}^k \mathbb{E}_\epsilon \sup _{\substack{h \in \mathcal{H} \\
g \in \mathcal{G}}} \frac{1}{n_{\mathtt{aux}}} \sum_{i=1}^{n_{\mathtt{aux}}} \epsilon_i\left|g\left(x_i\right)-h\left(x_i\right)\right|_{y^{\prime}} \\
&=\sum_{y^{\prime}=1}^k  \operatorname{Rad}_{n_{\mathtt{aux}}}\left(\left\{x \mapsto|g(x)-h(x)|_{y^{\prime}}: h \in \mathcal{H}, g \in \mathcal{G}\right\}\right) .
\end{align*}

Since $h$ and $g$ have range $[0,1]^k$,
then $(h-g)_{y^{\prime}}$  has range $[ -1,1]$ for every $y^{\prime}$, 
and since $r \mapsto  |r| $ is $1$-Lipschitz over $[-1, 1]$, combining this with the Lipschitz composition rule for Rademacher complexity and also the fact that a Rademacher random vector $\epsilon \in \{\pm 1\}^n$ is distributionally equivalent to its coordinate-wise negation $-\epsilon$, then, for every $y^{\prime} \in [k]$ 
\begin{align*}
&\operatorname{Rad}_{n_{\mathtt{aux}}}\left(\left\{x \mapsto|g(x)-h(x)|_{y^{\prime}}: h \in \mathcal{H}, g \in \mathcal{G}\right\}\right) \\
&\leq  \operatorname{Rad}_{n_{\mathtt{aux}}}\left(\left\{x \mapsto(g(x)-h(x))_{y^{\prime}}: h \in \mathcal{H}, g \in \mathcal{G}\right\}\right) \\
&=\frac{1}{n_{\mathtt{aux}}} \mathbb{E}_\epsilon \sup _{h \in \mathcal{H}} \sup _{g \in \mathcal{G}} \sum_{i=1}^{n_{\mathtt{aux}}} \epsilon_i\left(h\left(x_i\right)-g\left(x_i\right)\right)_{y^{\prime}} \\
&=\frac{1}{n_{\mathtt{aux}}} \mathbb{E}_\epsilon \sup _{h \in \mathcal{H}} \sum_{i=1}^{n_{\mathtt{aux}}} \epsilon_i h\left(x_i\right)_{y^{\prime}}+\frac{1}{n_{\mathtt{aux}}} \mathbb{E}_\epsilon \sup _{g \in \mathcal{G}} \sum_{i=1}^{n_{\mathtt{aux}}}-\epsilon_i g\left(x_i\right)_{y^{\prime}} \\
&= \operatorname{Rad}_{n_{\mathtt{aux}}}\left(\left\{x \mapsto h(x)_{y^{\prime}}:  h \in \mathcal{H} \right\}\right)+ \operatorname{Rad}_{n_{\mathtt{aux}}}\left(\left\{x \mapsto g(x)_{y^{\prime}}: g \in \mathcal{G}\right\}\right)
\end{align*}
\end{proof}

Inspired by ~\cite{hsu2021generalization}, we introduce \cref{lm:hsuphi} to tackle the error probability 
$\operatorname{Pr}_{(x, y)\sim\mu}\left[\underset{y^{\prime}}{\arg \max } g(x)_{y^{\prime}} \neq y\right]$.

Let us define $\psi (v) = 1- v $.  According to \cref{lm:hsuphi}, we can derive the upper bound for $\operatorname{Pr}_{x, y}\left[\underset{y^{\prime}}{\arg \max } g(x)_{y^{\prime}} \neq y\right] $ in  \cref{thm:main_distill} as below 
\begin{align} \label{eq:psi-g-and-main}
\mathbb{E}_{x, y}  \psi(g(x))_y&=\mathbb{E}_{x, y} \left(1- g(x)_y\right) \nonumber \\
&\geq \frac{1}{2} \mathbb{E}_{x, y}\left[ \mathds{1}\left[\underset{y^{\prime}}{\arg \max }g(x)_{y^{\prime}} \neq y\right]\right]  \nonumber  \\
&= \frac{1}{2}\operatorname{Pr}_{x, y}\left[\underset{y^{\prime}}{\arg \max } g(x)_{y^{\prime}} \neq y\right]     
\end{align}

Then we will study the upper bound for $\mathbb{E}_{x, y}  \psi(g(x))_y$ in \cref{lm:psi-g}.

\begin{lemma}\label{lm:psi-g}
Let classes of bounded functions $\gH$ and $\gG$ be given with  $h \in \gH: \gX \rightarrow [0, 1]^k$ and  $g \in \gG : \gX \rightarrow [0, 1]^k$. Let classes of bounded functions $\gH_m$  be given with  $h_m \in \gH_m: \gX \rightarrow [0, 1]^k$, $\forall m \in [M]$. For every $h_m \in \gH_m , \forall m \in [M] $, and for every  $g \in\gG $, with probability at least $1-\delta$,
\begin{align*}
   \E_{(x,y)\sim \mu }  [1 - g(x)_y)] \leq& \E_{(x,y)\sim \mu }  [1 - h(x)_y] + \Phi_{\mu_{\mathtt{aux}},{n_{\mathtt{aux}}}} (h_1,\ldots, h_M; g )+ 2\mathbb TV( \mu, \mu_{\mathtt{aux}}) +3 \sqrt{\frac{\log(2/\delta)}{2{n_{\mathtt{aux}}}}}  \\
    & + \mathcal{O} \left({k^{3/2}}\left [ \max_{j} \left(\frac{1}{M} \sum_{m=1}^M \operatorname{Rad}_{n_{\mathtt{aux}}}(\mathcal{H}_m|_j)	\right)+ \max_{j} \operatorname{Rad}_{n_{\mathtt{aux}}}(\mathcal{G}|_j) \right]\right) 
\end{align*}
\end{lemma}
\begin{proof}
    We define two function classes 
$$\mathcal{Q}_{\mathcal{H}}:=\left\{(x, y) \mapsto \psi\left(h(x)_y\right):h \in \mathcal{H}\right\} \quad \text { and } \quad \mathcal{Q}_{\mathcal{G}}:=\left\{(x, y) \mapsto \psi\left(g(x)_y\right): g \in \mathcal{G}\right\} ,$$ 
and use the fact that:
$$\frac{1}{n_{\mathtt{aux}}} \sum_{i=1}^{n_{\mathtt{aux}}}\left\|\psi\left(g\left(x_i\right)\right)-\psi\left(h\left(x_i\right)\right)\right\|_1=\frac{1}{n_{\mathtt{aux}}} \sum_{i=1}^{n_{\mathtt{aux}}}\left\|1- g\left(x_i\right)-1+ h\left(x_i\right)\right\|_1=\Phi_{\mu_{\mathtt{aux}},{n_{\mathtt{aux}}}} (h_1,\ldots, h_M; g ).$$

We use $\mathcal{Q}_{\mathcal{H}}$ and $\mathcal{Q}_{\mathcal{G}}$ in  \cref{lm:part_one}, and  use \cref{lm:phi_rad} and \cref{lm:fm_rad} to estimate $\operatorname{Rad}_n\left( \mathcal{Q}_{\mathcal{H}} \right) $  and $\operatorname{Rad}_n\left( \mathcal{Q}_{\mathcal{G}} \right) $, with probability $1-\delta$, yielding

\begin{align*}
   & \E_{(x,y)\sim \mu }  [\psi(g(x))_y)]\\
   & \leq  \E_{(x,y)\sim \mu }  [\psi(h(x))_y)] +  \frac{1}{n_{\mathtt{aux}}} \sum_{i=1}^{n_{\mathtt{aux}}}\min \left\{1,\|\psi(g(x_i))-\psi(h(x_i))\|_1\right\} + 2\mathbb{TV}(\mu_{\mathtt{aux},\mu})  + 3\sqrt{\frac{\log(2/\delta)}{2{n_{\mathtt{aux}}}}} \\
& \quad +  2 \sum_{y^{\prime}=1}^k  \left(\operatorname{Rad}_{n_{\mathtt{aux}}}\left(\left\{x \mapsto \psi(h(x))_{y^{\prime}}:  h \in \mathcal{H} \right\}\right)+ \operatorname{Rad}_{n_{\mathtt{aux}}}\left(\left\{x \mapsto \psi(g(x))_{y^{\prime}}: g \in \mathcal{G}\right\}\right)\right)\\
& \leq  \E_{(x,y)\sim \mu }  [1 - h(x)_y)] + \Phi_{\mu_{\mathtt{aux}}, {n_{\mathtt{aux}}}} (h_1,\ldots, h_M; g ) + 2\mathbb {TV}(\mu_{\mathtt{aux}},\mu)+3 \sqrt{\frac{\log(2/\delta)}{2{n_{\mathtt{aux}}}}} \\
& \quad +    \mathcal{O} \left({k^{3/2}}\left [ \max_{j} \operatorname{Rad}_{n_{\mathtt{aux}}}(\mathcal{H}|_j)+ \max_{j} \operatorname{Rad}_{n_{\mathtt{aux}}}(\mathcal{G}|_j) \right]\right) \tag*{(Due to \cref{eq:def_phi} and \cref{lm:phi_rad})}\\
&= \E_{(x,y)\sim \mu }  [1 - h(x)_y] + \Phi_{\mu_{\mathtt{aux}}, {n_{\mathtt{aux}}}} (h_1,\ldots, h_M; g )+ 2\mathbb {TV}(\mu_{\mathtt{aux}},\mu) + 3 \sqrt{\frac{\log(2/\delta)}{2{n_{\mathtt{aux}}}}} \\
&\quad +    \mathcal{O} \left({k^{3/2}}\left [ \max_{j} \left(\frac{1}{M} \sum_{m=1}^M \operatorname{Rad}_{n_{\mathtt{aux}}}(\mathcal{H}_m|_j)	\right)+ \max_{j} \operatorname{Rad}_{n_{\mathtt{aux}}}(\mathcal{G}|_j) \right]\right) \tag*{(Due to \cref{lm:fm_rad})}\\
\end{align*}
\end{proof}

To show our generalization bounds in \cref{thm:main_distill}, it remains to bound $\E_{(x,y)\sim \mu }  [1 - h(x)_y]$ in \cref{lm:psi-g}. 
\begin{lemma}\label{lm:part_domain_adapt}
Let classes of bounded functions $\gH_m$  be given with  $h_m \in \gH_m: \gX \rightarrow [0, 1]^k$, $\forall m \in [M]$, and $d_m$ be the VC dimension of $\gH_m$. Then with probability at least $ 1- \delta$ over the draw of $\mathcal{D}'=\{(x_i,y_i)\}_{i=1}^n$ from distribution $\mu$, and $\mathcal{D}_m'$ from distribution $\mu_m$ with size $n$ , for every $h_m \in \gH_m , \forall m \in [M] $,

\begin{align*}
\E_{(x,y)\sim \mu }  \left  [ 1 - h(x)_y \right ]  \leq & \frac{1}{M} \sum_{m=1}^M  \left(  \E_{(x,y)\sim \mu_m } \left [1 -  {h_m(x)}_y  \right]  +  \frac{1}{2} {\hat d}_{\gH \Delta \gH} (\mathcal{D}_m', \mathcal{D}')  \right. \\
& \left.  \quad +  \lambda_m +  4\sqrt{\frac{2d_m \log (2n)+ \log(2M/\delta)}{n}} \right) 
\end{align*}
where $\lambda_m = \varepsilon_{\mu_m}(h^*) +  \varepsilon_{\mu}(h^*)  $ and  $h^* := \arg\min_{h \in\gH } \varepsilon_{\mu_m} (h) +  \varepsilon_{\mu} (h)$.
\end{lemma}
\begin{proof}
Since the predictions from different local models ${h_m}$ are independent,  we can expand $h(x)$ as below:
\begin{align*}
	\E_{(x,y)\sim \mu }  \left [ 1 - h(x)_y \right ] &= \E_{(x,y)\sim \mu }  \left [ 1 - \left (\frac{1}{M} \sum_{m=1}^M {h_m(x)}_y \right) \right ] = \frac{1}{M}  \sum_{m=1}^M  \E_{(x,y)\sim \mu } \left [1 -  {h_m(x)}_y  \right] 
\end{align*}

We apply \cref{lm:domainadapt} for the target distribution $\mu$ and each local distribution $\mu_m$. Concretely, with probability $1-\delta/M $, 

\begin{align*}
 &  \E_{(x,y)\sim \mu } \left [1- h_m(x)_y \right] \\
 =&\E_{(x,y)\sim \mu } |h_m(x)_y - h_{\mu}^*(x)_y|\tag{use the fact of labeling function that $ h_{\mu}^*(x)_y =1, (x,y)\sim \mu $ }\\
 =&  \varepsilon_{\mu}(h_m) \tag{use the labeling function as in \cref{def:vareps}}\\
 \leq& \varepsilon_{\mu_m}(h_m) +
 \frac{1}{2} {\hat d}_{\gH \Delta \gH} (\mathcal{D}_m', \mathcal{D}') + 4\sqrt{\frac{2d_m \log (2n)+ \log(2M/\delta)}{n}} + \lambda_m\\
=& \E_{(x,y)\sim \mu_m } |h_m(x)_y - h_{\mu_m}^*(x)_y| +
 \frac{1}{2} {\hat d}_{\gH \Delta \gH} (\mathcal{D}_m', \mathcal{D}') + 4\sqrt{\frac{2d_m \log (2n)+ \log(2M/\delta)}{n}} + \lambda_m\\
 =& \E_{(x,y)\sim \mu_m } \left [1- h_m(x)_y \right]  +
 \frac{1}{2} {\hat d}_{\gH \Delta \gH} (\mathcal{D}_m', \mathcal{D}') + 4\sqrt{\frac{2d_m \log (2n)+ \log(2M/\delta)}{n}} + \lambda_m\tag{use the fact of labeling functions that $h_{\mu_m}^*(x)_y =1 , (x,y)\sim \mu_m $ }\\
\end{align*}
where $\lambda_m = \varepsilon_{\mu_m}(h^*) +  \varepsilon_{\mu}(h^*)  $ and  $h^* := \arg\min_{h \in\gH } \varepsilon_{\mu_m} (h) +  \varepsilon_{\mu} (h)$.

Combing all $m\in [M]$ together, with with probability $1-\delta$,   we have 
\begin{align*}
&\E_{(x,y)\sim \mu }  \left [ 1 - h(x)_y \right ] \\
& =\frac{1}{M}  \sum_{m=1}^M  \E_{(x,y)\sim \mu } \left [1 -  {h_m(x)}_y  \right]  \\
& \leq  \frac{1}{M} \sum_{m=1}^M  \left(  \E_{(x,y)\sim \mu_m } \left [1 -  {h_m(x)}_y  \right]    +  \frac{1}{2} {\hat d}_{\gH \Delta \gH} (\mathcal{D}_m', \mathcal{D}') +  \lambda_m +  4\sqrt{\frac{2d_m \log (2n)+ \log(2M/\delta)}{n}} \right) 
\end{align*}
\end{proof}

\begin{lemma}\label{lm:empirical_risk_eachm_rademacher}
With probability at least $1-\delta$, we have w.r.t the draw of $\sD_{m} \sim \mu_{m}$  with $|\sD_m|=n$ that 
	\begin{equation}
     \E_{(x,y)\sim \mu_m } \left [1 -  {h_m(x)}_y  \right]  \leq \mathtt{ERR}(\sD_m, h_m) + 2 \operatorname{Rad}_{n}(\mathcal{H}_m)  + 3 \sqrt{\frac{\log(2/\delta)}{2n}}
	\end{equation}
\end{lemma}
where $\mathtt{ERR}(\sD_m, h_m) =  \frac{1}{n} \sum_{j=1}^{n} \left[1- h_m(x_{m,j})_{y_{m,j}}\right]$.
\begin{proof}
The proofs directly follow \cref{lm:rademacher} with $b =1, a=0$.
\end{proof}

Given \cref{lm:psi-g}  and \cref{lm:part_domain_adapt} with at least $1-\delta/3$  probability for each event, and \cref{lm:empirical_risk_eachm_rademacher} with at least $1-\delta/3M$ probability for each local model $m \in [M]$, 
we can bound  
$\mathbb{E}_{x, y}  \psi(g(x))_y$ in  \cref{eq:psi-g-and-main}, which proves the main result in \cref{thm:main_distill}. 

\subsection{Proof for Generalization Bounds of Personalized Models in \cref{thm:main_distill_personal}} \label{app:proof-per-model-generalization} 

\paragraph{Overview}
For the generalization bounds of the personalized models, we will upper bound error probabilities of $p_m$ with the expected prediction distances between the global model and personalized model on $\mu$ as well as  errors of the global model on $\mu$.

\paragraph{Main Analysis}
The proofs for \cref{thm:main_distill_personal} are similar to \cref{lm:part_one} and \cref{lm:psi-g}. 
We first introduce \cref{lm:per-general-partone} as below.

\begin{lemma}
\label{lm:per-general-partone}
Let classes of bounded functions $\gP_m$ and $\gG$ be given with  $p_m \in \gP_m: \gX \rightarrow [0, 1]^k$ and  $g \in \gG : \gX \rightarrow [0, 1]^k$. Suppose $\{x_i\}_{i=1}^n$ is sampled from a distribution $\mu$.
For every $p_m \in \gP_m$ and every  $g \in\gG $, with probability at least $1-\delta$,
\begin{align*}
\E_{(x,y)\sim \mu }   p_m(x)_y   & \leq    \E_{(x,y)\sim \mu }   g(x)_y +  \frac{1}{n} \sum_{i=1}^n\min \left\{1,\|p_m(x_i)-g(x_i)\|_1\right\}  + 3\sqrt{\frac{\log(2/\delta)}{2n}} \\
&  + 2 \sum_{y^{\prime}=1}^k  \left(\operatorname{Rad}_n\left(\left\{x \mapsto p_m(x)_{y^{\prime}}:  p_m \in \mathcal{P}_m \right\}\right)+ \operatorname{Rad}_n\left(\left\{x \mapsto g(x)_{y^{\prime}}: g \in \mathcal{G}\right\}\right)\right)
\end{align*}
\end{lemma}

\begin{proof}
To start, for any $p_m \in \gP_m, g \in \gG$, write 
\begin{align*}
	\E_{x,y}  p_m(x)_y = \E_{x,y} (p_m(x)-g(x))_y + \E_{x,y}  g(x)_y   
\end{align*}
For the first term, since $p_m : \gX \rightarrow [0, 1]^k$ and  $g : \gX \rightarrow [0, 1]^k$, by Holder’s inequality
\begin{align*}
\E_{x,y} (p_m(x)-g(x))_y &=\int \min \left\{1,(p_m(x)-g(x))_y\right\} \mathrm{d} \mu(x, y)   \leq \int \min \left\{1,\|p_m(x)-g(x)\|_1\right\} \mathrm{d} \mu_{\mathcal{X}}(x) 
\end{align*}

Once again invoking standard Rademacher complexity arguments~\cref{lm:rademacher}, with probability at least $1-\delta$, every mapping $x \mapsto \min \left\{1,\|p_m(x)-g(x)\|_1\right\}$  where $p_m \in \gP_m $ and  $g \in\gG $ satisfies

\begin{align*}
    & \int \min \left\{1,\|p_m(x_i)-g(x_i)\|_1\right\} \mathrm{d} \mu_{\mathcal{X}}(x) \\
    & \leq \int \min\{1,\|p_m(x_i)-g(x_i)\|_1\} \mathrm d\mu(x) \\
    &\leq \frac{1}{n} \sum_{i=1}^n\min \left\{1,\|p_m(x_i)-g(x_i)\|_1\right\}  + 3\sqrt{\frac{\log(2/\delta)}{2n}} \\
    & + 2 \operatorname{Rad}_n\left(\left\{ x \mapsto \min \left\{1,\|p_m(x_i)-g(x_i)\|_1\right\} : p_m \in \mathcal{P}_m , g \in \gG \right\}\right) 
\end{align*}

For the final Rademacher complexity estimate, we follow the proofs in our previous ~\cref{lm:part_one} and have 
\begin{align*}
&\operatorname{Rad}_n\left(\left\{ x \mapsto \min \left\{1,\|p_m(x)-g(x)\|_1\right\} : p_m \in \mathcal{P}_m , g \in \gG \right\}\right)\\
&\leq \sum_{y^{\prime}=1}^k  \operatorname{Rad}_n\left(\left\{x \mapsto|p_m(x)-g(x)|_{y^{\prime}}: p_m \in \mathcal{P}_m, g \in \mathcal{G}\right\}\right) .
\end{align*}

Also following the proof steps in our \cref{lm:part_one}, we have for every $y^{\prime} \in [k]$ 
\begin{align*}
&\operatorname{Rad}_n\left(\left\{x \mapsto|p_m(x)-g(x)|_{y^{\prime}}: p_m \in \mathcal{P}_m, g \in \mathcal{G}\right\}\right) \\
&\leq \operatorname{Rad}_n\left(\left\{x \mapsto p_m(x)_{y^{\prime}}:  p_m \in \mathcal{P}_m \right\}\right)+ \operatorname{Rad}_n\left(\left\{x \mapsto g(x)_{y^{\prime}}: g \in \mathcal{G}\right\}\right)
\end{align*}
Combining the above results together, we complete the proof.
\end{proof}

Let us recall \cref{thm:main_distill_personal}. 
\pergen*
Then we prove \cref{thm:main_distill_personal} as below:
\begin{proof}[Proof for \cref{thm:main_distill_personal}]

Following the proofs in our previous~\cref{lm:psi-g}, we define  two function classes 
$$\mathcal{Q}_{\mathcal{P}_m}:=\left\{(x, y) \mapsto \psi\left(p_m(x)_y\right): p_m \in \mathcal{P}_m\right\} \quad \text { and } \quad \mathcal{Q}_{\mathcal{G}}:=\left\{(x, y) \mapsto \psi\left(g(x)_y\right): g \in \mathcal{G}\right\} ,$$ 
and use the fact that:
$$\frac{1}{n} \sum_{i=1}^n\left\|\psi\left(p_m\left(x_i\right)\right)-\psi\left(g\left(x_i\right)\right)\right\|_1=\frac{1}{n} \sum_{i=1}^n\left\|1- p_m\left(x_i\right)-1+ g\left(x_i\right)\right\|_1= \frac{1}{n} \sum_{i=1}^n\left\| p_m\left(x_i\right) - g\left(x_i\right)\right\|_1  $$

We use $\mathcal{Q}_{\mathcal{P}_m}$ and $\mathcal{Q}_{\mathcal{G}}$ in  \cref{lm:per-general-partone}, and  use \cref{lm:phi_rad} and \cref{lm:fm_rad} to estimate $\operatorname{Rad}_n\left( \mathcal{Q}_{\mathcal{P}_m} \right) $  and $\operatorname{Rad}_n\left( \mathcal{Q}_{\mathcal{G}} \right) $, with probability $1-\delta$, yielding

\begin{align*}
   &\E_{(x,y)\sim \mu }  [1- p_m(x)_y)] =  \E_{(x,y)\sim \mu }  [\psi(p_m(x))_y)] \\
   & \leq  \E_{(x,y)\sim \mu }  [\psi(g(x))_y)] +  \frac{1}{n} \sum_{i=1}^n\min \left\{1,\|\psi(p_m(x_i))-\psi(g(x_i))\|_1\right\}  + 3\sqrt{\frac{\log(2/\delta)}{2n}} \\
& \quad +  2 \sum_{y^{\prime}=1}^k  \left(\operatorname{Rad}_n\left(\left\{x \mapsto \psi(p_m(x))_{y^{\prime}}:  p_m \in \mathcal{P}_m \right\}\right)+ \operatorname{Rad}_n\left(\left\{x \mapsto \psi(g(x))_{y^{\prime}}: g \in \mathcal{G}\right\}\right)\right)\\
&\leq \E_{(x,y)\sim \mu }  [1 - g(x)_y] +  \frac{1}{n} \sum_{i=1}^n\min \left\{1,\|p_m(x_i)-g(x_i)\|_1\right\}  + 3 \sqrt{\frac{\log(2/\delta)}{2n}} \\
&\quad +    \mathcal{O} \left({k^{3/2}}\left [ \max_{j} \operatorname{Rad}_n(\mathcal{P}_m|_j)  + \max_{j} \operatorname{Rad}_n(\mathcal{G}|_j) \right]\right) \tag*{(Due to \cref{lm:phi_rad})}\\
\end{align*}

Finally, we use \cref{lm:hsuphi} to show that 
\begin{align} 
&\mathbb{E}_{x, y} \left(1- p_m(x)_y\right) \geq \frac{1}{2} \mathbb{E}_{x, y}\left[ \mathds{1}\left[\underset{y^{\prime}}{\arg \max }p_m(x)_{y^{\prime}} \neq y\right]\right]  \nonumber = \frac{1}{2}\operatorname{Pr}_{x, y}\left[\underset{y^{\prime}}{\arg \max } p_m(x)_{y^{\prime}} \neq y\right]     
\end{align}
Combining all results together, with probability at least $1-\delta$, we have, 
\begin{align*} 
 \operatorname{Pr}_{x, y}\left[\underset{y^{\prime}}{\arg \max } p_m(x)_{y^{\prime}} \neq y\right]    &\leq  2    \E_{(x,y)\sim \mu }  [1 - g(x)_y] + 2   \frac{1}{n} \sum_{i=1}^n\min \left\{1,\|p_m(x_i)-g(x_i)\|_1\right\}  + 6 \sqrt{\frac{\log(2/\delta)}{2n}} \\
&\quad +    \mathcal{O} \left({k^{3/2}}\left [ \max_{j} \operatorname{Rad}_n(\mathcal{P}_m|_j)  + \max_{j} \operatorname{Rad}_n(\mathcal{G}|_j) \right]\right)
\end{align*}
This completes the proof.

\end{proof}

\newpage

\section{Convergence Analysis}\label{app:convergence-proof}
In this section, we present the discussions and analysis for our convergence guarantees. The outline of this section is as follows:
\begin{itemize}[noitemsep,leftmargin=*]
    \item \cref{app:convergence-add-results} provides more discussions and additional convergence results. 
    \item \cref{app:proofs-g-model-converge} provides the proofs for the global model convergence guarantee in \cref{thm:convergence1}. 
    \item \cref{app:proofs-per-model-converge} provides the proofs for the personalized model convergence guarantee in \cref{thm:convergence2}. 
\end{itemize}

\subsection{Additional Discussions and Theoretical Results}\label{app:convergence-add-results}

\paragraph{Discussions on distillation gradient}
For simplicity, we denote  $\overline{f(\theta, x)}= \frac{1}{M} \sum_{m=1}^Mf({\theta_m}, x)$.
The closed-form expression of $\nabla_w \mathcal{R}$ can be expressed as:
\begin{align}
& \left\| \nabla_w \mathcal{R}\left(\left\{\theta_1, \ldots, \theta_M\right\}, w\right) \right\| \nonumber \\
& =\left\| \mathbb{E}_{x \sim \mu_{\text {aux }}} 
\sum_{i=1}^k \nabla_w \left[ \sigma(\overline{f(\theta, x)})_i \ln \left( \frac{\sigma(\overline{f(\theta, x)})_i}{\sigma(f(w, x))_i }  \right) \right] \right\|  \tag{KL divergence loss} \nonumber\\
& =\left\|  \mathbb{E}_{x \sim \mu_{\text {aux }}} \sum_{i=1}^k-\frac{\sigma(\overline{f(\theta, x)})_i}{\sigma(f(w, x))_i} \nabla_w \sigma(f(w, x))_i \right\| \nonumber \\
& =\left\|  \mathbb{E}_{x \sim \mu_{\text {aux }}} \sum_{i=1}^k\frac{\sigma(\overline{f(\theta, x)})_i}{\sigma(f(w, x))_i} \nabla_w \sigma(f(w, x))_i \right\| \label{eq:nabla_r_closed-form}
\end{align}
where $k$ is the number of classes and $i$ denotes the $i$-th class. 
Here we note that when the averaged logits from local models are qual to the logits of global model, i.e.,  $\sigma(\overline{f(\theta, x)})_i= \sigma(f(w, x))_i$

\begin{align}
 & \left\| \nabla_w \mathcal{R}\left(\left\{\theta_1, \ldots, \theta_M\right\}, w\right) \right\|  =\left\|  \mathbb{E}_{x \sim \mu_{\text {aux }}} \sum_{i=1}^k \nabla_w \sigma(f(w, x))_i \right\| = 0
\end{align}
because $\sum_{i=1}^k \sigma(f(w, x))_i =1$ (which leads to $ \nabla_w \sum_{i=1}^k \sigma(f(w, x))_i =0$ ) .
Therefore, the norm of distillation gradient can be small when  the averaged logits from local models are close to the logits of global model.

\paragraph{Discussions for Assumptions.}
\cref{asp:smooth} on Lipschitz smooth and  \cref{asp:stograd} on the bounded variance for gradients due to stochastic sampling noise are standard for smooth and non-convex optimization.
\cref{asp:boundeddiv} quantifies the diversity of FL clients' data distribution, which is  widely used in FL optimization~\cite{karimireddy2020scaffold,fallah2020personalized,reddi2021adaptive,Li2020On,ozkara2021quped}.
We follow \cite{reddi2021adaptive,ozkara2021quped,fallah2020personalized} to assume bounded gradient for non-convex FL optimization in~\cref{asp:boundedgrad}.

\paragraph{Convergence of \name personalized models.}
 \begin{restatable}[Convergence of \name personalized model]{theorem}{perconv}
\label{thm:convergence2}
\textcolor{black}{ 
When $\eta_p = \frac{1}{(L+\lambda)\sqrt T}$, $ \eta_l = \frac{1}{EL\sqrt{T}} $, $\eta_g = \frac{1}{L_R RT}$, then the algorithm satisfies:
    \begin{equation*}
        \frac{1}{TS} \sum_{t=0}^{T-1} \sum_{s=0}^{S-1} \E \|\nabla_v P_m(v_m^{t,s}, w^t) \|^2 \leq \gO\Big(\frac{(L+\lambda)\Delta_{P_m}+\phi_2}{\sqrt T S}   + \frac{G_P(L+\lambda)(L\Delta_\mathcal L + \psi_1)^{1/2}}{T^{1/4}L\sqrt {E}S} +   \frac{G_P(L+\lambda)\sqrt{\psi_2}}{T^{3/4} L_R ES} + \frac{G_P(L+\lambda)\sigma}{LES}\Big)
    \end{equation*} 
    where $\Delta_{P_m} =P_m(v_m^0, w_m^0) - P_m(v_m^t, w^t)$, $\phi_1 = 64(3\bar \gamma  + \frac{2\sigma^2}{E})$, $\phi_2 = S\sigma^2 + \frac{\sqrt{\phi_1}G_P(L+\lambda)}{LE} +  \sqrt{\psi_2} G_P (L+\lambda)   + \frac{G_P\bar \gamma (L+\lambda)}{L\sqrt E}$. $\psi_1,\psi_2$ are defined the same as in \cref{thm:convergence1}.}
\end{restatable}

\begin{remark}
\looseness=-1
\chulin{
(1) \textbf{Local steps}: a larger local step $S$ can reduce number of rounds $T$ for convergence. 
(2) \textbf{Connection to global model}: The terms associated with $\bar\gamma, \psi_1, \psi_2$ are related to the convergence rate of the global model, which is indicated in \cref{thm:convergence1}. For example, a large $E$ can also reduce the number of communication rounds $T$ for the personalized model to convergence. 
We obtain a convergence rate of $\gO(1/{T}^{1/4})$ for personalized models. 
It is worth noting that previous studies have shown that in strongly convex settings, personalized models converge at the same rate as the global model~\cite{li2021ditto}. However, in  strongly convex settings, the minimizers are ensured to be unique, which can simplify the establishment of connections between global and personalized models by considering their distances to the corresponding minimizers. Here, we present the results in the more general non-convex setting and additionally analyze the effect of the global model's ensemble distillation on personalized models.
}
\end{remark}

\subsection{Proofs for the Global Model Convergence Guarantee in \cref{thm:convergence1}} \label{app:proofs-g-model-converge}

 \paragraph{Additional notations}
Recall the parameter-averaged model is $ \bar {\theta}^{t+1}  = \frac{1}{M}\sum_{m=1}^M {\theta}_m^{t+1} $, which is used to initialize the server global model at round $t$ before the KD training. Let 
\begin{align}
\bar \eta_g = \eta_g R , \quad \bar\eta_l = \eta_l E
\end{align}
Based on the update rules, we define  $g^t$ and $g_m^t$ as below, which capture the update of global model during server training, and the update of local model during client training, respectively.
\begin{align} \label{eq:update_rule_simple}
 w^{t+1} =  \bar{\theta}^{t+1} -  \bar \eta_g   g^t  , \quad \theta_m^{t+1} =  w^{t} - \bar\eta_l  g_m^t 
\end{align}
That is: 
\begin{align} \label{eq:def_gt}
	&g^t :=  - \frac{1}{\eta_g R }( w^{t+1}- \bar{\theta}^{t+1})=  \frac{1}{R}\sum_{r=0}^{R-1} \widetilde{\nabla}_w \gR(\{\theta^{t+1}_{m}\},  w^{t,r}),   \nonumber \\
	&g_m^t  := - \frac{1}{\eta_lE}( \theta_m^{t+1}- w^{t})=  \frac{1}{E}\sum_{e=0}^{E-1} \widetilde{\nabla} \gL_m( {\theta}_m^{t,e}) 
\end{align}
According to server update rule $w^{t+1} -w^{t} =\bar {\theta}^{t+1} - w^{t}  - \bar \eta_g g^t $.
Note that $\bar {\theta}^{t+1} - w^{t} =  \frac{1}{M}\sum_{m=1}^M  {\theta}_m^{t+1} - w^{t} = - 
\frac{1}{M}\sum_{m=1}^M   \bar \eta_l g_m^t$ based on \cref{eq:def_gt}.
Then we define, 
\begin{equation}   \label{eq:def_delta_w}
\delta_w^t :=  \frac{1}{M}\sum_{m=1}^M g_m^t   +    \frac{\bar \eta_g}{\bar \eta_l}  g^t, \text{\quad which indicates $w^{t+1} -w^{t}= - \bar \eta_l \delta_w^t $} 
\end{equation}
According to client update rule $\theta_m^{t+1} -\theta_m^{t} = (w^{t} -  \bar \eta_l g_m^t)- (w^{t-1} -  \bar \eta_l g_m^{t-1})$. Note that  $w^{t}-w^{t-1}=   -  \bar \eta_l \frac{1}{M}\sum_{m=1}^M    g_m^t -  \bar \eta_g g^t  $  based on \cref{eq:def_gt}. Then we define, 
\begin{equation}  \label{eq:def_delta_theta_m}
\delta_{\theta_m}^t :=  \frac{\bar \eta_g}{\bar \eta_l}   g^{t-1}  +  \frac{1}{M}\sum_{i=1}^M   g_i^{t-1} - g_m^{t-1} + g_m^{t},  \text{\quad which indicates $ \theta_m^{t+1} -\theta_m^{t}  =    - \bar \eta_l \delta_{\theta_m}^t $}
 \end{equation}  
In our analysis, we define one virtual sequence $\bar w^{t,e}$, motivated by \cite{Li2020On},
\begin{equation}
   \bar w^{t,e}=\frac{1}{M}\sum_{m=1}^M{\theta_m^{t,e}} 
\end{equation}
In particular, $\bar{w}^{t+1,0}=w^{t}$ and  $\bar{w}^{t+1,E-1}=\bar\theta^{t+1}$.

 \paragraph{Proof Outline}
Recall the generic FL objective, which is to minimize the average loss measured on all clients' data:  
\begin{align} 
 \add{ \gL(w) := \frac{1}{M}  \sum_{m=1}^M \gL_m\left(w\right)}
\end{align}
The goal is to bound the gradients of  the global model w.r.t the $\gL(w)$, which is used to show that the trained models can converge to the stationary points: 
\begin{align}
     \sum_{t=0}^{T-1}\sum_{e=0}^{E-1} \frac{1}{ET} \mathbb E \|\nabla \gL(\bar w^{t,e})\|^2
\end{align}

\paragraph{Challenges}
The challenges of convergence analysis include:
(1) Bi-level optimization between server distillation for $w^{t}$ and client training for $\{\theta_m^{t}\}$, which incorporates two objectives (i.e., minimizing distillation loss and local loss respectively), as shown in \cref{eq:def_gt}. 
(2) Mutual initializations. At each round, the global model is initialized by averaged local models before distillation, and local models are initialized by the global model before local training. Such mutual initializations intervene in the model updating trajectories of $w^{t}$ and $\{\theta_m^{t}\}$  w.r.t their training objective. In particular,
the server optimization of $w$ will be influenced by the drift of client optimization of $\theta_m$, as shown in \cref{eq:def_delta_w} (i.e., additional deviation with the term $ \frac{1}{M}\sum_{m=1}^M g_m^t  $).
Moreover, client optimization is also influenced by the drift of server optimization, as shown in  \cref{eq:def_delta_theta_m} (i.e., additional deviation with the terms $ \frac{\bar \eta_g}{\bar \eta_l}   g^{t-1}  +  \frac{1}{M}\sum_{i=1}^M   g_i^{t-1} - g_m^{t-1} $).

To overcome the aforementioned challenges, we regard $\{\theta_m^{t}\}$ as the intermediate models to update $w^{t+1}$, and quantify the effects of local client updates and server distillation updates on reducing $\gL(w^{t+1})$.

\paragraph{Supporting lemmas}
Before we start, we introduce a useful existing result by Jensen's inequality in \cref{prop:jensen-ineq}:
\begin{lemma}[Jensen's inequality]\label{prop:jensen-ineq}
    For any vector $x_i \in \mathbb{R}^d, i=1, \ldots, M$, by Jensen's inequality, we have
\begin{align}
\left\|\sum_{i=1}^M x_i\right\|^2 \leq M \sum_{i=1}^M\left\|x_i\right\|^2
\end{align}
\end{lemma}

We also introduce the following supporting lemmas:
\begin{lemma}[Bounded local client drift error] \label{lm:bound_local_client_drift}
If $\bar\eta_l \leq \frac{1}{2L}\Leftrightarrow  \eta_l \leq  \frac{1}{2LE}   $, we have
\begin{align}
& \E \left[ \left\|\widetilde{\nabla} \gL_m( {\theta}_m^{t,e}) -  \nabla \gL_m (w^t)   \right\|^2\right] \leq 2  \sigma^2 +  16 L^2  \bar\eta_l^2 \left( 3 \E \left[ \left\|   \nabla \gL_m(w^t) \right\|^2\right]    + \frac{2 \sigma^2}{E} \right).
\end{align}
Moreover, the averaged drift error over $E$ local steps and $M$ clients is:
\begin{align}
&\frac{1}{ME}\sum_{m,e}^{M, E}   \E \left[ \left\|  \widetilde{\nabla} \gL_m( {\theta}_m^{t,e}) -  \nabla \gL_m (w^t)   \right\|^2\right] \leq 2  \sigma^2 +  16 L^2  \bar\eta_l^2 \left( 3 \bar\gamma +  3 \E \left[ \left\|   \nabla \gL(w^t) \right\|^2\right]    + \frac{2 \sigma^2}{E} \right).
\end{align}	
\end{lemma}
\begin{proof}
\begin{align*}
& \E \left[ \left\|  \widetilde{\nabla} \gL_m( {\theta}_m^{t,e}) -  \nabla \gL_m (w^t)   \right\|^2\right]\\
&=\E \left[ \left\|  \widetilde{\nabla} \gL_m( {\theta}_m^{t,e}) -  \nabla \gL_m( {\theta}_m^{t,e}) +   \nabla \gL_m( {\theta}_m^{t,e})  -  \nabla \gL_m (w^t)   \right\|^2\right]\\
& \leq 2 \E \left[ \left\|  \widetilde{\nabla} \gL_m( {\theta}_m^{t,e}) -  \nabla \gL_m( {\theta}_m^{t,e})  \right\|^2\right] +  2 \E \left[ \left\|  \nabla \gL_m( {\theta}_m^{t,e})  -  \nabla \gL_m (w^t)   \right\|^2\right]\\
& \leq 2  \sigma^2 +  2 \E \left[ \left\|  \nabla \gL_m( {\theta}_m^{t,e})  -  \nabla \gL_m (w^t)   \right\|^2\right] \tag{\cref{asp:stograd}}\\
& \leq 2  \sigma^2 +  2 L^2 \E \left[ \left\|  {\theta}_m^{t,e} - w^t   \right\|^2\right] \tag{\cref{asp:smooth}}\\
\end{align*}	
If $\bar\eta_l \leq \frac{1}{2L}\Leftrightarrow  \eta_l \leq  \frac{1}{2LE}   $, we have
\begin{align*}
&\E \left[ \left\|  {\theta}_m^{t,e} - w^t   \right\|^2\right] = \E \left[ \left\|  {\theta}_m^{t,e-1} - w^t   - \eta_l \widetilde{\nabla} \gL_m( {\theta}_m^{t,e-1})  \right\|^2\right]  \\
&=\E \left[ \left\|  {\theta}_m^{t,e-1} - w^t -   \eta_l \nabla \gL_m(w^t)  + \eta_l \nabla \gL_m(w^t)     - \eta_l \widetilde{\nabla} \gL_m( {\theta}_m^{t,e-1})  \right\|^2\right]  \\
&\leq 2 \E \left[ \left\|  {\theta}_m^{t,e-1} - w^t -   \eta_l \nabla \gL_m(w^t) \right\|^2\right]  +  2  \eta_l^2  \E \left[ \left\| \nabla \gL_m(w^t)     -  \widetilde{\nabla} \gL_m( {\theta}_m^{t,e-1}) \right\|^2\right]  \\
&\leq 2 \E \left[ \left\|  {\theta}_m^{t,e-1} - w^t  \right\|^2\right]   + 2 \E \left[ \left\|    \eta_l \nabla \gL_m(w^t) \right\|^2\right]   - 2 \E \left[ \left\langle   {\theta}_m^{t,e-1} - w^t  ,  \eta_l \nabla \gL_m(w^t) \right \rangle  \right]    \\
& \quad + 2  \eta_l^2  \E \left[ \left\| \nabla \gL_m(w^t) - \nabla \gL_m({\theta}_m^{t,e-1}) +  \nabla \gL_m({\theta}_m^{t,e-1})     -  \widetilde{\nabla} \gL_m( {\theta}_m^{t,e-1}) \right\|^2\right]  \\
&\leq 2 \left(1+\frac{1}{2E}\right) \E \left[ \left\|  {\theta}_m^{t,e-1} - w^t  \right\|^2\right]   + 2 \eta_l^2 \left(1+{2E}\right) \E \left[ \left\|   \nabla \gL_m(w^t) \right\|^2\right]   \\ 
& \quad + 4  \eta_l^2 L^2  \E \left[  \left\| {\theta}_m^{t,e-1}- w^t \right\|^2\right]  + 4 \eta_l^2 \sigma^2  \\
&= 2 \left(1+\frac{1}{2E} +  2  \eta_l^2 L^2 \right) \E \left[ \left\|  {\theta}_m^{t,e-1} - w^t  \right\|^2\right]   + 2 \eta_l^2 \left(1+{2E}\right) \E \left[ \left\|   \nabla \gL_m(w^t) \right\|^2\right]    + 4 \eta_l^2 \sigma^2  \\
& \overset{(a)}{\leq}  2 \left(1+\frac{1}{E}\right) \E \left[ \left\|  {\theta}_m^{t,e-1} - w^t  \right\|^2\right]   +  \frac{6 \bar\eta_l^2}{E} \E \left[ \left\|   \nabla \gL_m(w^t) \right\|^2\right]    +  \frac{4\bar\eta_l^2\sigma^2}{E^2}  \\
& \leq  2 \sum_{e=0}^{E-1} \left(1+\frac{1}{E}\right)^e   \left( \frac{6 \bar\eta_l^2}{E} \E \left[ \left\|   \nabla \gL_m(w^t) \right\|^2\right]    + \frac{4\bar\eta_l^2\sigma^2}{E^2} \right) \\
&\overset{(b)}{\leq} 8 \bar\eta_l^2 \left( 3 \E \left[ \left\|   \nabla \gL_m(w^t) \right\|^2\right]    + \frac{2 \sigma^2}{E} \right)
\end{align*}	
Here (a) is because:  $\eta_l= \frac{\bar\eta_l}{E}$ and when \add{$\bar\eta_l^2 \leq \frac{1}{4L^2}$}, we have  $  2  \eta_l^2 L^2 =\frac{2 \bar \eta_l^2 L^2}{E^2} \leq \frac{1}{2E^2} \leq \frac{1}{2E}$ for all $E\geq 1$. 
Moreover, 
$2 \eta_l^2 \left(1+{2E}\right) = 2 \left(1+{2E}\right)  \frac{\bar\eta_l^2}{E^2} \leq   \frac{6 \bar\eta_l^2}{E}  $ because $\frac{1+2E}{E} \leq 3$ for $E\geq 1$.

(b) is because:
$\sum_{e=0}^{E-1} (1+ 1/E)^e = \frac{(1+ 1/E)^E-1}{1/E} \leq \frac{\operatorname{e}-1}{1/E} \leq 2E$
by using the fact 
$\sum_{i=0}^{n-1} x^i=\frac{x^n-1}{x-1} \text { and }\left(1+\frac{x}{n}\right)^n \leq e^x \text { for any } x \in \mathbb{R}, n \in \mathbb{N}$. 

Combining the above results together,  we have
\begin{align*}
& \E \left[ \left\| \widetilde{\nabla} \gL_m( {\theta}_m^{t,e}) -  \nabla \gL_m (w^t)   \right\|^2\right] \leq 2  \sigma^2 +  16 L^2  \bar\eta_l^2 \left( 3 \E \left[ \left\|   \nabla \gL_m(w^t) \right\|^2\right]    + \frac{2 \sigma^2}{E} \right)
\end{align*}

By the expectation $E[\|X\|^2]= E[\|X- E[X]\|^2] + E[\|X\|]^2$, we have the averaged error over $M$ clients: 
\begin{align*}
 \frac{1}{M}\sum_{m=1}^M \E \left[ \left\|   \nabla \gL_m(w^t) \right\|^2\right] 
&= \frac{1}{M}\sum_{m=1}^M \E \left[ \left\|   \nabla \gL_m(w^t)  -  \nabla \gL(w^t)    \right\|^2 \right] + \E \left[ \left\|   \nabla \gL(w^t) \right\|^2\right] \\
&\leq \bar\gamma +  \E \left[ \left\|   \nabla \gL(w^t) \right\|^2\right]  \tag{\cref{asp:boundeddiv}} 
\end{align*}

Moreover, the averaged error over $M$ clients and $E$ local steps is:

\begin{align*}
 \frac{1}{ME}\sum_{m, e}^{M,E} \E \left[ \left\|   \nabla \gL_m(w^t) \right\|^2\right] 
&\leq  2  \sigma^2 +  16 L^2  \bar\eta_l^2 \left( 3 \bar\gamma +  3\E \left[ \left\|   \nabla \gL(w^t) \right\|^2\right]       + \frac{2 \sigma^2}{E} \right)
\end{align*}
Thus, proved.
\end{proof}

\begin{lemma}[Bounded distillation drift error] \label{lm:bound_distill_drift}
If $\bar\eta_g \leq \frac{1}{2L_R} \Leftrightarrow \eta_g \leq \frac{1}{2 R L_R} $, we have 
\begin{align}
&\E \left[ \left\|  \widetilde{\nabla}_w \gR(\{\theta^{t+1}_{m}\},  w^{t,r}) -  \nabla_w \gR(\{\theta^{t+1}_{m}\},  w^{t})     \right\|^2\right]\\
&\leq 2  \sigma_R^2 +   {16 L_R^2  \bar\eta_g^2} \left( 3 \E \left[ \left\|  \nabla_w \gR(\{\theta^{t+1}_{m}\},  w^{t})\right\|^2\right]    + \frac{2 \sigma_R^2}{R} \right).
\end{align}	
\end{lemma}

\begin{proof}
[Proof of \cref{lm:bound_distill_drift}]	
\begin{align*}
&\E \left[ \left\|   \widetilde{\nabla}_w \gR(\{\theta^{t+1}_{m}\},  w^{t,r}) -  \nabla_w \gR(\{\theta^{t+1}_{m}\},  w^{t})     \right\|^2\right] \\
&= \E \left[ \left\|   \widetilde{\nabla}_w \gR(\{\theta^{t+1}_{m}\},  w^{t,r})  -  \nabla_w \gR(\{\theta^{t+1}_{m}\},  w^{t,r})    +   \nabla_w \gR(\{\theta^{t+1}_{m}\},  w^{t,r})     -  \nabla_w \gR(\{\theta^{t+1}_{m}\},  w^{t})     \right\|^2\right] \\
&\leq 2 \E \left[ \left\|   \widetilde{\nabla}_w \gR(\{\theta^{t+1}_{m}\},  w^{t,r})  -  \nabla_w \gR(\{\theta^{t+1}_{m}\},  w^{t,r})   \right\|^2\right]   +  2 \E \left[ \left\| \nabla_w \gR(\{\theta^{t+1}_{m}\},  w^{t,r})     -  \nabla_w \gR(\{\theta^{t+1}_{m}\},  w^{t})     \right\|^2\right] \\
& \leq 2  \sigma_R^2 +  2 L_R^2 \E \left[\left\| w^{t,r}-  w^{t}   \right\|^2\right] \tag{\cref{asp:stograd}, \cref{asp:smooth}}\\
\end{align*}
If $\bar\eta_g \leq \frac{1}{2L_R} \Leftrightarrow \eta_g \leq \frac{1}{2 R L_R} $, we have 
\begin{align*}
& \E \left[\left\| w^{t,r}-  w^{t}   \right\|^2\right]  = \E \left[\left\| w^{t,r-1}  -  w^{t}  - \eta_g \widetilde{\nabla}_w \gR(\{\theta^{t+1}_{m}\},  w^{t,r-1})   \right\|^2\right]  \\
&= \E \left[\left\| w^{t,r-1}  -  w^{t}  - \eta_g \nabla_w \gR(\{\theta^{t+1}_{m}\},  w^{t})+ - \eta_g \nabla_w \gR(\{\theta^{t+1}_{m}\},  w^{t})-  \eta_g \widetilde{\nabla}_w \gR(\{\theta^{t+1}_{m}\},  w^{t,r-1} )   \right\|^2\right]  \\
&\leq 2 \E \left[\left\| w^{t,r-1}  -  w^{t}  - \eta_g \nabla_w \gR(\{\theta^{t+1}_{m}\},  w^{t})  \right\|^2\right]  + 2 \E \left[\left\| \eta_g \nabla_w \gR(\{\theta^{t+1}_{m}\},  w^{t})-  \eta_g \widetilde{\nabla}_w \gR(\{\theta^{t+1}_{m}\},  w^{t,r-1})   \right\|^2\right]  \\
&\leq 2 \left(1 + \frac{1}{2R}\right) \E \left[\left\| w^{t,r-1}  -  w^{t}  \right\|^2\right]  +   2 \eta_g^2 \left(1 + {2R}\right)  \E \left[\left\| \nabla_w \gR(\{\theta^{t+1}_{m}\},  w^{t})  \right\|^2\right]  \\
& \quad + 4 \eta_g^2 L_R^2 \E \left[\left\|     w^{t,r-1} - w^{t} \right\|^2\right]  + 4\eta_g^2\sigma_R^2  \\
&= 2 \left(1+\frac{1}{2R} +  2  \eta_g^2 L_R^2 \right) \E \left[ \left\|    w^{t,r-1} - w^{t}   \right\|^2\right]   + 2 \eta_g^2 \left(1+{2R}\right) \E \left[ \left\|  \nabla_w \gR(\{\theta^{t+1}_{m}\},  w^{t}) \right\|^2\right]    + 4 \eta_g^2 \sigma_R^2  \\
& \overset{(a)}{\leq} 2 \left(1+\frac{1}{R}\right) \E \left[ \left\|    w^{t,r-1} - w^{t}  \right\|^2\right]   +  \frac{6 \bar\eta_g^2}{ R} \E \left[ \left\|  \nabla_w \gR(\{\theta^{t+1}_{m}\},  w^{t}) \right\|^2\right]    +  \frac{4\bar\eta_g^2\sigma_R^2}{ R^2}  \\
&\leq 2 \sum_{r=0}^{R-1} \left(1+\frac{1}{R}\right)^r   \left( \frac{6 \bar\eta_g^2}{ R} \E \left[ \left\|   \nabla_w \gR(\{\theta^{t+1}_{m}\},  w^{t}) \right\|^2\right]    + \frac{4\bar\eta_g^2\sigma_R^2}{ R^2} \right) \\
& \overset{(b)}{\leq} {8 \bar\eta_g^2} \left( 3 \E \left[ \left\|  \nabla_w \gR(\{\theta^{t+1}_{m}\},  w^{t})\right\|^2\right]    + \frac{2 \sigma_R^2}{R} \right)
\end{align*}
Here (a) is because: $\eta_g= \frac{\bar\eta_g}{R}$ and when {$\bar\eta_g^2 \leq \frac{1}{4L_R^2}$}, we have  $  2  \eta_g^2 L_R^2 =\frac{2 \bar \eta_g^2 L_R^2}{R^2} \leq \frac{1}{2R^2} \leq \frac{1}{2R}$ for all $R\geq 1$. 
Moreover, 
$2 \eta_g^2 \left(1+{2R}\right) = 2 \left(1+{2R}\right)  \frac{\bar\eta_g^2}{R^2 } \leq   \frac{6 \bar\eta_g^2}{ R}  $ because $\frac{1+2R}{R} \leq 3$ for $R\geq 1$.

(b) is because:
$\sum_{e=0}^{R-1} (1+ 1/R)^e = \frac{(1+ 1/R)^R-1}{1/R} \leq \frac{\operatorname{e}-1}{1/R} \leq 2R$
by using the fact 
$\sum_{i=0}^{n-1} x^i=\frac{x^n-1}{x-1} \text { and }\left(1+\frac{x}{n}\right)^n \leq e^x \text { for any } x \in \mathbb{R}, n \in \mathbb{N}$. 

Combining the above results, we have 
\begin{align*}
&\E \left[ \left\|   \widetilde{\nabla}_w \gR(\{\theta^{t+1}_{m}\},  w^{t,r}) -  \nabla_w \gR(\{\theta^{t+1}_{m}\},  w^{t})     \right\|^2\right] \leq 2  \sigma_R^2 +   {16 L_R^2  \bar\eta_g^2} \left( 3 \E \left[ \left\|  \nabla_w \gR(\{\theta^{t+1}_{m}\},  w^{t}) \right\|^2\right]    + \frac{2 \sigma_R^2}{R} \right) 
\end{align*}
\end{proof}

Recall \cref{eq:update_rule_simple}, we have
\begin{equation}
    \bar \theta^{t+1} = \frac{1}{M}\sum_{m=1}^M \theta_m^{t+1} = \frac{1}{M} (\sum_{m=1}^M w^t - \bar\eta_l g_m^t).
\end{equation}
and we define
\begin{equation}
    \bar w^{t,e} = \frac{1}{M} \sum_{m=1}^M \theta_m^{t,e}.
\end{equation}

\begin{lemma}
\label{lemma:distill_step}
\begin{equation}
    \mathbb E[\mathcal L(w^{t+1}) - L(\bar \theta^{t+1})] \leq \frac{\bar \eta_g}{2} (G^2+ \psi_2) + \frac{\bar \eta_g^2 L}{2}\psi_2,
\end{equation}
where $\psi_2 = 4 \sigma_R^2 + 32 L_R^2 \bar \eta_g^2 (3 G_R^2 + \frac{2\sigma_R^2}{R}) + 2G_R^2$.
\end{lemma}
\begin{proof}
    \begin{align}
        \mathbb E[\mathcal L(w^{t+1}) - L(\bar \theta^{t+1})] & \leq \mathbb E[\langle \nabla \mathcal L(\bar\theta^{t+1}), -\bar\eta_g g^t\rangle] + \frac{\bar\eta_g^2 L}{2}\mathbb E\|g^t\|^2\\
        &\leq \frac{\bar \eta_g}{2} \mathbb E\|\nabla \mathcal L(\bar\theta^{t+1})\|^2 + \frac{\bar\eta_g}{2} \mathbb E\|g^t\|^2 + \frac{\bar\eta_g^2 L}{2}\mathbb E\|g^t\|^2\\
        &= \frac{\bar \eta_g}{2} \mathbb E\|\frac{1}{M} \sum_{m=1}^M \nabla \mathcal L_m(\bar\theta^{t+1})\|^2 + (\frac{\bar\eta_g}{2} + \frac{\bar\eta_g^2 L}{2} ) \mathbb E\|g^t\|^2 \tag{Based on $\mathcal L = (\theta) \frac{1}{M} \sum_{m=1}^M \nabla \mathcal L_m(\theta)$}.\\
         &\leq \frac{\bar \eta_g}{2} \frac{1}{M} \sum_{m=1}^M  \mathbb E\|\nabla \mathcal L_m(\bar\theta^{t+1})\|^2 + (\frac{\bar\eta_g}{2} + \frac{\bar\eta_g^2 L}{2} ) \mathbb E\|g^t\|^2 \tag{\cref{prop:jensen-ineq}}\\
          &\leq \frac{\bar \eta_g}{2}  G^2  + (\frac{\bar\eta_g}{2} + \frac{\bar\eta_g^2 L}{2} ) \mathbb E\|g^t\|^2 \tag{\cref{asp:boundedgrad}}.
    \end{align}

    Note that
    \begin{align}
        \mathbb E\|g^t\|^2 &= \mathbb E\|\frac{1}{R} \sum_{r=0}^{R-1} \tilde \nabla_W R(\{\theta_m^t\}, w^{t,r})\|^2\\
        &= \mathbb E\|\frac{1}{R} \sum_{r=0}^{R-1} \left( \widetilde{\nabla}_w \gR(\{\theta^{t+1}_{m}\},  w^{t,r}) -  \nabla_w \gR(\{\theta^{t+1}_{m}\},  w^{t})    \right)  +  \nabla_w \gR(\{\theta^{t+1}_{m}\},  w^{t})\|^2\\
        &\leq 2 \mathbb E\|\frac{1}{R} \sum_{r=0}^{R-1} \left( \widetilde{\nabla}_w \gR(\{\theta^{t+1}_{m}\},  w^{t,r}) -  \nabla_w \gR(\{\theta^{t+1}_{m}\},  w^{t})    \right) \|^2 + 2\mathbb E\| \nabla_w \gR(\{\theta^{t+1}_{m}\},  w^{t})\|^2 \tag{from \cref{lm:bound_distill_drift}}\\
          &\leq 2 \frac{1}{R} \sum_{r=0}^{R-1} \mathbb E\| \left( \widetilde{\nabla}_w \gR(\{\theta^{t+1}_{m}\},  w^{t,r}) -  \nabla_w \gR(\{\theta^{t+1}_{m}\},  w^{t})    \right) \|^2 + 2\mathbb E\| \nabla_w \gR(\{\theta^{t+1}_{m}\},  w^{t})\|^2 \tag{\cref{prop:jensen-ineq}}\\
        &\leq 4 \sigma_R^2 + 32 L_R^2 \bar \eta_g^2 (3 G_R^2 + \frac{2\sigma_R^2}{R}) + 2G_R^2 = \psi_2. \label{eq:define_psi_2}
    \end{align}
    Therefore, 
    \begin{equation}
         \mathbb E[\gL(w^{t+1}) -\gL(\bar \theta^{t+1})] \leq \frac{\bar\eta_g}{2}G^2 + (\frac{\bar\eta_g}{2} + \frac{\bar\eta_g^2L}{2})\psi_2.
    \end{equation}

\end{proof}

\begin{lemma}[From \cite{li2019communication}] 
\label{lemma:local_step}
    \begin{equation}
  \frac{1}{E}    \mathbb E   [\gL (\bar \theta^{t+1}) - \gL(w^t)] \leq \frac{1}{E} \sum_{e=0}^{E-1} -\frac{\eta_l}{2} \|\nabla \gL(\bar w^{t,e})\|^2 + \frac{\eta_l^2  L \sigma^2}{2M} + 8 \eta_l^3 E^2 L^2 \bar \gamma^2.
    \end{equation}
\end{lemma}
\begin{proof}
    We leverage the results from Equation (33) of \cite{li2019communication} with $A_T = 0$ and $B_T = 1$\footnote{$A_T$ and $B_T$ are defined in \cite{li2019communication}.}, which are implied by Theorem 2 in \cite{li2019communication}. 
\end{proof}

\paragraph{Completing the proof of \cref{thm:convergence1}}

Recall our main theorem
\globalconv*

\begin{proof}
Combining \cref{lemma:distill_step} and \cref{lemma:local_step},
    \begin{align}
        \mathbb E[\gL (w^{t+1}) - \gL (w^t)] &= \mathbb E[\gL (w^{t+1}) - \gL(\bar\theta^t) + \gL (\bar \theta^t) - \gL(w^t)]\\
        &\leq \frac{\bar \eta_g}{2}(G^2+\psi_2) + \frac{\bar\eta_g^2 L}{2}\psi_2 +\sum_{e=0}^{E-1} - \frac{\eta_l}{2} \|\nabla \mathcal L(\bar w^{t,e})\|^2 + \frac{E\eta_l^2 L \sigma^2}{2M} + 8\eta_l^3 E^3 L^2 \bar \gamma^2.
    \end{align}

    Rearrage the inequality and take $\frac{1}{T}\sum_{t=0}^{T-1}$ on both side. We get
    \begin{equation}
        \frac{1}{ET} \sum_{t=0}^{T-1}\sum_{e=0}^{E-1} \|\nabla \gL(\bar w^{t,e})\|^2 \leq \frac{2}{\eta_l ET} \big(\gL(w^0) - \gL(w^{T-1})\big) +  \frac{\eta_l L \sigma^2}{M} + 16\eta_l^2 E^2 L^2 \bar \gamma^2.+ \frac{\bar \eta_g (G^2+ \psi_2) + \bar \eta_g^2 L\psi_2}{E\eta_l}
    \end{equation}
    Let $\eta_l = \frac{1}{LE\sqrt {T}}$ and $\eta_g = \frac{1}{L_R RT}$. Then,
    \begin{equation}
        \frac{1}{ET} \sum_{t=0}^{T-1}\sum_{e=0}^{E-1} \|\nabla \gL(\bar w^{t,e})\|^2 \leq \frac{2L}{\sqrt {T}} \big(\gL(w^0) - \gL(w^{T-1})\big) + 
        \frac{\sigma^2}{EM\sqrt T} + 16 \frac{\bar\gamma^2}{T} + \frac{L (G^2+ \psi_2) + L^2 \psi_2/L_R T}{EL_R \sqrt T}
    \end{equation}
\end{proof}

\subsection{Proofs for Personalized Model Convergence Guarantee in \cref{thm:convergence2}}
\label{app:proofs-per-model-converge}

\paragraph{Additional notations}

Let 
\begin{align}
\bar \eta_p  = S \eta_p
\end{align}
Based on the update rules, we define  $\delta_{v_m}^t$  as below, which capture the update of personalized model during client training.
\begin{align}
    v_m^{t+1}- v^{t} = - \bar \eta_p \delta_{v_m}^t 
\end{align}
That is: 
\begin{equation} \label{eq:def_delta_v_m}
	 \delta_{v_m}^t := - \frac{1}{\eta_p S}( v_m^{t+1}- v^{t})=    \frac{1}{S}\sum_{s=0}^{S-1}  \widetilde{\nabla} P_m({v}_m^{t,s}, w^t  ) =   \frac{1}{S}\sum_{s=0}^{S-1} \left(\widetilde{\nabla} \gL_m( {v}_m^{t,s})  + \lambda \left( {v}_m^{t,s} - w^t\right)\right) 
\end{equation}

\paragraph{Proof Outline}

 The goal is to bound the gradients of personalized models  w.r.t the (\ref{eq:per-obj}), which is used to show that the trained models can converge to the stationary points: 
\begin{equation}
      \frac{1}{TS} \sum_{t=0}^{T-1} \sum_{s=0}^{S-1} \E\|\nabla_v P_m(v_m^{t,s}, w^t) \|^2
\end{equation}

\paragraph{Supporting lemmas}

We first introduce some supporting lemmas:

\begin{lemma}
When $\eta_p \leq \frac{1}{L+\lambda}$,
\label{lm:personalized}
    \begin{equation}
        \frac{1}{TS} \sum_{t=0}^{T-1} \sum_{s=0}^{S-1} \|\nabla_v P_m(v_m^{t,s}, w^t) \|^2 \leq \frac{2(P_m(v^0, w^0) - P_m(v^T, w^T))}{\eta_p TS} + (L+\lambda)\eta_p \delta^2 + \frac{1}{TS}\sum_{t=0}^{T-1}\frac{G_P  \E \|w^{t+1} - w^{t}\|_2}{\eta_p}
    \end{equation}
\end{lemma}
\begin{proof}
Let $\eta_p \leq \frac{1}{L+\lambda}$.
    \begin{align}
       \E [  P_m(v_m^{t,s+1}, w^t) - P_m(v_m^{t,s}, w^t)] & \leq \E \langle   \nabla_v P_m(v_m^{t,s}, w^t), -\eta_p \tilde \nabla_v P_m(v_m^{t,s}, w^t)  \rangle + \frac{(L+\lambda)\eta_p^2}{2}  \E \|\tilde \nabla_v P_m(v_m^{(t,s)}, w^t)\|^2 \tag{\cref{asp:smooth}} \\
        &\leq -\eta_p \|\nabla_v P_m(v_m^{t,s}, w^t)\|^2 + \frac{(L+\lambda)\eta_p^2}{2}(\sigma^2 + \|\nabla_v P_m(v_m^{t,s} - w^t)\|^2)  \\
        &\leq -\frac{\eta_p}{2} \|\nabla_vP_m(v_m^{t,s}, w^t)\|^2 + \frac{(L+\lambda) \eta_p^2 \sigma^2}{2} \tag{By $\eta_p \leq \frac{1}{L+\lambda}$ }.
    \end{align}

    Taking a summation $\sum_{s=0}^{S-1}$, 
    \begin{equation}
        \E [ P(v_m^{t+1}, w^{t}) - P(v_m^{t}, w^t)] \leq -\frac{\eta_p}{2}\sum_{s=0}^{S-1} \|\nabla_v P_m(v_m^{t,s},w^t)\|^2 + \frac{(L+\lambda)S\eta_p^2 \sigma^2}{2}.
    \end{equation}

    We next bound $P_m(v_m^{t+1},w^{t+1}) - P_m(v_m^{t+1},w^t)$. Since bounded gradient implies Lipschitz function,  $P_m$ is $G_P$ Lipschitz in terms of $w$, i.e., 
    \begin{equation}
       \E[  P_m(v_m^{t+1},w^{t+1}) - P_m(v_m^{t+1},w^t) ] \leq G_p  \E \|w^{t+1} - w^t\|_2. \tag{\cref{asp:boundedgrad}}
    \end{equation}

    Combine the two statements, rearrange, and take the sum $\sum_{t=0}^{T-1}$ on both side:
    \begin{equation}
         \frac{1}{TS} \sum_{t=0}^{T-1} \sum_{s=0}^{S-1}  \E \|\nabla_v P_m(v_m^{t,s}, w^t) \|^2 \leq \frac{2(P_m(v^0, w^0) - P_m(v^T, w^T))}{\eta_p TS} + (L+\lambda)\eta_p \delta^2 + \frac{1}{TS}\sum_{t=0}^{T-1}\frac{2G_P  \E\|w^{t+1} - w^{t}\|_2}{\eta_p}
    \end{equation}
\end{proof}

\begin{lemma}
\label{lm:wt}
When $\eta_l = \frac{1}{EL \sqrt T}$,
    \begin{equation}
         \E \|w^{t+1} - w^t\|^2 \leq 8 \eta_l^2 \sigma^2 + \frac{\eta_l^2\phi_1}{T} + 4\eta_l^2\mathbb E\|\nabla \gL(w^t)\|^2 + 2\eta_g^2 R^2 G_R^2. 
    \end{equation}
    where $\phi_1 = 64(3\bar \gamma  + \frac{2\sigma^2}{E})$.
\end{lemma}
\begin{proof}
By definition
    \begin{align}
       \E \| w^{t+1} - w^t\|^2 &= \E  \|\eta_l E \frac{1}{M}\sum_{m=1}^M g_m^t + \eta_g Rg^t \|^2 \nonumber\\
        &\leq 2\eta_l^2 E^2 \E  \| \frac{1}{M}\sum_{m=1}^M g_m^t\|^2   +  2\eta_g^2 R^2 \E  \| g^t \|^2 
    \end{align}
For the first term,
    \begin{align*}
        \mathbb E\|\frac{1}{M}\sum_{m} g_m^t\|^2 &= \mathbb E \|\frac{1}{M} \sum_m \Big(\frac{1}{E} \sum_{e=0}^{E-1} \tilde \nabla \gL_m(\theta_m^{t,e}) - \nabla \gL_m(w^t)\Big) + \nabla \gL (w^t)\|^2\\
        &\leq \frac{2}{EM} \sum_{m=1}^{M}\sum_{e=0}^{E-1} \mathbb E \|\tilde \nabla \gL_m(\theta_m^{t,e}) - \nabla \gL_m(w^t) \|^2 + 2\mathbb E\|\nabla \gL(w^t)\|^2\\
&\leq 4\sigma^2 + 64\eta_l^2 E^2 L^2 (3\bar \gamma + 3\E  \left\|   \nabla \gL(w^t) \right\|^2 + \frac{2\sigma^2}{E}) + 2\mathbb E\|\nabla \gL(w^t)\|^2 
        \tag{\cref{lm:bound_local_client_drift}}\\
    &= 4\sigma^2 +  64\frac{1}{T} (3\bar \gamma + 3\E  \left\|   \nabla \gL(w^t) \right\|^2 + \frac{2\sigma^2}{E}) + 2\mathbb E\|\nabla \gL(w^t)\|^2 \\
      &= 4\sigma^2 +  \frac{1}{T}\phi_1  +  (2+ \frac{ 192}{T} ) \mathbb E\|\nabla \gL(w^t)\|^2 
    \end{align*}
For the second term, recall \cref{eq:define_psi_2}, then we have
        \begin{equation}
             \mathbb E\|g^t\|^2 \leq \psi_2
        \end{equation}
    
Therefore,
\begin{align*}
    \mathbb E   \|w^{t+1} - w^t\|^2 &\leq 2\eta_l^2 E^2 \E\|\frac{1}{M} \sum_{m=1}^M g_m^t\|^2 + 2\eta_g^2 R^2 (\sigma_R^2 + G_R^2)\\
    &= 8 \eta_l^2 \sigma^2 + \frac{\eta_l^2\phi_1}{T} + 4\eta_l^2 (1+\frac{96}{T})\mathbb E\|\nabla \gL(w^t)\|^2 + 2\eta_g^2 R^2 \psi_2. 
\end{align*}
\end{proof}

\paragraph{Completing the proof of \cref{thm:convergence2}}

Recall \cref{thm:convergence2}:
\perconv*
Next, we combine \cref{lm:personalized} and \cref{lm:wt} to prove the above theorem.
\begin{proof}
From \cref{lm:personalized}, 
\begin{equation}
         \frac{1}{TS} \sum_{t=0}^{T-1} \sum_{s=0}^{S-1} \E\|\nabla_v P_m(v_m^{t,s}, w^t) \|^2 \leq \frac{2\Delta_{P_m}}{\eta_p TS} + (L+\lambda)\eta_p \sigma^2 + \frac{L+\lambda}{\sqrt TS}\sum_{t=0}^{T-1} 2G_P \E \|w^{t+1} - w^{t}\|_2
\end{equation}
Expand the last term according to \cref{lm:wt}.
    \begin{align*}
    \frac{1}{\sqrt T}\sum_{t=0}^{T-1} \E\|w^{t+1}- w^t\| &\leq \mathbb E \sqrt{\sum_{t=0}^{T-1} \|w^{t+1} - w^t\|^2} \tag{Taking square root for \cref{prop:jensen-ineq}}\\
    &\leq \sqrt{\sum_{t=0}^{T-1} \E\|w^{t+1} - w^t\|^2} \tag{Jensen's inequality  $\E [f(x)] \leq f(\E [x]) $ for the concave function $f(x)$ }\\
    &= \sqrt{8\sigma^2 \eta_l^2 T +  \eta_l^2\phi_1 + 4\sum_{t=0}^{T-1}  \eta_l^2\mathbb E\|\nabla \gL(w^t)\|^2 + 2\eta_g^2 R^2 \psi_2 T}\\
    &\leq \gO(\sigma \sqrt T\eta_l) + \eta_l \sqrt{\phi_1} + 2\eta_l \cdot \sqrt{\sum_{t=0}^T \sum_{e=0}^{E-1} \mathbb E\|\nabla \gL (\bar w^{t,e})\|^2} + \gO(\eta_g R \sqrt{\psi_2} \sqrt T) \tag{$\sqrt{\sum_{i=1}^n x_i^2} \leq \sum_{i=1}^n x_i$ for $x_i \geq 0$} \\
    &\leq \gO(\frac{\sigma}{LE}) + \frac{\sqrt{\phi_1}}{LE\sqrt T} + \frac{2}{L\sqrt {E}} \gO\Big(T^{-1/4}(L\Delta_{\gL}+\psi_1)^{1/2} + \frac{\bar \gamma}{\sqrt T}+T^{-3/4}\frac{L \sqrt {\psi_2}}{L_R \sqrt E}\Big) + \frac{\sqrt{\psi_2}}{\sqrt T}.
\end{align*}

\begin{align*}
 &\frac{1}{TS} \sum_{t=0}^{T-1} \sum_{s=0}^{S-1} \E\|\nabla_v P_m(v_m^{t,s}, w^t) \|^2 \\
 & \leq \frac{2\Delta_{P_m} (L+\lambda)}{ \sqrt TS} + \frac{\sigma^2}{ \sqrt{T} }   \\
 & + \frac{(L+\lambda) 2G_P }{S} \left (\gO(\frac{\sigma}{LE}) + \gO\left( (\frac{\sqrt{\phi_1}}{LE } + \sqrt{\psi_2}  + \frac{2 \bar\gamma}{L\sqrt {E}} ) \frac{1}{\sqrt T} \right)   + \frac{2}{L\sqrt {E}} \gO\Big(T^{-1/4}(L\Delta_{\gL}+\psi_1)^{1/2} +T^{-3/4}\frac{L \sqrt {\psi_2}}{L_R \sqrt E}\Big)  \right) \\
  & = \gO( \frac{(L+\lambda)\Delta_{P_m}+\phi_2}{\sqrt T S}) +  \frac{(L+\lambda) 2G_P }{S} \left (\gO(\frac{\sigma}{LE})   + \frac{2}{L\sqrt {E}} \gO\Big(T^{-1/4}(L\Delta_{\gL}+\psi_1)^{1/2} +T^{-3/4}\frac{L \sqrt {\psi_2}}{L_R \sqrt E}\Big)  \right)
\end{align*}

\end{proof}

\end{document}